\documentclass{article} 
\PassOptionsToPackage{table}{xcolor} 
\usepackage{iclr2027_conference,times}

\usepackage{amsmath,amsfonts,bm}

\def\eqref#1{equation~\ref{#1}}

\def\1{\bm{1}}

\DeclareMathAlphabet{\mathsfit}{\encodingdefault}{\sfdefault}{m}{sl}
\SetMathAlphabet{\mathsfit}{bold}{\encodingdefault}{\sfdefault}{bx}{n}

\definecolor{citecolor}{HTML}{0071bc}
\definecolor{refcolor}{HTML}{750014}
\usepackage[pagebackref=false,breaklinks=true,colorlinks,citecolor=citecolor,linkcolor=refcolor,urlcolor=black,bookmarks=false]{hyperref}
\usepackage{url}
\usepackage{amsmath}
\usepackage{amssymb}
\usepackage{amsthm}
\newtheorem{proposition}{Proposition}
\newtheorem{corollary}{Corollary}
\newtheorem{lemma}{Lemma}
\usepackage{booktabs}
\usepackage{xcolor}
\usepackage{array}
\definecolor{oursgreen}{HTML}{D2EBD2}  
\definecolor{ourslight}{HTML}{E6F5E6}  
\definecolor{ourred}{HTML}{750014}
\usepackage{graphicx}
\makeatletter\@ifundefined{KV@Gin@alt}{\define@key{Gin}{alt}{}}{}\makeatother
\usepackage{tabularx}
\usepackage{subcaption}
\usepackage{multirow}
\usepackage{tikz}
\usepackage{float}
\usepackage{wrapfig}
\usepackage{enumitem}
\makeatletter
\g@addto@macro\fs@ruled{\def\@fs@post{\kern2pt\hrule\relax\vspace{-8pt}}}
\makeatother
\floatstyle{ruled}
\newfloat{algorithm}{tbp}{loa}
\floatname{algorithm}{Algorithm}
\newcommand{\ms}[2]{$#1_{\,\pm#2}$}
\usetikzlibrary{arrows.meta,positioning,calc,fit,decorations.pathreplacing}

\title{Test-Time Adaptation of Quantized ViTs via Single-Pass Quantizer-Aligned Recalibration}

\newcommand{\coderepo}{chahh9808/QuAR}
\newcommand{\codeurl}{https://github.com/\coderepo}
\iclrfinalcopy

\ifdefined\LaTeXML
\author{Hyeongheon Cha \\ KAIST \\ \texttt{hyeongheon@kaist.ac.kr}
\and Young D. Kwon \\ Samsung AI Center-Cambridge, University of Cambridge \\ \texttt{yd.kwon@samsung.com}
\and Sung-Ju Lee \\ KAIST \\ \texttt{profsj@kaist.ac.kr}}
\date{\href{\codeurl}{\includegraphics[height=1em,alt={GitHub}]{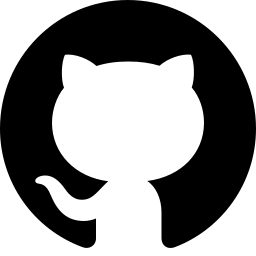}~\texttt{\coderepo}}}
\else
\author{Hyeongheon Cha\textsuperscript{1} \quad Young D. Kwon\textsuperscript{2}\textsuperscript{,}\textsuperscript{3} \quad Sung-Ju Lee\textsuperscript{1} \\
 \textsuperscript{1} KAIST \quad \textsuperscript{2} Samsung AI Center-Cambridge \quad \textsuperscript{3} University of Cambridge \\
 \texttt{\{hyeongheon,profsj\}@kaist.ac.kr} \quad \texttt{yd.kwon@samsung.com} \\
 \href{\codeurl}{\includegraphics[height=0.95em,alt={GitHub}]{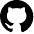}\,\texttt{\coderepo}}
 }
\fi

\newcommand{\method}{QuAR}
\newcommand{\ours}{Ours}
\newcommand{\cbar}{\ensuremath{\overline{\mathrm{C}}}}

\begin{document}

\maketitle
 \vspace{-17pt}
\begin{abstract}
 \vspace{-3pt}
Post-training quantization is a standard route to fitting vision transformers~(ViTs) into edge compute and memory budgets, yet quantized models become especially brittle under distribution shift. Test-time adaptation~(TTA) addresses such shifts without labels, but most existing approaches are poorly aligned with the constraints of quantized inference. Prevailing TTA methods recover accuracy through backpropagation, while backprop-free methods often still incur overhead from extra forward passes or parameter updates, and lightweight feature- or logit-level methods recover only part of the loss. 
Across these approaches, a quantization-specific failure mode that amplifies the drop is not directly targeted: under shift, activations occupy frozen quantizers' calibrated ranges differently, distorting their \emph{code distribution}. We propose \emph{Quantizer-Aligned Recalibration} (\method{}), a \emph{single-pass} TTA method tailored to quantized ViTs that neither backpropagates nor updates any model parameters. \method{} recalibrates activations at the input to a frozen quantizer, mapping the test stream's running per-channel statistics back toward the source calibration. On ImageNet-C with ViT-B, \method{} achieves the highest mean accuracy among state-of-the-art backprop-free TTA methods at 3-, 4-, 6- and 8-bit weight/activation precision, outperforming the strongest baseline by $2.28$ points at 8 bits and $4.00$ at 3 bits, with $46\%$ lower latency and a memory overhead of only $0.17$\,MB ($0.01\%$ of peak inference memory). Analysis and diagnostics trace the gain to a reduced per-channel mismatch at these quantizers, which restores the code distribution the baselines leave unchanged or distort further. A single fixed configuration remains ahead across continual streams, non-i.i.d.\ label shift, seven out-of-distribution suites, and three other backbones. 
 \vspace{-5pt}
\end{abstract}

\section{Introduction}
\label{sec:intro}
Vision transformers (ViTs) \citep{dosovitskiy2021image} are the backbone of many vision systems, but their compute and memory cost makes them impractical to run at full precision on edge devices. Post-training quantization (PTQ) \citep{yuan2022ptq4vit,wu2024adalog} is a standard remedy. From a  label-free calibration set it fixes weights and activation quantizers to low bit-widths with no retraining, so inference can better fit edge budgets. The deployed quantized model pays for that efficiency with rigidity, its quantizers frozen and its activation ranges tuned to the calibration distribution.

This rigidity becomes a liability at test time. When the input distribution shifts under corruption or domain change, the calibrated activation ranges no longer match it, causing a substantial accuracy drop. The drop is amplified by quantization, $1.7\times$ at W3A3 (Appendix~\ref{app:analysis}), leaving the unadapted model at $28.4$ top-1 where full precision holds $55.6$ (Figure~\ref{fig:teaser}a; W$n$A$m$ denotes $n$-bit weights and $m$-bit activations). The cause is concrete. Under shift, the per-channel activation statistics entering the quantizers drift from their calibration values (Figure~\ref{fig:problem}a). Those quantizers are frozen and per-tensor, so they cannot follow that drift. They map the activations onto a part of their range they were never calibrated for, distorting their \emph{code distribution}: the discrete levels the activations are mapped to. These per-interface distortions accumulate along the frozen path to the logits, where they meet a decision margin the shift has already narrowed. The quantization error does not grow under shift, but the same error now flips more labels. It flips more of them the coarser the step, which is why quantization amplifies the drop (Section~\ref{sec:why}). A fifth of the channels at those quantizers forfeit at least half a bit of code resolution to that drift (Figure~\ref{fig:problem}b). Recovering accuracy therefore requires correcting this intermediate, quantizer-level mismatch, not only the final prediction.

\begin{figure}[t]
\centering
\vspace{-20pt}
\includegraphics[width=0.78\textwidth]{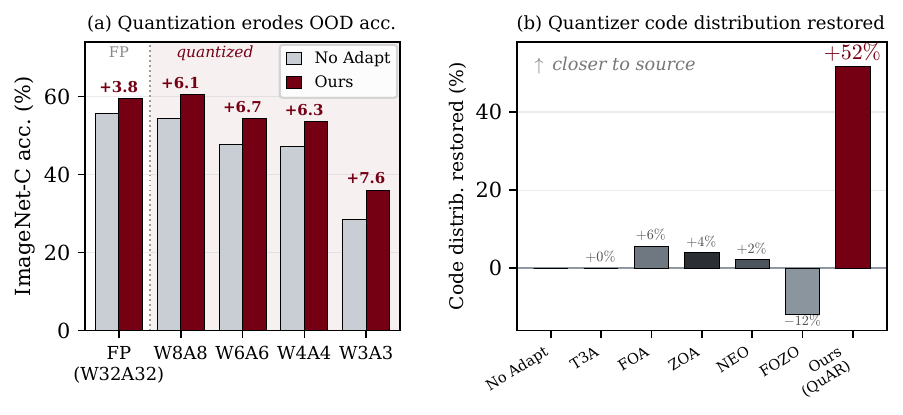}
\vspace{-15pt}
\caption{\emph{Quantization makes activation drift costlier; \method{} acts on the quantization path to undo it.} \emph{(a)} ViT-B accuracy on ImageNet-C by precision: unadapted accuracy falls from $55.6$ to $28.4$ at W3A3, while \method{}'s gain grows as precision drops, from $+3.8$ to over $+6$ at every quantized setting. \emph{(b)} At the recalibrated quantizer inputs (W6A6), corruption drifts the code distribution off its source calibration; bars show the fraction removed (Wasserstein-$1$ between code histograms, Appendix~\ref{app:analysis}). Only \method{} recovers much of that drift, ${\sim}9\times$ the best baseline.
}
\label{fig:teaser}
\vspace{-10pt}
\end{figure}

Test-time adaptation (TTA) addresses distribution shift without labels \citep{wang2021tent,wang2022continual,niu2023towards,iwasawa2021test}, and a recent line of work targets the quantized setting specifically \citep{deng2025test}. These methods recover accuracy, but quantized deployment asks for two things they do not provide. First is the cost. Entropy- or discrepancy-minimizing methods backpropagate through the network and update some or all of its parameters, while FOA \citep{niu2024test}, ZOA \citep{deng2025test}, and FOZO \citep{fozo2026} avoid backprop but still require extra forward passes per batch and update parameters at test time; each extra forward pass through a quantized graph and each model parameter update is what quantized deployment is designed to avoid. Second is placement. These methods act on either side of the quantizer inputs rather than at them, downstream at the feature or logit level and upstream on prompts or model parameters. We confirm this directly: the compared methods leave that code distribution unchanged or move it further from the source, whereas only ours recovers much of that drift (Figure~\ref{fig:teaser}b). 

We argue that for quantized ViTs the right intervention is not to re-optimize the model but to \emph{recalibrate its activations} back toward the statistics the quantizers were calibrated on. We introduce \emph{Quantizer-Aligned Recalibration} (\method{}), a \emph{single-pass}, \emph{quantization-native} TTA method built from parameter-free affine recalibrations, each applied at the input to a frozen quantizer. The local branch, \method{}'s quantizer-level core, acts at every intermediate activation interface in the late blocks, mapping the running per-channel mean and standard deviation of the test stream back toward the source calibration statistics. The global branch, a complementary correction, applies the same recalibration to the CLS representation at the head input. Both branches use source-initialized running target moments, perform no backpropagation, update no model parameters, and store no past samples; the method adds a single multiply--add per recalibrated activation inside one forward pass.

On ImageNet-C with quantized ViT-B, \method{} attains the highest mean accuracy among state-of-the-art backprop-free TTA methods at every evaluated precision, W3A3, W4A4, W6A6 and W8A8, ahead of the strongest baseline by $+2.28$ at W8A8 and $+4.00$ at W3A3, at one forward pass and no parameter updates, $46\%$ less latency than ZOA and FOZO (Section~\ref{sec:main}). One fixed configuration keeps that lead on long-term continual streams, non-i.i.d.\ online label shift, seven further out-of-distribution benchmarks (ImageNet-R/Sketch/\cbar{}/3DCC, CIFAR-10/100-C, and DomainNet-126), and three other backbones, the transformers DeiT-B and ViT-S, and a convolutional ResNet-50 (Sections~\ref{sec:continual} and~\ref{sec:generalization}). \method{} is also the most accurate method that fits an $8$\,GB Jetson Orin Nano \citep{jetson_orin_nano} at batch $64$, where the backprop-based methods run out of memory (Section~\ref{sec:backprop}).

We make the following contributions:
\vspace{-5pt}
\begin{itemize}[leftmargin=0.7cm]
\item We propose \method{}, a \emph{single-pass} TTA method tailored to \emph{quantized} ViTs that recalibrates activations at frozen quantizer inputs, targeting the distortion of their \emph{code distribution}, a quantization-specific failure mode that existing TTA does not directly address.
\item We trace the gain to a per-channel mismatch at the quantizer inputs, using a mismatch measurement, a precision sweep and a placement ablation.
\item We give a first-order account linking activation-moment drift, frozen-grid code resolution, and margin-sensitive prediction changes, and measure the quantities the account depends on.
\end{itemize}

\begin{figure}[t]
\vspace{-10pt}
\centering
\includegraphics[width=0.8\textwidth]{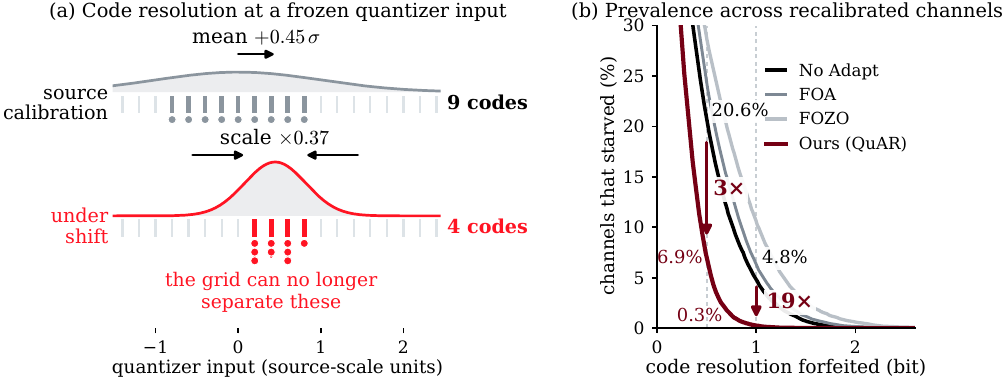}
\vspace{-5pt}
\caption{\emph{Code resolution lost under shift} (W6A6 ViT-B, ImageNet-C sev.\ 5). \emph{(a)} Both rows carry the same nine positions, one grid step apart, against the same frozen grid: the shift moves and narrows it, so distinct values collapse onto one code (Proposition~\ref{prop:levels}), and the same contraction would be free in full precision, where resolution is relative rather than absolute. Both magnitudes are measured $99$th percentiles of their per-channel distributions. \emph{(b)} Fraction of all $21{,}504$ channels at the frozen quantizers forfeiting at least a given resolution; other baselines in Appendix~\ref{app:theory}.
}
\vspace{-10pt}
\label{fig:problem}
\end{figure}

\section{Related Work}
\label{sec:related}

\paragraph{Quantized vision transformers.}
Post-training quantization (PTQ) compresses a trained ViT to low-bit weights and activations from a small calibration set, with no retraining. PTQ4ViT \citep{yuan2022ptq4vit} introduces twin uniform quantizers for the asymmetric post-Softmax and post-GELU activations, and AdaLog \citep{wu2024adalog} uses an adaptive logarithmic quantizer for those two activations at very low bit-widths, while the remaining linear-layer activations stay uniformly quantized. In either case each activation quantizer is \emph{frozen} after calibration, its per-tensor range tuned to the \emph{source} activation distribution. A test-time shift makes the per-channel activations occupy that calibrated range differently and distorts the quantizer's code distribution, a quantization-specific failure that feature- or logit-level adaptation does not directly address. We take such a PTQ-quantized ViT as the fixed deployment artifact and adapt only its activation statistics at test time; the quantizers and weights are never modified.

\paragraph{Test-time adaptation.}
Fully test-time adaptation updates a source model online from unlabeled test data. TENT \citep{wang2021tent} minimizes prediction entropy through normalization-layer affine parameters, and CoTTA \citep{wang2022continual}, EATA \citep{niu2022efficient}, and SAR \citep{niu2023towards} stabilize this for continual and wild streams; T3A \citep{iwasawa2021test} is backprop-free, adjusting a non-parametric classifier from an online support set. A related line aligns intermediate activations or features to source statistics. ActMAD \citep{mirza2023actmad} is a test-time training method that matches per-channel, location-aware activation means and variances across several convolutional layers by backpropagating the matching loss and updating all parameters; TTT++ \citep{liu2021ttt} aligns feature moments through an auxiliary self-supervised objective; and LAME \citep{boudiaf2022parameter} adjusts output label assignments rather than activations. PEA \citep{pea2026} aligns each transformer block's output with a full-covariance whitening--coloring transform, which needs a second forward pass, more tokens than channels to estimate, and a dense full-precision matrix per block, recomputed each batch. These methods are full-precision and mostly rely on backpropagation and parameter updates that are costly or ill-defined once the model is quantized. 
The recent Adaptive Quantile Recalibration~\citep{aqr2026} 
also recalibrates internal pre-activations in full precision, aligning channel-wise distributions with quantile transforms; its published ViT setting uses large batches, and we compare against that paper-default configuration
in Appendix~\ref{app:aqr}. \method{} instead matches per-channel mean and variance in \emph{closed form} at the \emph{frozen} quantizer of a low-bit ViT and at small batch sizes: a single-forward, parameter-update-free shift toward source statistics, requiring none of the gradients or full-network updates that make such test-time training infeasible on a quantized model. The Local-only ablation (Table~\ref{tab:main}) is itself a closed-form source-statistic activation-matching baseline at these interfaces.

\paragraph{Backprop-free TTA for quantized models.}
To avoid backpropagation, FOA \citep{niu2024test} optimizes a few input prompts with CMA-ES and additionally shifts activations toward source statistics, using two forward passes per step. FOZO \citep{fozo2026} estimates descent directions with SPSA-style perturbations. For the quantized setting, ZOA \citep{deng2025test} adapts the quantized model's LayerNorm affine parameters with zeroth-order updates against a feature-discrepancy loss, leaving the quantizers themselves untouched, and NEO \citep{neo2025} recenters features by a cumulative mean at the classifier, which is backprop-free and single-forward. 
FORGE \citep{forge2026} re-normalizes BN-folded convolution outputs 
in integer-only CNNs for microcontrollers, 
using the statistics associated with the folded BN affine. This addresses a different architecture and intervention site from ours: \method{} targets activation-quantizer inputs in LayerNorm-based PTQ ViTs and uses partial local matching where full matching over-corrects.
The affine form of our correction is that of normalization-statistics adaptation \citep{schneider2020improving,mirza2023actmad}; the difference is the site. Applied where those methods apply it, at the block outputs, the same affine recovers only about half the gain at every bit-width, whereas at the quantizer inputs it restores the code distribution and adds $1.8$ to $3.1$ points (Appendix~\ref{app:analysis}). These methods avoid backpropagation, but their adaptation cost is often several times a plain forward pass, and none is aware of the frozen quantizer's code distribution. Our method is backprop-free and single-forward, and recalibrates the activations \emph{entering each quantizer} in closed form against frozen per-channel \emph{source} statistics, adding one multiply and one add per recalibrated activation: the local branch at the intermediate block quantizers and the global branch at the input quantizer of the classification head.

\section{Method}
\label{sec:method}
\vspace{-3pt}
\subsection{Preliminaries}
Let $f_\theta$ be a ViT quantized by PTQ \citep{yuan2022ptq4vit,wu2024adalog}, with its weights and per-tensor activation quantizers fixed from a small calibration set drawn from the clean source distribution. Write a quantized activation as $Q(x) = s \cdot \big(\mathrm{clip}(\lfloor x/s \rceil + z,\, q_{\min}, q_{\max}) - z\big)$, where the scale $s$ and zero-point $z$ are frozen at calibration time. The clean source distribution the quantizers were calibrated on induces, at each activation interface $\ell$, a per-channel \emph{source} mean and standard deviation $(\boldsymbol\mu_s^\ell, \boldsymbol\sigma_s^\ell)$. At test time the inputs are shifted, so the per-channel statistics of the activations entering each quantizer drift away from $(\boldsymbol\mu_s^\ell,\boldsymbol\sigma_s^\ell)$; the frozen $(s,z)$ then round a mismatched range, degrading accuracy. Because the activation quantizers are per-tensor, a per-channel range drift is invisible to the single $(s,z)$ pair yet still distorts the rounded values. An upstream LayerNorm does not control this drift either. LayerNorm normalizes each token across channels, whereas the drift that mis-ranges the quantizer is per channel across tokens, so the two act along different axes. Not every quantized interface sits directly after a normalization, with the scale drift concentrated at those without one and the mean drift everywhere (Appendix~\ref{app:theory}). Restoring the source per-channel moments before quantization restores the quantizer toward its source-calibrated operating point, its code distribution; Lemma~\ref{prop:qerr} makes this precise. Our goal is to undo this drift \emph{at the activation level}, online, without touching $\theta$, $s$, or $z$, and without extra forward passes.

\vspace{-3pt}
\subsection{Recalibration at the quantizer inputs}
\method{} corrects the moment drift that displaces the frozen quantizer's code distribution, with a single per-channel affine map on the activations entering that quantizer. A local branch applies it inside the late transformer blocks, and a global branch at the classification head's input.

\begin{figure}[t]
\vspace{-25pt}
\centering
\includegraphics[width=\textwidth]{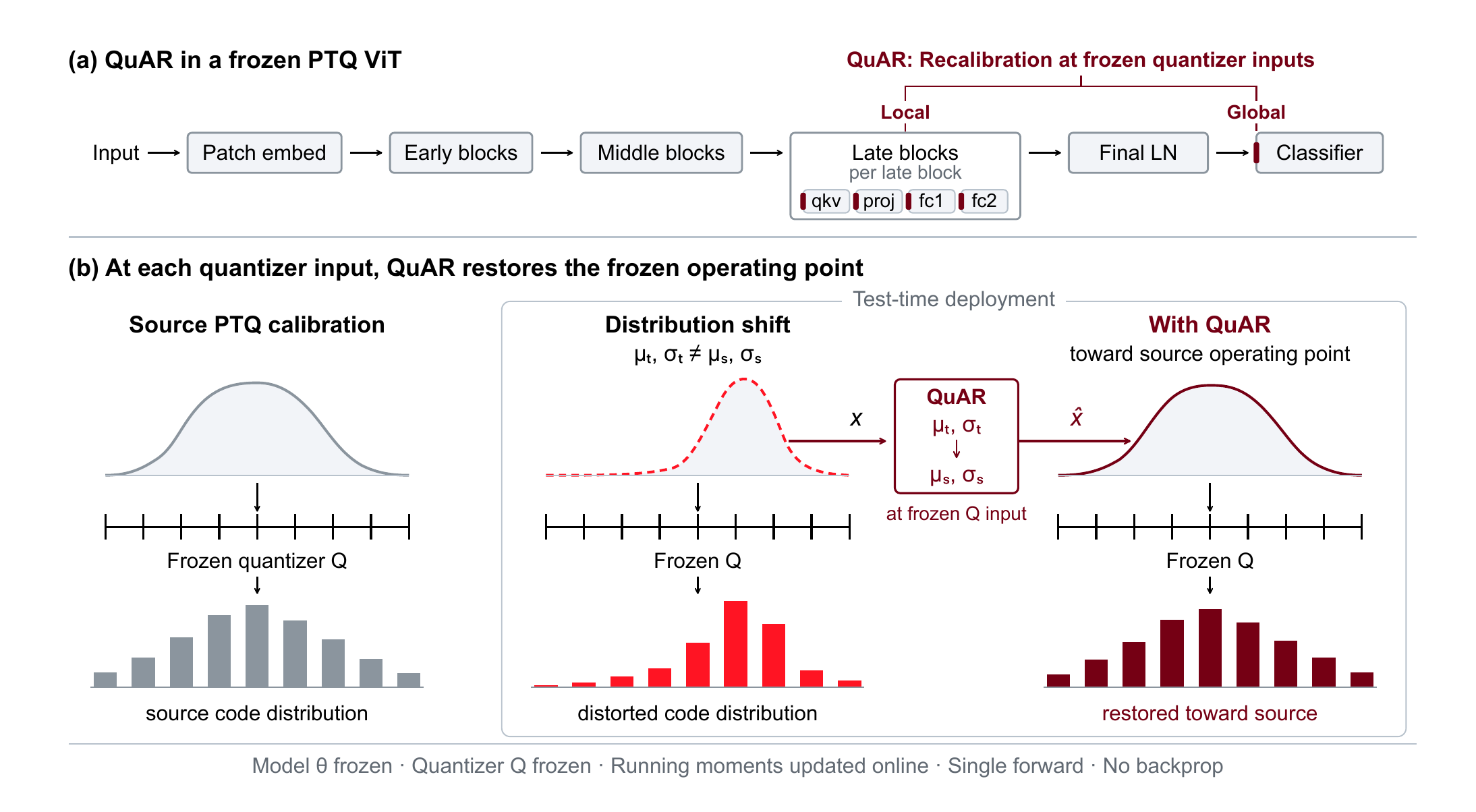}
\vspace{-25pt}
\caption{\textbf{Quantizer-Aligned Recalibration (\method{}).}
\textbf{(a)} \method{} recalibrates frozen quantizer inputs locally in the late ViT blocks and globally at the classifier input.
\textbf{(b)} Under distribution shift, the running target moments $(\mu_t,\sigma_t)$ drift from the source calibration $(\mu_s,\sigma_s)$, distorting the frozen quantizer's code distribution. At each recalibrated interface, \method{} maps the activation toward the source operating point before quantization, restoring code usage toward the source.}
\vspace{-10pt}
\label{fig:method}
\end{figure}

\vspace{-5pt}
\paragraph{Local branch.}
\method{}'s quantizer-level core (Figure~\ref{fig:method}) recalibrates at the input to \emph{every} quantized linear of the last four transformer blocks (8--11), not a chosen subset of them: the set $\mathcal{L}$ of \texttt{attn.qkv}, \texttt{attn.proj}, \texttt{mlp.fc1} and \texttt{mlp.fc2} inputs. The late blocks leave little of the frozen path to compound the correction, which the accuracy bound of Section~\ref{sec:why} makes precise; within that region accuracy is robust to the exact choice (Section~\ref{sec:scope}).

At interface $\ell$, let $x \in \mathbb{R}^{(B\cdot N)\times C}$ be the activation about to enter the quantizer ($B$ batch, $N$ tokens, $C$ channels). Using the source calibration moments $(\boldsymbol\mu_s^\ell,\boldsymbol\sigma_s^\ell)$ and running target moments $(\boldsymbol\mu_t^\ell,\boldsymbol\sigma_t^\ell)$ estimated from the stream, we apply a per-channel affine recalibration before quantization:
\begin{equation}
\label{eq:local}
\hat{x}_{:,c} \;=\; (1-\lambda)\,x_{:,c} \;+\; \lambda\Big[\big(x_{:,c} - \mu_{t,c}^\ell\big)\,
\frac{\sigma_{s,c}^\ell}{\sigma_{t,c}^\ell + \epsilon} \;+\; \mu_{s,c}^\ell\Big],
\qquad c = 1,\dots,C,
\end{equation}
and feed $Q(\hat{x})$ to the frozen quantizer; the variance ratio $\sigma_{s,c}^\ell/(\sigma_{t,c}^\ell+\epsilon)$ is clamped to $[1/\rho,\rho]$ (we use $\rho{=}3$), which keeps a near-zero $\sigma_{t,c}$ on a collapsed channel from inflating the correction 
and bounds the applied affine gain
over arbitrarily long continual streams (Proposition~\ref{prop:drift}, Appendix~\ref{app:theory}). Equation~\ref{eq:local} forms a convex combination of the activation and its fully recalibrated version by a weight $\lambda \in [0,1]$: at $\lambda=1$ it is the full recalibration, recentering to the source mean and rescaling to the source per-channel scale, and at $\lambda=0$ the identity. We use $\lambda=0.5$ in the local branch: a \emph{partial}, or soft, recalibration that applies a fraction $\lambda$ of the full correction (a matter of \emph{strength}, not of which interfaces are corrected). The full recalibration ($\lambda=1$) tends to over-correct the per-channel scale, most visibly on the smaller ViT-S and distilled DeiT-B, and $0.5$ sits at the conservative end of a flat plateau; it is selected once and then fixed for every reported result (Section~\ref{sec:hpsens}). A partial match is also the \emph{bias--variance optimum} once the target moments are estimated from finite samples: the noisier the per-channel scale estimate, the more the optimal strength shrinks toward the identity (Proposition~\ref{prop:shrink}). The recalibration is implemented as a forward pre-hook and adds no forward.

\vspace{-5pt}
\paragraph{Global branch.}
A complementary correction applies the same affine recalibration of Equation~\ref{eq:local} to the final CLS representation $h \in \mathbb{R}^{C}$ at the input to the classification head, using CLS-specific source and running target moments, but at \emph{full} strength. It can match fully, where the local branch matches only partially, for a positional reason. The CLS representation is \emph{terminal}, consumed only by the frozen, source-trained head, so a full match aligns the head input with the source feature distribution it was trained on, and has nothing downstream to distort. A local interface instead feeds the rest of the network, where forcing activations all the way to source risks over-writing test-distribution structure that downstream blocks may use, so a partial match restores the quantizer's calibrated range while retaining more of that structure.  The full match corrects the residual shift at the head input and, like the local branch, does so in closed form against source statistics.

\vspace{-3pt}
\subsection{Running target moments}
\label{sec:moments}
\vspace{-3pt}
The target moments are estimated online and \emph{source-initialized}: $\boldsymbol\mu_t^{\ell,(0)} = \boldsymbol\mu_s^\ell$, $\boldsymbol\sigma_t^{\ell,(0)} = \boldsymbol\sigma_s^\ell$. For each batch, we recalibrate with the current $(\boldsymbol\mu_t,\boldsymbol\sigma_t)$ (Equation~\ref{eq:local}) and update them with a running average whose forgetting rate $\alpha_t$ sets how quickly older batches are down-weighted:
\begin{equation}
\label{eq:ema}
\boldsymbol\mu_t^\ell \leftarrow (1-\alpha_t)\,\boldsymbol\mu_t^\ell + \alpha_t\,
\widehat{\boldsymbol\mu}(x),\qquad
\boldsymbol\sigma_t^\ell \leftarrow (1-\alpha_t)\,\boldsymbol\sigma_t^\ell + \alpha_t\,
\widehat{\boldsymbol\sigma}(x),
\end{equation}
where $\widehat{\boldsymbol\mu},\widehat{\boldsymbol\sigma}$ are the per-channel batch moments. By default we use a forgetting-free \emph{cumulative} estimate, $\alpha_t = 1/(n_\ell+1)$ over the $n_\ell$ batches seen, the source moments counting as one pseudo-batch, so all test batches weigh equally; moments from one batch alone are noisier and class-dependent, while moments converged over an extra pass change little (Appendix~\ref{app:sens}; Table~\ref{tab:pooling}). The positional argument behind the full match also sets when each branch updates. The local branch, whose output feeds the rest of the network, corrects before it updates, so it uses only the preceding batches and never normalizes a batch by its own statistics; the global branch, where a moment error has nothing downstream to compound it, updates first and so adapts from the first batch of a domain. Accuracy on the first five batches is already within about one point of the full-stream value, and using one rule for both branches costs accuracy either on the first batch or under label shift (Appendix~\ref{app:sens}; Table~\ref{tab:branchlag}). For continually evolving domains (Section~\ref{sec:continual}), where the target moments drift, we instead use a fixed forgetting rate that tracks them (an exponential moving average, $\alpha_t = \alpha = 0.1$). Neither rate is tuned per setting. Algorithm~\ref{alg:lgar} summarizes the procedure.

\begin{algorithm}[t]
\caption{\method{} inference for one test batch.}
\label{alg:lgar}
{\small
\begin{flushleft}
\textbf{Require:} frozen quantized model; source moments $(\boldsymbol\mu_s^\ell,\boldsymbol\sigma_s^\ell)$ from clean source data, $\ell\in\mathcal{L}\cup\{\mathrm{cls}\}$; forgetting rate $\alpha_t$; variance-match strength $\lambda$; stabilizer $\epsilon$.\\[1pt]
\textbf{for} each test batch (one forward pass) \textbf{do}\\
\hspace{1.2em}\textbf{for} each local interface $\ell\in\mathcal{L}$ (pre-quant hook) \textbf{do}\\
\hspace{2.4em}$\hat{x}\!\leftarrow\!(1{-}\lambda)\,x+\lambda\big[(x-\boldsymbol\mu_t^\ell)\odot\boldsymbol\sigma_s^\ell\oslash(\boldsymbol\sigma_t^\ell+\epsilon)+\boldsymbol\mu_s^\ell\big]$; feed $Q(\hat{x})$\quad$\triangleright$ partial match, strength $\lambda$\\
\hspace{2.4em}$\boldsymbol\mu_t^\ell\!\leftarrow\!(1{-}\alpha_t)\boldsymbol\mu_t^\ell+\alpha_t\widehat{\boldsymbol\mu}(x)$;\ $\boldsymbol\sigma_t^\ell\!\leftarrow\!(1{-}\alpha_t)\boldsymbol\sigma_t^\ell+\alpha_t\widehat{\boldsymbol\sigma}(x)$\quad$\triangleright$ running update (cumulative $\alpha_t{=}1/(n{+}1)$)\\
\hspace{1.2em}\textbf{end for}\\
\hspace{1.2em}update $(\boldsymbol\mu_t^{\mathrm{cls}},\boldsymbol\sigma_t^{\mathrm{cls}})$ likewise, \emph{then} $\hat{h}\!\leftarrow\!(h-\boldsymbol\mu_t^{\mathrm{cls}})\oslash(\boldsymbol\sigma_t^{\mathrm{cls}}+\epsilon)\odot\boldsymbol\sigma_s^{\mathrm{cls}}+\boldsymbol\mu_s^{\mathrm{cls}}$\\
\hspace{1.2em}\textbf{predict} $\arg\max\,\mathrm{head}(\hat{h})$
\end{flushleft}}
\end{algorithm}

\subsection{Test-time cost}
\label{sec:cost}
\vspace{-3pt}
\method{} costs a \textbf{single forward pass} per batch, \textbf{zero test-time-updated model or quantizer parameters}, and \textbf{no stored past samples}. Its only state is about $170$\,KB (0.01\% of peak memory) of running target moments (Appendix~\ref{app:cost}); it is backprop-free and adds a multiply--add per recalibrated activation. One forward pass suffices because each correction depends only on what passes through its own interface. The target is fixed, since calibration gives the source moments, and the input-side statistics are the running moments of the very activations the hook sees. The hook therefore corrects the activation and updates the moments inside the ordinary forward, at any depth, with nothing to evaluate at the output and no search direction. This lets \method{} act before the intermediate quantizers in a single pass. FOA, ZOA and FOZO instead adjust a few parameters against a loss that requires running the network on the candidates, and their derivative-free estimates take two such runs in the cheapest configuration each publishes; NEO and T3A are single-forward because they too are closed-form, but act only at the head. The source statistics $(\boldsymbol\mu_s, \boldsymbol\sigma_s)$ add nothing at test time. A forward hook accumulates them once, offline, over unlabeled clean source images, and they ship with the model. PTQ calibration already presumes that access (Appendix~\ref{app:repro}) and the optimization-based baselines need it for their source statistics; no target data, labels, or gradients are involved.

\subsection{Analysis of the quantization-specific loss}
\label{sec:why}
We analyze, to first order, 
how recalibration reduces the quantization-specific component of the shift-induced drop, through a simplified account of moment drift, frozen-grid code occupancy, and margin-sensitive logit perturbations, assuming the per-channel moments capture the shift 
(\textbf{(A1)}: the standardized activation $(x_c-\mu_{\cdot,c})/\sigma_{\cdot,c}$ keeps an approximately shift-invariant, sub-Gaussian shape) and a source-calibrated uniform affine quantizer (\textbf{(A2)}); the remaining assumptions, formal statements and proofs, and the logarithmic-quantizer extension are in Appendix~\ref{app:theory}. Under them, matching an activation's per-channel moments to the source values returns it to the quantizer's calibrated range and code occupancy, whereas a moment mismatch over- or under-fills that range (Lemma~\ref{prop:qerr}).

Propagating the quantization error $e_{\ell,c}=Q(h_{\ell,c})-h_{\ell,c}$ to the output links this to accuracy. With $K_{>\ell}$ the product of the per-block operator norms after interface $\ell$, the logit perturbation relative to the unquantized network obeys $\|\hat z-z\|\le\sum_\ell K_{>\ell}\,\|e_\ell\|$, and each $\mathbb{E}\,e_{\ell,c}^2$ splits into a granular term $\Delta_{\ell,c}^2/12$, fixed by the frozen step and hence \emph{invariant} to recalibration, and a clipping term that grows once the activation leaves the calibrated range (Proposition~\ref{prop:acc}).
\vspace{-5pt}
\paragraph{Why a displaced operating point costs accuracy.} Because the granular term is set by the frozen step and clipping barely changes at the tested bit-widths (Appendix~\ref{app:theory}), recalibration cannot work by shrinking $\{e_\ell\}$; what it moves is where the activation sits on the grid, returning it to the occupancy that the quantizer and the frozen path downstream were calibrated at (Figure~\ref{fig:teaser}b). No compared method does this; T3A and NEO act downstream of these interfaces, FOA and FOZO write upstream through prompts, and ZOA through the LayerNorm affines of the early blocks (Section~\ref{sec:analysis}). Displacement also carries a cost that exists only under a frozen fixed-point grid.
\begin{proposition}[Code-level starvation under a frozen grid]\label{prop:levels}
Let channel $c$ enter a frozen per-tensor quantizer of step $\Delta$ and clip ratio $\rho$, so at
calibration it is resolved into $L_{s,c}=2\rho\,\sigma_{s,c}/\Delta$ code levels, and let
$b_c=\log_2(\sigma_{s,c}/\sigma_{t,c})$. A channel that contracts loses $b_c>0$ levels of
resolution, and since its quantized output then takes at most $L_{s,c}2^{-b_c}$ distinct values, the
label information it can carry obeys $I(Q(x_c);y)\le\log_2 L_{s,c}-b_c$. \end{proposition}
\vspace{-6pt}
In full precision, where resolution is relative to the value represented, or under a quantizer recalibrated to $\sigma_{t,c}$, the same contraction loses nothing; the loss exists only because the grid is frozen. Restoring $\sigma_{t,c}$ to $\sigma_{s,c}$ restores the levels, and since $b_c$ varies across channels no single global scale can. On this stream a fifth of the channels forfeit at least half a bit, which \method{} cuts to $7\%$ while no compared method removes more than $1\%$ (Figure~\ref{fig:problem}b). This predicts that matching the scale as well as the mean pays most where resolution is forfeited, and it does, $+2.72$ points at those interfaces against $+1.04$ at those that forfeit none, at matched module type, width and depth (Appendix~\ref{app:analysis}).
\vspace{-7pt}
\paragraph{Why the late blocks.} A correction's residual is weighted by the downstream gain
$K_{>\ell}$, so a late interface is the cheaper place to correct: we measure $K_{>\ell}$ to be ${\sim}4.4\times$ smaller at the last four blocks than at the early ones (Corollary~\ref{prop:depth}, Appendix~\ref{app:theory}). Recalibrating earlier blocks instead forces the correction through many downstream quantizers that compound it, dropping accuracy below No Adapt and collapsing for the full network (Table~\ref{tab:scope}; Appendix~\ref{app:scope}).
\vspace{-7pt}
\paragraph{Why the correction is worth more as precision falls.} A prediction survives quantization while the propagated perturbation stays inside the unquantized logit margin. The shift shrinks that margin by $46\%$ on ImageNet-C, while the granular error is unchanged by the shift and grows as the step coarsens (Appendix~\ref{app:analysis}). Low precision therefore leaves more predictions within flipping distance, and restoring the operating point rescues more of them, $+3.8$ points at full precision against over $+6$ at every quantized setting (Figure~\ref{fig:teaser}a).
\vspace{-7pt}
\paragraph{Why the match is partial rather than full.} With finite-sample moment estimates the accuracy-relevant error is minimized not at full restoration but at a stable, partial match (Proposition~\ref{prop:shrink} and Corollary~\ref{cor:partial}, Appendix~\ref{app:theory}). A full match divides each activation by the \emph{estimated} test scale $\sigma_t$ and re-injects the estimate's noise, which a partial match ($\lambda{<}1$) and a smoothed running estimate damp. A full match therefore moves the code distribution further toward source without improving accuracy (Table~\ref{tab:sens}); the configuration that maximizes the diagnostic is not the one to deploy, and \method{} uses $\lambda{=}0.5$ with a cumulative target estimate.

\section{Experiments}
\label{sec:exp}

\paragraph{Setup.} Unless noted, all experiments use a quantized ImageNet-trained ViT-B on ImageNet-C
\citep{hendrycks2019benchmarking} at severity 5, all 15 corruptions, the RobustBench \citep{croce2021robustbench} 5k subset, batch size 64, with episodic reset between corruptions, on an NVIDIA TITAN RTX ($24$\,GB). We quantize with PTQ4ViT \citep{yuan2022ptq4vit} at W6A6/W8A8 and AdaLog \citep{wu2024adalog} at W4A4/W3A3. We compare against T3A \citep{iwasawa2021test}, FOA \citep{niu2024test}, ZOA \citep{deng2025test}, NEO \citep{neo2025}, and FOZO \citep{fozo2026}, the backprop-free methods that run on the frozen quantized graph. Accuracy is the 15-corruption mean top-1, with 3-seed mean $\pm$ std (seeds $0/1/2$) for the main comparisons, the scope and sensitivity ablations, and the code-distribution diagnostics. The remaining protocol detail is in Appendix~\ref{app:repro}.
\subsection{Main results}
\label{sec:main}

\paragraph{Accuracy.}
\method{} attains the best mean accuracy at \emph{every} bit-width (Table~\ref{tab:main}), with margins over the strongest baseline from $+2.28$ at W8A8 to $+4.00$ at W3A3, against a seed standard deviation of at most $0.14$ for either method. \method{} at W8A8 also exceeds the unadapted full-precision accuracy under corruption, and at W6A6 comes within $1.2$ points of it. The ranking is not an artifact of the corruption mean, the evaluation subset, the quantizer, or symmetric bit-widths. \method{} wins the majority of the 15 individual corruptions and keeps its lead on the full $50$k ImageNet-C evaluation (Appendix~\ref{app:tables}), under the alternative PTQ method wherever that quantizer yields a usable model (Appendix~\ref{app:crossquant}), and at asymmetric W4A8, where the margin is $+2.01$ (Appendix~\ref{app:w4a8}).
\paragraph{Test-time cost.}
It reaches that accuracy at single-forward cost and without writing a parameter, at least $46\%$ less latency than the two-forward FOA/FOZO/ZOA, which also write $2.3$K to $25$K parameters per batch (Appendix~\ref{app:cost}). That comparison reports the optimization-based baselines at their minimum-forward setting; FOA at its recommended population ($K{=}28$) costs far more forwards and still stays below \method{} at every bit-width (Appendix~\ref{app:foa28}). We also measured cost on a Jetson Orin Nano \citep{jetson_orin_nano}, where \method{} runs at $1.071\times$ the No Adapt latency (Appendix~\ref{app:ondevice}). NEO shares that single-forward cost but trails \method{} by $3.00$ points at W6A6.

\begin{table}[t]
\vspace{-10pt}
\centering
\caption{Mean top-1 (\%) on ImageNet-C by bit-width. Lower block: \emph{component ablation}. Columns: \#FP, forwards/batch; Upd.\ par., test-time-updated parameters; Lat., per-image latency (ms); Code dist., percentage of the corruption-induced code-distribution drift removed at W6A6, as \emph{intermediate\,/\,head}; Closed, no iterative or zeroth-order optimization. \emph{Full precision}: W32A32 reference.}
\vspace{-9pt}
\label{tab:main}
\resizebox{\textwidth}{!}{%
\begin{tabular}{l|rrrr|ccc|cc}
\toprule
 & \multicolumn{4}{c|}{\emph{ImageNet-C accuracy (\%), bit-width}} & \multicolumn{3}{c|}{\emph{Test-time cost}} & \multicolumn{2}{c}{\emph{Mechanism}} \\
\cmidrule(lr){2-5}\cmidrule(lr){6-8}\cmidrule(lr){9-10}
Method & W3A3 & W4A4 & W6A6 & W8A8 & \#FP & Upd.\ par. & Lat.\,(ms) & Code dist. & Closed \\
\midrule
\textit{Full precision} & \multicolumn{4}{c|}{\textit{No Adapt $55.58$, \ours{} \ms{59.38}{0.01}\,(W32A32)}} & 1 & 0 & n/a & n/a & n/a \\
\midrule
No Adapt & \ms{28.35}{0.00} & \ms{47.30}{0.01} & \ms{47.68}{0.00} & \ms{54.36}{0.00} & 1 & 0 & 6.0 & 0\,/\,0 & n/a \\
T3A & \ms{28.49}{0.01} & \ms{47.42}{0.04} & \ms{47.79}{0.01} & \ms{54.16}{0.02} & 1 & 0 & 8.2 & 0\,/\,0 & yes \\
FOA & \ms{31.52}{0.09} & \ms{50.25}{0.07} & \ms{50.39}{0.05} & \ms{57.85}{0.16} & 2 & 2.3K & 13.0 & 1\,/\,48 & no \\
ZOA & \ms{31.45}{0.05} & \ms{50.87}{0.09} & \ms{51.37}{0.04} & \ms{58.16}{0.05} & 2 & $\approx$25K & 12.3 & 2\,/\,25 & no \\
NEO & \ms{31.91}{0.02} & \ms{50.29}{0.02} & \ms{51.35}{0.01} & \ms{58.09}{0.03} & 1 & 0 & 6.0 & 0\,/\,21 & yes \\
FOZO & \ms{31.98}{0.14} & \ms{50.09}{0.02} & \ms{49.78}{0.25} & \ms{57.74}{0.20} & 2 & 2.3K & 12.3 & $-$15\,/\,19 & no \\
\midrule
\multicolumn{10}{l}{\emph{Component ablation of \ours{}}} \\
\rowcolor{ourslight}Local-only & \ms{35.30}{0.07} & \ms{52.87}{0.03} & \ms{53.42}{0.05} & \ms{59.11}{0.02} & 1 & 0 & 6.6 & 47\,/\,49 & yes \\
\rowcolor{ourslight}Global-only & \ms{31.43}{0.02} & \ms{50.21}{0.02} & \ms{51.20}{0.03} & \ms{58.10}{0.01} & 1 & 0 & 6.0 & 0\,/\,84 & yes \\
\rowcolor{oursgreen}\textbf{\method{}} & \ms{\mathbf{35.98}}{0.00} & \ms{\mathbf{53.58}}{0.04} & \ms{\mathbf{54.35}}{0.02} & \ms{\mathbf{60.44}}{0.01} & \textbf{1} & \textbf{0} & \textbf{6.6} & \textbf{47\,/\,90} & \textbf{yes} \\
\bottomrule
\end{tabular}%
}
\vspace{-5pt}
\end{table}
\vspace{-5pt}

\begin{table*}[t]
\centering
\caption{Long-term continual adaptation: a $10$-round $\times$
$15$-corruption ImageNet-C stream with state carried throughout. Round-$r$ is the 15-corruption mean of
pass $r$; Avg.\ is the R1--R10 mean $\pm$ std over three seeds. \#FP: forwards/batch;
No Adapt is stateless. \method{} and NEO-Continual use the EMA estimator at a matched
$\alpha{=}0.1$. Best per bit-width \textbf{bold}.}
\vspace{-9pt}
\label{tab:t3}
\resizebox{\textwidth}{!}{%
\begin{tabular}{l|l|c|rrrrrrrrrr|r}
\toprule
BW & Method & \#FP & R1 & R2 & R3 & R4 & R5 & R6 & R7 & R8 & R9 & R10 & Avg. \\
\midrule
 & No Adapt & 1 & 28.35 & 28.35 & 28.35 & 28.35 & 28.35 & 28.35 & 28.35 & 28.35 & 28.35 & 28.35 & \ms{28.35}{0.00} \\
 & ZOA & 2 & 31.27 & 31.61 & 31.31 & 32.10 & 32.06 & 32.08 & 32.10 & 32.15 & 32.34 & 32.40 & \ms{31.94}{0.25} \\
 & NEO-Continual & 1 & 31.69 & 31.64 & 31.69 & 31.66 & 31.67 & 31.69 & 31.66 & 31.67 & 31.69 & 31.63 & \ms{31.67}{0.01} \\
\rowcolor{oursgreen}\cellcolor{white}\multirow{-4}{*}{W3A3} & \textbf{\method{}} & 1 & \textbf{35.66} & \textbf{35.64} & \textbf{35.64} & \textbf{35.66} & \textbf{35.67} & \textbf{35.70} & \textbf{35.75} & \textbf{35.70} & \textbf{35.63} & \textbf{35.65} & \ms{\mathbf{35.67}}{0.02} \\
\midrule
 & No Adapt & 1 & 47.68 & 47.68 & 47.68 & 47.68 & 47.68 & 47.68 & 47.68 & 47.68 & 47.68 & 47.68 & \ms{47.68}{0.00} \\
 & ZOA & 2 & 51.59 & 52.24 & 52.62 & 52.67 & 52.63 & 52.85 & 52.95 & 53.27 & 53.29 & 53.23 & \ms{52.73}{0.12} \\
 & NEO-Continual & 1 & 51.16 & 51.14 & 51.13 & 51.18 & 51.14 & 51.15 & 51.14 & 51.13 & 51.17 & 51.15 & \ms{51.15}{0.01} \\
\rowcolor{oursgreen}\cellcolor{white}\multirow{-4}{*}{W6A6} & \textbf{\method{}} & 1 & \textbf{53.94} & \textbf{53.90} & \textbf{53.86} & \textbf{53.93} & \textbf{53.98} & \textbf{53.92} & \textbf{53.93} & \textbf{53.89} & \textbf{53.92} & \textbf{53.89} & \ms{\mathbf{53.92}}{0.00} \\
\bottomrule
\end{tabular}%
}
\vspace{-8pt}
\end{table*}

\paragraph{Component ablation.}
\label{sec:ablation}

At W6A6 the local branch alone is worth $+5.74$ over No Adapt, the global branch alone $+3.52$, and the two together $+6.67$ (Table~\ref{tab:main}, lower block); the combination is best at every bit-width (per-corruption breakdown in Appendix~\ref{app:tables}). The local branch's gain widens as precision drops, from $+4.75$ at W8A8 to $+6.95$ at W3A3, and its contribution on top of the global branch is largest at W3A3, $+4.55$, where parameter-update-free recovery is hardest (Appendix~\ref{app:analysis}). The local branch's gain is primarily the restored quantizer scale. Over the global branch, a mean-only local recalibration adds $+0.85$ and the scale match a further $+2.30$ (Appendix~\ref{app:sens}). Only the configurations with the local branch move the intermediate interfaces, removing $47\%$ of the code-distribution drift against at most $2\%$ for any baseline, while the full method is highest at the head.

\vspace{-3pt}
\subsection{Continual and non-i.i.d.\ streams}
\label{sec:continual}

\paragraph{Long-term continual adaptation.} \method{} recalibrates against fixed source statistics, so it does not drift over long streams (Appendix~\ref{app:theory}, Proposition~\ref{prop:drift}). Over a 10-round $\times$ 15-corruption ImageNet-C stream with state carried throughout (Table~\ref{tab:t3}), \method{} is best in every round at both bit-widths and stays essentially flat, $35.6$--$35.8$ and $53.9$--$54.0$. The comparison is against ZOA, the optimization-based quantized-TTA method with a domain memory across segments, and NEO-Continual, the concurrent NEO's EMA variant at the matched forgetting rate. \method{} beats ZOA by $+3.73$ at W3A3 against $+1.19$ at W6A6, on one forward against two, so the margin grows as precision falls, and clears NEO-Continual by $+4.00$ and $+2.77$; that a final-only recentering trails ZOA at both bit-widths says the forgetting rate does not close the gap; the advantage is the local pre-quant branch.

\paragraph{Online label shift.}\label{sec:noniid} Under the non-i.i.d.\ setting, consecutive test batches are dominated by only a few classes. Following SAR \citep{niu2023towards}, we organize the full ImageNet-C test stream using the class-imbalance ratio $q_{\max}/q_{\min}$, where $q_{\max}$ and $q_{\min}$ denote the maximum and minimum class sampling probabilities.
We sweep it from $1$ to $5{\times}10^{5}$, where $10^{3}$ gives the dominant class ${\sim}50\%$ of each batch (sampler in Appendix~\ref{app:repro}). \method{} is the most accurate at every imbalance ratio (Appendix~\ref{app:noniid}; Table~\ref{tab:noniid} and Figure~\ref{fig:noniid}), leading by $1.7$--$2.7$ points while remaining stable ($54.55{\to}54.73$). This stability follows from its forgetting-free cumulative estimator, whose running mean equals the pooled stream mean and is asymptotically invariant to class order. Other baselines degrade under strong shift (e.g., ZOA by $0.8$ points), while NEO is similarly stable but $3.5$ points less accurate.

\vspace{-3pt}
\subsection{Generalization across shifts and backbones}
\label{sec:generalization}

\begin{table}[t]
\vspace{-5pt}
\centering
\begin{minipage}[t]{0.595\textwidth}
\vspace{-5pt}
\centering
\caption{Generalization to other OOD benchmarks at W3A3 and W6A6, 3-seed means; other bit-widths and std in
Appendices~\ref{app:generalization} and~\ref{app:beyond}. Best \textbf{bold}, second \underline{underlined}.}
\vspace{-8pt}
\label{tab:t4app}
\resizebox{\textwidth}{!}{%
\begin{tabular}{l|rrrrr|rrrrr}
\toprule
Method & \multicolumn{5}{c|}{W3A3} & \multicolumn{5}{c}{W6A6} \\
 & \cbar & R & Sk. & 3DCC & DN & \cbar & R & Sk. & 3DCC & DN \\
\midrule
No Adapt & 36.63 & 42.48 & 25.01 & 34.21 & 54.40 & 59.07 & 55.89 & 40.13 & 47.68 & 62.45 \\
T3A & 36.82 & 38.27 & 27.63 & 34.18 & \underline{57.29} & 59.83 & 52.74 & \underline{44.08} & 47.63 & \underline{65.00} \\
FOA & 40.86 & 43.27 & 27.83 & 35.55 & 55.40 & 63.40 & 56.17 & 42.27 & 48.69 & 61.86 \\
ZOA & 40.65 & 43.18 & 27.99 & 35.78 & 56.19 & \underline{64.02} & \underline{56.92} & 42.81 & 49.31 & 64.58 \\
NEO & 41.17 & \textbf{44.01} & 28.22 & 35.87 & 56.48 & 63.94 & \textbf{57.25} & 42.78 & \underline{49.40} & 64.56 \\
FOZO & \underline{41.57} & 43.02 & \underline{29.47} & \underline{35.88} & 52.25 & 61.67 & 54.72 & 42.51 & 48.50 & 56.95 \\
\rowcolor{oursgreen}\textbf{\method{}} & \textbf{44.92} & \underline{43.90} & \textbf{32.10} & \textbf{37.50} & \textbf{59.72} & \textbf{65.45} & 56.60 & \textbf{45.65} & \textbf{51.20} & \textbf{66.85} \\
\bottomrule
\end{tabular}%
}
\end{minipage}
\hfill
\begin{minipage}[t]{0.385\textwidth}
\centering
\centering
\vspace{-5pt}
\caption{Generalization to other backbones at W6A6 (3-seed mean $\pm$ std). Best \textbf{bold}, second \underline{underlined}.}
\vspace{-8pt}
\label{tab:backbones}
\resizebox{\textwidth}{!}{%
\begin{tabular}{l|ccc}
\toprule
Method & ViT-B & DeiT-B & ViT-S \\
\midrule
No Adapt & \ms{47.68}{0.00} & \ms{43.12}{0.00} & \ms{32.14}{0.00} \\
T3A & \ms{47.79}{0.01} & \ms{42.81}{0.01} & \ms{32.25}{0.02} \\
FOA & \ms{50.38}{0.03} & \ms{45.71}{0.05} & \ms{34.35}{0.07} \\
ZOA & \ms{\underline{51.37}}{0.03} & \ms{46.42}{0.05} & \ms{36.12}{0.04} \\
NEO & \ms{51.35}{0.01} & \ms{46.50}{0.01} & \ms{\underline{36.13}}{0.00} \\
FOZO & \ms{50.51}{0.09} & \ms{\underline{46.61}}{0.03} & \ms{32.36}{0.65} \\
\rowcolor{oursgreen}\textbf{\method{}} & \ms{\mathbf{54.27}}{0.05} & \ms{\mathbf{49.09}}{0.02} & \ms{\mathbf{38.35}}{0.06} \\
\bottomrule
\end{tabular}}
\end{minipage}
\end{table}

\paragraph{Other distribution shifts.} Without per-dataset tuning, \method{} transfers beyond ImageNet-C to further OOD benchmarks. Table~\ref{tab:t4app} shows it is strongest at every bit-width on ImageNet-\cbar{}, ImageNet-Sketch (Sk.), ImageNet-3DCC, and the domain shifts of DomainNet-126 (DN), and within $0.7$ points of the best on ImageNet-R, where the quantization-agnostic baselines edge ahead. It also leads on CIFAR-10/100-C at every bit-width, by $0.8$ to $3.7$ points over the runner-up (Appendix~\ref{app:beyond}).
\paragraph{Other backbones.} With the same configuration, \method{} also transfers across
transformer backbones at W6A6 (Table~\ref{tab:backbones}): best on all three, $+6.6$, $+6.0$ and $+6.2$ over No Adapt on ViT-B, DeiT-B and ViT-S, and seed-stable, std $\leq 0.06$ on every backbone. \method{} also carries over to a convolutional ResNet-50 at W6A6, $31.40$ against NEO's $30.79$ and No Adapt's $30.12$ (Appendix~\ref{app:cnn}). Tuning $\lambda$ per backbone would add up to $0.2$ points on the transformers and $2.5$ on the ResNet-50.

\begin{table*}[t]
\centering
\caption{Recalibration-scope ablation, Local-only (no global branch), 3-seed
mean$\pm$std; No Adapt is $47.30$/$47.68$ at W4A4/W6A6. Left group: which \emph{modules} within
blocks 8--11, and the block \emph{position} of the four recalibrated blocks, all modules. Right
group: \emph{trailing-block depth}, the number of late blocks recalibrated, all modules.
Deployed default configuration \textbf{bold}.}
\vspace{-9pt}
\label{tab:scope}
\footnotesize
\resizebox{\textwidth}{!}{%
\begin{tabular}{l|cc|l|cc}
\toprule
Setting & W4A4 & W6A6 & \emph{Trailing-block depth} & W4A4 & W6A6 \\
\midrule
\multicolumn{3}{l|}{\emph{Modules} (blocks 8--11)} & 1 block (11) & \ms{48.31}{0.03} & \ms{48.54}{0.02} \\
\quad\texttt{attn.qkv}+\texttt{attn.proj} & \ms{50.18}{0.05} & \ms{50.72}{0.05} & 2 blocks (10--11) & \ms{50.85}{0.03} & \ms{51.53}{0.05} \\
\quad\quad{}+\,\texttt{mlp.fc1} & \ms{51.27}{0.05} & \ms{51.86}{0.05} & 3 blocks (9--11) & \ms{52.03}{0.03} & \ms{52.56}{0.05} \\
\rowcolor{oursgreen}\quad\quad{}+\,\texttt{mlp.fc2} & \ms{\mathbf{52.87}}{0.03} & \ms{\mathbf{53.42}}{0.04} & 4 blocks (8--11) & \ms{\mathbf{52.87}}{0.03} & \ms{\mathbf{53.42}}{0.04} \\
\multicolumn{3}{l|}{\emph{Block position} (4 blocks)} & 5 blocks (7--11) & \ms{52.75}{0.02} & \ms{53.39}{0.07} \\
\quad early (blocks 0--3) & \ms{30.93}{0.08} & \ms{32.12}{0.04} & 6 blocks (6--11) & \ms{52.11}{0.06} & \ms{53.28}{0.03} \\
\quad middle (blocks 4--7) & \ms{42.09}{0.01} & \ms{43.29}{0.04} & 7 blocks (5--11) & \ms{51.82}{0.05} & \ms{53.35}{0.02} \\
\cellcolor{oursgreen}\quad late (blocks 8--11) & \cellcolor{oursgreen}\ms{\mathbf{52.87}}{0.03} & \cellcolor{oursgreen}\ms{\mathbf{53.42}}{0.04} & all 12 blocks & \ms{28.24}{0.04} & \ms{31.93}{0.07} \\
\bottomrule
\end{tabular}}
\vspace{-8pt}
\end{table*}

\vspace{-3pt}
\subsection{Analysis at the intermediate quantizer inputs}
\label{sec:analysis}

\paragraph{Intermediate activation mismatch.}
At the 16 interfaces the local branch acts on (blocks 8--11 $\times$ \{qkv, proj, fc1, fc2\}, W6A6), the combined per-channel source--target mismatch under No Adapt is $0.347$ (Appendix~\ref{app:analysis}). Local-only and \ours{} reduce it to $0.153$ ($-56.0\%$), whereas Global-only and NEO, acting downstream at the final CLS, reduce it by $0\%$. Saturation stays negligible at $\leq 0.1\%$, so the gain comes from a restored distribution rather than reduced clipping (Appendix~\ref{app:theory}).
\vspace{-5pt}
\paragraph{Where to recalibrate.}
\label{sec:scope}
\emph{Position} is decisive (Table~\ref{tab:scope}). Applying the same four-block, all-module recalibration to the early or middle blocks collapses accuracy to $32.12$ and $43.29$ at W6A6, below No Adapt at $47.68$, whereas the late blocks (8--11) reach $53.42$, and the other three bit-widths agree (Appendix~\ref{app:scope}). This is the empirical counterpart of the late-blocks argument in Section~\ref{sec:why}. Along the \emph{depth} axis, accuracy rises up to four trailing blocks and then plateaus, so we fix the smallest scope on it; extending to all twelve blocks collapses accuracy again, to $31.93$ at W6A6. Along the \emph{module} axis, each interface contributes, post-GELU \texttt{mlp.fc2} adding $+1.56$ at W6A6 (Appendix~\ref{app:scope}). Placement relative to the quantizer is the distinguishing choice: the same correction after the quantizer, or at the block outputs where normalization-statistics adaptation acts, recovers only about half the gain (Appendix~\ref{app:analysis}).

\vspace{-4pt}
\subsection{Hyperparameter sensitivity}
\label{sec:hpsens}

\method{} has one tuned hyperparameter, the recalibration strength $\lambda$ of Equation~\ref{eq:local}; the cumulative moment estimator is parameter-free. We select $\lambda{=}0.5$ once, on the ViT-B W6A6 sweep of Table~\ref{tab:sens}, and reuse it unchanged at every other bit-width, backbone and dataset; the baselines keep their published configurations (Appendix~\ref{app:repro}). Accuracy is insensitive to the choice: every $\lambda$ from $0.1$ to $1.0$ beats the best baseline, it moves within $0.2$ points across $\lambda\in[0.5,0.7]$ and across batch sizes $1$ to $128$, and even $8$ clean calibration batches come within $1.5$ points of the default (Appendix~\ref{app:sens}). The value also carries across backbones: ViT-B and DeiT-B peak at $0.5$, ViT-S at $0.3$ (Appendix~\ref{app:backbones}).

\vspace{-4pt}
\subsection{Comparison to backprop-based TTA}
\label{sec:backprop}

The methods in Table~\ref{tab:main} run on the frozen quantized graph directly. Backprop-based TTA instead needs a gradient through it, which the deployed integer kernels do not provide; we supply one with a straight-through estimator and run TENT \citep{wang2021tent}, EATA \citep{niu2022efficient}, SAR \citep{niu2023towards} and SPA \citep{spa2025} at W6A6 (Table~\ref{tab:backprop_acc}, Appendix~\ref{app:backprop}). None of the three lightweight methods reaches \method{}. SPA does reach a higher top-1, but not at the highest imbalance of Section~\ref{sec:noniid}, where it scores $52.63{\pm}4.24$ against \method{}'s $54.73{\pm}0.07$ (Appendix~\ref{app:spanoniid}). It also pays for that top-1 with three forward passes, a backward pass and a full-precision predictor per batch: $5.2\times$ the latency and $14.6\times$ the inference memory, $23$\,GB. \method{} adds only $0.17$\,MB at $1.10\times$ the latency, whereas the three lightweight ones need $9.2$\,GB at $1.9$ to $3.7\times$ (Appendix~\ref{app:cost}). We also ran the four on the Orin Nano, where each ran out of memory in our harness, so none runs there and \method{} is the most accurate method that does (Appendix~\ref{app:ondevice}).

\begin{table}[t]
\centering
\begin{minipage}[t]{0.51\textwidth}
\vspace{-10pt}
\centering
\centering
\caption{Sensitivity of \method{} (W6A6) to $\lambda$, the fraction of the full recalibration applied, and to batch size and source-stat size;
3-seed mean, $\pm$ std in Appendix~\ref{app:sens}, No Adapt $47.68$. Defaults in \textbf{bold}.}
\label{tab:sens}
\vspace{-5pt}
\resizebox{\linewidth}{!}{%
\begin{tabular}{l|cccccc}
\toprule
Recalibration $\lambda$ & 0 & 0.1 & 0.3 & \textbf{0.5} & 0.7 & 1.0 \\
Mean top-1 (\%) & 51.20 & 52.61 & 54.00 & \textbf{54.35} & 54.20 & 53.52 \\
\midrule
Batch size & 1 & 4 & 8 & 16 & \textbf{64} & 128 \\
Mean top-1 (\%) & 54.27 & 54.39 & 54.47 & 54.35 & \textbf{54.35} & 54.29 \\
\midrule
Source-stat batches & 16 & 32 & 64 & 128 & \textbf{200} & 256 \\
Mean top-1 (\%) & 53.65 & 53.88 & 54.18 & 54.29 & \textbf{54.35} & 54.33 \\
\bottomrule
\end{tabular}}
\end{minipage}
\hfill
\begin{minipage}[t]{0.46\textwidth}
\centering
\vspace{-10pt}
\caption{Comparison to backprop-based TTA (W6A6); 3-seed mean.
Deployable: peak memory fits a Jetson Orin Nano ($8$\,GB). Lat.: per image.
Mem.: peak memory.}
\vspace{-8pt}
\label{tab:backprop_acc}
\resizebox{\linewidth}{!}{%
\begin{tabular}{l|c|c|c|c}
\toprule
Method & Deployable & Lat.\,(ms) & Mem.\,(MB) & Mean top-1 \\
\midrule 
No Adapt & \checkmark & 6.0 ($1.0\times$) & \phantom{0}1,601 ($1.0\times$) & 47.68 \\
\midrule
SAR & $\times$ & 22.2 ($3.7\times$) & \phantom{0}9,160 ($5.7\times$) & 49.77 \\
TENT & $\times$ & 11.3 ($1.9\times$) & \phantom{0}9,159 ($5.7\times$) & 50.05 \\
EATA & $\times$ & 11.4 ($1.9\times$) & \phantom{0}9,160 ($5.7\times$) & 52.79 \\
SPA & $\times$ & 31.2 ($5.2\times$) & 23,450 ($14.6\times$) & 59.14 \\
\midrule
\rowcolor{oursgreen}\textbf{\method{}} & \checkmark & \textbf{6.6} ($\mathbf{1.1\times}$) & \phantom{0}\textbf{1,601} ($\mathbf{1.00\times}$) & \textbf{54.35} \\
\bottomrule
\end{tabular}}
\end{minipage}
\vspace{-10pt}
\end{table}
\section{Limitations and Conclusion}
\label{sec:conclusion}
\vspace{-5pt}
\paragraph{Limitations.}
The strength $\lambda$ is chosen once, not per backbone, which costs up to $0.2$ points on the transformers (Section~\ref{sec:generalization}). The per-channel shift does not fold into a per-tensor quantizer, so the deployable form is an explicit \texttt{Mul}/\texttt{Add} wrapper, which an ONNX Runtime check confirms is representable in standard operators (Appendix~\ref{app:deploy}). We have not implemented \method{} down to an integer-kernel deployment, so its latency, memory, and energy on such a stack are left to future work. On clean source data top-1 changes by at most $0.42$ points across W3A3--W8A8 (Appendix~\ref{app:noshift}), so the correction is safe but inert when adaptation is unnecessary.
\vspace{-5pt}
\paragraph{Conclusion.}
We presented \method{}, a single-pass test-time adaptation method for quantized vision transformers that recalibrates activations at the inputs to its frozen quantizers toward source calibration statistics. Across W3A3--W8A8 on ImageNet-C it is the most accurate backprop-free method at low test-time cost, its two branches are complementary, and one fixed configuration carries to continual streams, label shift, out-of-distribution suites, and other backbones. 
Our analysis identifies a quantizer-aligned activation-recalibration principle: correcting activation moments at frozen quantizer inputs becomes increasingly valuable as precision falls and the grid forfeits code resolution.
For quantized deployment, much of the accessible test-time accuracy is recoverable without any gradient, extra forward pass, parameter update, or stored data.

\section*{AI Use Statement}

We used generative AI tools to help interpret results, edit the text for grammar and readability, draft the figure scripts, and clean up code. The authors checked every interpretation and every figure against the underlying data, reviewed every edited passage, and verified that all AI-edited code reproduces the reported numbers. We take responsibility for the final content of this work, including text, claims or artifacts produced with the aid of generative AI.

\section*{Ethics Statement}

This work uses only publicly available benchmarks and publicly released model checkpoints. It involves no human subjects, collects no new data, and introduces no dataset. The contribution is methodological; it inherits the known biases and limitations of the models and datasets it is applied to, and neither adds to nor removes them. Test-time adaptation of a quantized model lowers the cost of running vision classifiers on edge hardware; we see no application-specific risk beyond those of the underlying classifiers, and no ethical concern beyond those general to machine learning research on public benchmarks.

\section*{Reproducibility Statement}

\paragraph{Protocol.}
Every number reported in the paper follows the protocol of Appendix~\ref{app:repro}, which states the quantized checkpoints and calibration recipe, how the source per-channel moments are accumulated, the evaluation protocol, the seed protocol, and the hardware. \method{} runs one configuration everywhere: local scope blocks 8--11 $\times$ \{qkv, proj, fc1, fc2\}, $\lambda{=}0.5$, the cumulative moment estimate ($\alpha{=}0.1$ EMA for the continual setting only), clip ratio $\rho{=}3$ and $\epsilon{=}10^{-6}$; none of these is tuned per dataset, bit-width, or backbone. \method{} itself is deterministic given the test stream: the seeds $0/1/2$ of the three-seed tables vary only the stream order and the internal randomness of the optimization-based baselines, while the quantized checkpoint and its calibration are fixed at seed $0$.

\paragraph{Quantization and source statistics.}
We quantize with PTQ4ViT \citep{yuan2022ptq4vit} for W6A6/W8A8 and AdaLog \citep{wu2024adalog} for W4A4/W3A3, following the quantization protocol of ZOA \citep{deng2025test} and FOZO \citep{fozo2026}, and with AdaLog for the asymmetric W4A8 setting of Appendix~\ref{app:w4a8}. The source moments are accumulated in one pass over $12{,}800$ unlabeled clean source images (Appendix~\ref{app:repro}), so no target data, labels or gradients are used. The CIFAR- and DomainNet-fine-tuned models of Appendix~\ref{app:beyond} follow the same recipe, with calibration and source moments drawn from their own training splits.

\paragraph{Evaluation data.}
ImageNet-C evaluation uses the RobustBench \citep{croce2021robustbench} $5$k subset, a standard subset for efficient evaluation also used by our FOZO baseline; the online label-shift study of Section~\ref{sec:noniid} instead uses the full $50$k per corruption; CIFAR-10-C/100-C use the full $10$k per corruption and DomainNet-126 its three target domains in full.

\paragraph{Code and artifacts.}
Code, the launch scripts behind every table, and the source calibration statistics (${\sim}170$\,KB per bit-width) are available at \url{\codeurl}.
The quantized checkpoints are reproduced by the calibration scripts included there (PTQ4ViT with its fixed calibration set, AdaLog from its public release, and the fine-tuning and calibration recipes of Appendix~\ref{app:beyond} for the other benchmarks), listed with checksums in the repository, and will be added to it as well. Experiments use PyTorch 2.8 and timm 1.0 on CUDA 12.8; every dataset we evaluate on is a public benchmark.

\paragraph{Baselines, cost, and analysis.}
Baseline configurations are in Appendix~\ref{app:repro}. For a fair cost comparison every optimization-based baseline runs at its cheapest published two-forward setting, which is also ZOA's default protocol; Appendix~\ref{app:foa28} additionally reports FOA at its recommended population. Measured compute, peak memory and per-image forward latency, together with the measurement protocol, are in Appendix~\ref{app:cost}; they are taken on an NVIDIA TITAN RTX ($24$\,GB), and the on-device measurements of Appendix~\ref{app:ondevice} on a Jetson Orin Nano ($8$\,GB). The assumptions behind the analysis of Section~\ref{sec:why} and all of its proofs are in Appendix~\ref{app:theory}.

\bibliography{iclr2027_conference}
\bibliographystyle{iclr2027_conference}

\appendix
\setcounter{proposition}{0}\renewcommand{\theproposition}{\thesection.\arabic{proposition}}
\setcounter{lemma}{0}\renewcommand{\thelemma}{\thesection.\arabic{lemma}}
\setcounter{corollary}{0}\renewcommand{\thecorollary}{\thesection.\arabic{corollary}}
\newpage

\section{Experimental and reproducibility details}
\label{app:repro}

\textbf{Model and quantization.} All experiments use ViT-B/16 from \texttt{timm} \citep{dosovitskiy2021image} unless noted. W6A6 and W8A8 use PTQ4ViT \citep{yuan2022ptq4vit} with Hessian-guided calibration on $32$ calibration images; W3A3 and W4A4 use AdaLog \citep{wu2024adalog} checkpoints. Weights and activation quantizers are frozen at test time. For the backbone study (Appendix~\ref{app:backbones}), DeiT-B/16 and ViT-S/16 are freshly PTQ4ViT-quantized at W6A6 with the same recipe. The benchmarks beyond ImageNet (Appendix~\ref{app:beyond}) use ViT-B/16 backbones fine-tuned on their own source data, CIFAR-10, CIFAR-100 and DomainNet's \emph{real} domain, and quantized with this same recipe at all four bit-widths.

\textbf{Source statistics.} The source per-channel moments $(\mu_s,\sigma_s)$ at the $16$ local interfaces and at the CLS position are accumulated in a single pass over $200$ batches of $64$ ($\approx 12{,}800$ clean, uncorrupted ImageNet \citep{deng2009imagenet} validation images, the split from which the FOA, FOZO and ZOA source statistics are also computed under their published protocol, which uses the full set). No corrupted or labeled data are used. For CIFAR-10-C, CIFAR-100-C and DomainNet-126 the same single pass runs on the training split of the corresponding source dataset, so no test or target image is seen.

\textbf{Evaluation protocol.} ImageNet-C uses the RobustBench \citep{croce2021robustbench} $5{,}000$-image test subset at severity $5$ over all $15$ corruptions, batch size $64$. Following standard TTA practice, the source artifact is held fixed (the PTQ-quantized checkpoint and its source calibration statistics, calibrated once at seed 0); for the three-seed tables, seeds $0/1/2$ vary the online test-stream order and, for the optimization-based methods, their internal randomness (FOA's CMA-ES sampling, FOZO's and ZOA's zeroth-order perturbations). The reported standard deviations therefore capture each method's online-adaptation variance: it is within $0.01$ for No Adapt and small for the closed-form methods (T3A, NEO, and \method{}, which carry no internal randomness and depend on the stream only through their running statistics), while reflecting optimization noise for FOA/FOZO/ZOA. Because every method is evaluated on this identical fixed checkpoint, any one-time PTQ-calibration variance is common to all and does not change their relative ranking. Each corruption is adapted episodically (state reset between corruptions) except in the continual setting (Section~\ref{sec:continual}), where the running moments are reset only once at the start of the $10$-round stream. ImageNet-R uses its $200$-class subset; ImageNet-Sketch uses all $1{,}000$ classes. CIFAR-10-C and CIFAR-100-C use the full $10{,}000$ images per corruption over the same $15$ corruptions at severity $5$, with each $32\times32$ image resized to $224\times224$; DomainNet-126 uses the three target domains in full and resets between domains, with \texttt{Resize}$(256)$ and \texttt{CenterCrop}$(224)$ on both source and target images.

\textbf{Hardware.} All experiments run on NVIDIA TITAN RTX GPUs ($24$\,GB), one card per run; the cost measurements of Appendix~\ref{app:cost} are taken on an idle card with one measurement at a time. The on-device measurements of Appendix~\ref{app:ondevice} use a Jetson Orin Nano Super developer kit \citep{jetson_orin_nano} (Orin \texttt{sm\_87}, the $8$\,GB module, of which $7.31$\,GiB is shared between GPU and CPU).

\textbf{\method{} hyperparameters.} We use one configuration throughout: local scope blocks~8--11 $\times$ \{qkv, proj, fc1, fc2\} and the post-LN CLS position; local recalibration strength $\lambda=0.5$ (the global CLS branch uses the full recalibration, $\lambda{=}1$); a forgetting-free cumulative running moment estimate by default (a forgetting EMA at rate $\alpha=0.1$ for the continual setting only); clip ratio $\rho=3$; $\epsilon=10^{-6}$. None of these is tuned per dataset, bit-width, or backbone. The local branch corrects with the moments of the preceding batches and updates afterwards; the global branch updates first (Appendix~\ref{app:sens}).

\textbf{Baselines.} All baselines use their public configurations. T3A \citep{iwasawa2021test} uses a per-class support set of size $k=20$. The optimization-based baselines are run in their cheapest published two-forward-pass configuration: FOA \citep{niu2024test} with CMA-ES population $K{=}2$ (two forwards per batch); FOZO \citep{fozo2026} with one SPSA sample, perturbation $0.5$, learning rate $0.08$; ZOA \citep{deng2025test} with one step, one SPSA average, and its domain memory, perturbing the LayerNorm affine parameters its paper adapts, namely blocks 1--8 $\times$ \{norm1, norm2\}, with the first block, the last three blocks and the final norm held frozen as that paper specifies ($\approx$\,25K parameters), at its published perturbation and aggregation learning rates ($2{\times}10^{-4}$ and $10^{-2}$). NEO \citep{neo2025} is parameter-free latent recentering, and No Adapt is the frozen quantized model.

\textbf{Online label shift.} The non-i.i.d.\ study of Section~\ref{sec:noniid} follows the SAR protocol \citep{niu2023towards}. We order the full ImageNet-C stream ($50$k images per corruption, W6A6) by an imbalance ratio: at each step the sampler draws a class from a distribution peaked on one class, with probability $q_{\max}$ on it against $q_{\min}$ on each other class, and the ratio is $q_{\max}/q_{\min}$. A ratio of $1$ is i.i.d., $10^{3}$ gives a ${\sim}50\%$ per-batch major class, and $5{\times}10^{5}$ is near class-sorted.

\textbf{Code.} Code and configurations to reproduce every table and figure are available at \url{\codeurl}.

\section{Per-bit-width and per-corruption tables}
\label{app:tables}

Table~\ref{tab:t1percorr} gives the full per-corruption accuracy behind the main accuracy comparison (Table~\ref{tab:main}) at all four bit-widths. Table~\ref{tab:t2pc} is the per-corruption breakdown of the component ablation (Table~\ref{tab:main}), and Table~\ref{tab:t4app} gives ImageNet-R/Sketch accuracy across all bit-widths, and Table~\ref{tab:full50k} repeats the main accuracy comparison on the full ImageNet-C test set ($50$k images per corruption).

\begin{table}[t]
\centering
\caption{\textbf{Full ImageNet-C evaluation} ($50$k images per corruption, all $15$ corruptions,
sev.\ 5, quantized ViT-B, batch 64, episodic; 3-seed mean $\pm$ std over seeds $0/1/2$, No Adapt within $0.01$ across seeds). Mean top-1 (\%) by bit-width,
replicating the accuracy comparison of Table~\ref{tab:main} on the full test set rather than the
$5$k subset. Lower block: component ablation (Local-only, Global-only, \ours{}). \method{} is the
best method at every bit-width and the ranking matches the $5$k main table; the W6A6 column also
agrees with the i.i.d.\ column of Table~\ref{tab:noniid}. Best per column \textbf{bold}.}
\label{tab:full50k}
{\small
\begin{tabular}{l|rrrr}
\toprule
Method & W3A3 & W4A4 & W6A6 & W8A8 \\
\midrule
No Adapt & 28.26 & 47.24 & 47.60 & 54.21 \\
T3A & \ms{29.68}{0.01} & \ms{48.62}{0.00} & \ms{49.11}{0.00} & \ms{55.76}{0.01} \\
FOA & \ms{31.25}{0.07} & \ms{50.28}{0.06} & \ms{50.50}{0.22} & \ms{58.17}{0.00} \\
ZOA & \ms{32.06}{0.15} & \ms{51.10}{0.14} & \ms{52.91}{0.05} & \ms{58.69}{0.02} \\
NEO & \ms{31.73}{0.01} & \ms{50.32}{0.01} & \ms{51.43}{0.00} & \ms{58.06}{0.00} \\
FOZO & \ms{31.88}{0.04} & \ms{50.03}{0.08} & \ms{49.94}{0.26} & \ms{58.05}{0.15} \\
\midrule
\multicolumn{5}{l}{\emph{Component ablation of \ours{}}} \\
\rowcolor{ourslight}Local-only & \ms{35.56}{0.02} & \ms{53.03}{0.01} & \ms{53.74}{0.01} & \ms{59.49}{0.01} \\
\rowcolor{ourslight}Global-only & \ms{31.38}{0.00} & \ms{50.29}{0.00} & \ms{51.20}{0.00} & \ms{58.11}{0.00} \\
\rowcolor{oursgreen}\textbf{\method{}} & \ms{\mathbf{36.18}}{0.02} & \ms{\mathbf{53.79}}{0.01} & \ms{\mathbf{54.64}}{0.01} & \ms{\mathbf{60.82}}{0.01} \\
\bottomrule
\end{tabular}}
\end{table}

\begin{table*}[t]
\centering
\caption{Full per-corruption top-1 accuracy (\%) on ImageNet-C, by
quantization bit-width. \ours{} uses
$\lambda{=}0.5$; baselines as in Table~\ref{tab:main}; seed 0. Best per (bit-width, corruption)
in \textbf{bold}.}
\label{tab:t1percorr}
\resizebox{\textwidth}{!}{%
\begin{tabular}{l|l|rrrrrrrrrrrrrrr|r}
\toprule
BW & Method & Gauss & Shot & Impul. & Defoc. & Glass & Motion & Zoom & Snow & Frost & Fog & Bright. & Contr. & Elastic & Pixel & JPEG & Avg. \\
\midrule
 & No Adapt & 14.52 & 13.42 & 13.76 & 27.42 & 18.56 & 27.20 & 27.90 & 25.60 & 35.30 & 37.82 & 63.76 & 4.78 & 28.90 & 35.30 & 51.02 & 28.35 \\
 & T3A & 14.36 & 13.32 & 13.86 & 27.24 & 18.40 & 27.50 & 28.00 & 26.18 & 35.56 & 38.62 & 63.84 & 4.72 & 29.14 & 35.50 & 51.20 & 28.50 \\
 & FOA & 16.74 & 15.06 & 15.66 & 30.20 & 20.66 & 31.48 & 31.40 & 31.46 & 38.48 & 44.92 & 64.98 & 8.40 & 32.70 & 37.76 & 52.88 & 31.52 \\
 & ZOA & 16.56 & 15.34 & 16.12 & 29.48 & 20.92 & 30.64 & 31.24 & 32.58 & 37.90 & 43.86 & 64.94 & 7.90 & 33.20 & 38.30 & 52.04 & 31.40 \\
 & NEO & 16.98 & 16.16 & 16.52 & 30.84 & 20.98 & 31.14 & 31.26 & 31.98 & 39.12 & 45.04 & 65.02 & 8.00 & 33.80 & 38.30 & 53.38 & 31.90 \\
 & FOZO & 17.72 & 16.32 & 16.92 & 28.78 & 21.40 & 31.90 & 32.14 & 33.18 & 39.94 & 43.20 & 65.24 & 7.40 & 34.90 & 39.54 & 53.70 & 32.15 \\
\rowcolor{oursgreen}\cellcolor{white}\multirow{-7}{*}{W3A3} & \textbf{\method{}} & \textbf{20.94} & \textbf{20.48} & \textbf{20.82} & \textbf{31.88} & \textbf{25.14} & \textbf{35.02} & \textbf{35.84} & \textbf{38.52} & \textbf{43.22} & \textbf{45.60} & \textbf{66.78} & \textbf{10.30} & \textbf{44.20} & \textbf{45.48} & \textbf{55.50} & \textbf{35.98} \\
\midrule
 & No Adapt & 45.30 & 45.36 & 44.98 & 41.34 & 29.54 & 43.98 & 37.24 & 49.12 & 57.12 & 60.94 & 74.66 & 21.64 & 38.46 & 56.14 & 63.70 & 47.30 \\
 & T3A & 45.18 & 45.12 & 44.98 & 41.18 & 29.58 & 44.10 & 37.32 & 49.76 & 57.02 & 61.20 & 74.58 & 22.58 & 38.72 & 56.24 & 63.76 & 47.42 \\
 & FOA & 46.02 & 45.92 & 46.46 & 43.36 & 31.28 & 46.16 & 40.96 & 53.86 & 58.34 & \textbf{64.12} & 74.74 & 36.42 & 43.30 & 57.42 & 64.62 & 50.20 \\
 & ZOA & \textbf{48.18} & 47.46 & \textbf{48.54} & 43.72 & 32.82 & 47.30 & 40.68 & 53.92 & 58.46 & 63.96 & 74.82 & 35.98 & 43.94 & 57.48 & 64.38 & 50.78 \\
 & NEO & 46.10 & 46.26 & 46.40 & 43.74 & 32.08 & 46.58 & 40.04 & 53.90 & 59.08 & 63.88 & 75.26 & 36.42 & 42.70 & 57.28 & 64.70 & 50.29 \\
 & FOZO & 45.78 & 45.22 & 45.80 & 43.20 & 32.82 & 47.52 & 41.62 & 54.80 & 59.24 & 61.38 & 74.52 & 32.52 & 45.04 & 57.44 & 64.54 & 50.10 \\
\rowcolor{oursgreen}\cellcolor{white}\multirow{-7}{*}{W4A4} & \textbf{\method{}} & 47.06 & \textbf{47.84} & 47.46 & \textbf{45.14} & \textbf{37.46} & \textbf{51.22} & \textbf{46.96} & \textbf{58.04} & \textbf{61.26} & 62.88 & \textbf{76.54} & \textbf{37.68} & \textbf{55.30} & \textbf{61.24} & \textbf{66.98} & \textbf{53.54} \\
\midrule
 & No Adapt & 44.24 & 43.14 & 44.54 & 38.30 & 30.62 & 43.10 & 35.30 & 53.62 & 60.44 & 56.22 & 75.36 & 26.66 & 39.34 & 59.90 & 64.50 & 47.69 \\
 & T3A & 44.12 & 43.12 & 44.38 & 38.30 & 30.78 & 43.14 & 35.40 & 54.04 & 60.32 & 58.36 & 75.32 & 25.36 & 39.72 & 59.92 & 64.72 & 47.80 \\
 & FOA & 44.50 & 42.78 & 44.54 & 41.90 & 34.44 & 46.44 & 39.00 & 57.34 & 61.94 & 65.22 & 75.90 & 32.76 & 44.22 & 59.90 & 65.76 & 50.44 \\
 & ZOA & 45.96 & 44.84 & 46.88 & 42.92 & 35.32 & 45.70 & 38.78 & 58.30 & 62.94 & 66.30 & 76.08 & 34.60 & 44.06 & 60.84 & 66.66 & 51.35 \\
 & NEO & 45.52 & 44.52 & 46.40 & 42.36 & 35.04 & 45.80 & 39.66 & 57.90 & 62.98 & \textbf{66.74} & 75.48 & 36.50 & 44.20 & 61.14 & 66.12 & 51.36 \\
 & FOZO & 43.38 & 43.44 & 43.44 & 41.50 & 35.90 & 46.04 & 39.70 & 56.34 & 62.38 & 61.84 & 75.28 & 30.34 & 45.28 & 60.14 & 65.90 & 50.06 \\
\rowcolor{oursgreen}\cellcolor{white}\multirow{-7}{*}{W6A6} & \textbf{\method{}} & \textbf{46.80} & \textbf{45.74} & \textbf{47.26} & \textbf{44.88} & \textbf{40.64} & \textbf{49.60} & \textbf{46.00} & \textbf{62.36} & \textbf{64.04} & 66.32 & \textbf{76.90} & \textbf{37.46} & \textbf{55.26} & \textbf{63.56} & \textbf{68.68} & \textbf{54.37} \\
\midrule
 & No Adapt & 56.34 & 55.98 & 56.44 & 46.56 & 34.28 & 52.44 & 43.64 & 61.14 & 62.20 & 67.32 & 77.04 & 25.80 & 43.78 & 65.72 & 66.70 & 54.36 \\
 & T3A & 56.28 & 55.94 & 56.32 & 46.56 & 34.56 & 52.40 & 43.74 & 61.58 & 61.54 & 66.26 & 77.02 & 23.78 & 43.98 & 65.72 & 66.86 & 54.17 \\
 & FOA & 56.22 & 55.42 & 56.74 & 48.28 & 36.34 & 53.08 & 46.10 & 63.38 & 64.72 & \textbf{71.24} & 77.78 & 54.80 & 47.52 & 66.06 & 68.12 & 57.72 \\
 & ZOA & 57.20 & 56.54 & 57.60 & \textbf{49.66} & 37.12 & 53.88 & 45.88 & 63.30 & 64.34 & 70.78 & 78.00 & 56.10 & 48.20 & 66.14 & 68.48 & 58.21 \\
 & NEO & 56.92 & 56.26 & 57.18 & 49.28 & 36.64 & 53.74 & 45.72 & 63.82 & 64.32 & 70.74 & 77.34 & \textbf{57.08} & 48.02 & 66.48 & 67.94 & 58.10 \\
 & FOZO & 57.18 & 55.56 & 56.92 & 48.76 & 37.28 & 54.30 & 47.40 & 64.28 & 65.38 & 68.36 & 77.82 & 50.52 & 50.00 & 66.54 & 68.16 & 57.90 \\
\rowcolor{oursgreen}\cellcolor{white}\multirow{-7}{*}{W8A8} & \textbf{\method{}} & \textbf{57.70} & \textbf{56.82} & \textbf{58.08} & 49.64 & \textbf{41.60} & \textbf{57.16} & \textbf{51.82} & \textbf{66.60} & \textbf{65.66} & 70.16 & \textbf{79.18} & 53.90 & \textbf{60.04} & \textbf{68.00} & \textbf{70.14} & \textbf{60.43} \\
\bottomrule
\end{tabular}%
}
\end{table*}

\begin{table*}[t]
\centering
\caption{Component ablation of \method{}: per-corruption accuracy on ImageNet-C
sev.\ 5 (W6A6 PTQ4ViT ViT-B, 5k subset, bs=64, seed 0, episodic reset;
cumulative running target moments, lagged). Companion to Table~\ref{tab:main} (lower block).}
\label{tab:t2pc}
\resizebox{\textwidth}{!}{%
\begin{tabular}{l|rrrrrrrrrrrrrrr|r}
\toprule
Method & Gauss & Shot & Impul. & Defoc. & Glass & Motion & Zoom & Snow & Frost & Fog & Bright. & Contr. & Elastic & Pixel & JPEG & Avg. \\
\midrule
  \rowcolor{ourslight}Local-only & 46.48 & 45.06 & 47.24 & 43.98 & 39.68 & 48.98 & 45.12 & 61.82 & 63.14 & 63.72 & 76.44 & 33.80 & 54.24 & 63.38 & 68.46 & 53.44 \\
  \rowcolor{ourslight}Global-only & 45.40 & 44.30 & 46.00 & 42.02 & 34.42 & 45.92 & 39.28 & 57.92 & 62.70 & \textbf{66.64} & 75.60 & 36.26 & 43.66 & 61.36 & 65.92 & 51.16 \\
  \rowcolor{oursgreen}\textbf{\method{}} & \textbf{46.80} & \textbf{45.74} & \textbf{47.26} & \textbf{44.88} & \textbf{40.64} & \textbf{49.60} & \textbf{46.00} & \textbf{62.36} & \textbf{64.04} & 66.32 & \textbf{76.90} & \textbf{37.46} & \textbf{55.26} & \textbf{63.56} & \textbf{68.68} & \textbf{54.37} \\
\bottomrule
\end{tabular}%
}
\end{table*}

\section{Sensitivity to the recalibration strength, batch size, source-moment estimation, and clamp ratio}
\label{app:sens}

Table~\ref{tab:sens} sweeps the recalibration strength $\lambda$, the test batch size, and the source-moment set size at W6A6 (B3 scope, ImageNet-C, 3-seed 0/1/2, episodic), all with the default forgetting-free cumulative moment estimate; Table~\ref{tab:sens_full} repeats the sweep with per-cell standard deviations. Because that estimate averages moments over all batches seen and never forgets, the method is essentially \emph{batch-size invariant}: accuracy stays within $0.2$ points across every batch size from $1$ to $128$ ($54.27$ at $1$, $54.47$ at $8$, $54.35$ at $64$, $54.29$ at $128$), so it does not collapse at small batches the way a noisier per-batch or fixed-rate estimate would. Even at batch size $1$ the local per-channel moments are still taken over the $N$ tokens of the single image, so $\sigma_t^\ell$ remains well-defined; only the CLS branch sees a single token per image there, and its variance ratio is kept finite by the clamp (Section~\ref{sec:method}). The frozen source moments $(\boldsymbol\mu_s,\boldsymbol\sigma_s)$ are estimated once, offline, from clean source batches, and accuracy depends only weakly on how many once there are a few hundred images; with $8$ batches ($512$ images) it is $52.90$, within about $1.5$ points of the default, and it saturates near the default level from $64$ batches onward ($54.18$ at $64$, $54.29$ at $128$, $54.33$ at $256$). The moments do not depend on the evaluation split either, since the same $12{,}800$ images drawn from the training set give $54.14$, matching the validation-set default to within $0.2$ points. The variance-ratio clamp $\rho$ is a stability rail rather than a tuned knob: sweeping $\rho\in\{2,3,5,\infty\}$ (the last removes the clamp) moves mean accuracy monotonically over a $0.5$-point range ($54.05/54.37/54.45/54.54$), so the default $\rho{=}3$ gives up at most $0.2$ points relative to no clamp while still bounding the applied affine gain over long continual streams (Proposition~\ref{prop:drift}). This benign-regime cost is repaid exactly where the clamp is meant to act. At the deployed batch size the per-channel target variance is well estimated and the clamp rarely binds; but when the per-batch token count is small or the stream is strongly class-concentrated, individual channels collapse toward a constant ($\sigma_{t,c}\!\to\!0$) and the unclamped ratio $\sigma_{s,c}/\sigma_{t,c}$ diverges. Sweeping batch size and stream order under the cumulative estimator makes the crossover explicit (Table~\ref{tab:clamp}): removing the clamp changes accuracy by at most $0.2$ points at batch sizes $64$ down to $4$, but costs $1.0$ point at batch size $1$ on an i.i.d.\ stream and $1.6$ points at batch size $1$ under a class-sorted (online-label-shift) stream. The batch-size invariance reported above relies on the clamp at the smallest batches: without it, batch-$1$ accuracy falls from $54.31$ to $53.29$. The clamp is therefore not free on the development benchmark, but it is what keeps the recalibration safe at small batch, under label shift, and over long continual streams.

\begin{table}[t]
\centering
\caption{Full \method{} sensitivity sweep, 3-seed mean$\pm$std over seeds $0/1/2$ (W6A6 ViT-B, ImageNet-C 15-corruption 5k, cumulative estimator; expanded from Table~\ref{tab:sens}). Deployed defaults in \textbf{bold}; per-cell std $\leq 0.08$.}
\label{tab:sens_full}
\footnotesize
\resizebox{\textwidth}{!}{%
\begin{tabular}{l|cccccc}
\toprule
Variance match $\lambda$ & 0 & 0.1 & 0.3 & \textbf{0.5} & 0.7 & 1.0 \\
Mean top-1 (\%) & \ms{51.20}{0.02} & \ms{52.61}{0.01} & \ms{54.00}{0.03} & \ms{\mathbf{54.35}}{0.02} & \ms{54.20}{0.06} & \ms{53.52}{0.01} \\
\midrule
Batch size & 1 & 4 & 8 & 16 & \textbf{64} & 128 \\
Mean top-1 (\%) & \ms{54.27}{0.04} & \ms{54.39}{0.04} & \ms{54.47}{0.02} & \ms{54.35}{0.01} & \ms{\mathbf{54.35}}{0.02} & \ms{54.29}{0.06} \\
\midrule
Source-stat batches & 16 & 32 & 64 & 128 & \textbf{200} & 256 \\
Mean top-1 (\%) & \ms{53.65}{0.02} & \ms{53.88}{0.08} & \ms{54.18}{0.02} & \ms{54.29}{0.06} & \ms{\mathbf{54.35}}{0.02} & \ms{54.33}{0.04} \\
\bottomrule
\end{tabular}}
\end{table}

\begin{table}[t]
\centering
\caption{\emph{Variance-ratio clamp ablation} (W6A6, 15-corruption mean, cumulative estimator, seed 0): default clamp $\rho{=}3$ vs.\ no clamp $\rho{=}\infty$, across batch size and stream order (i.i.d.\ vs.\ class-sorted). The clamp is essentially free at batch $\ge 4$ ($\le 0.2$ points), but at batch $1$ recovers $+1.0$ (i.i.d.) and $+1.6$ (class-sorted) points, where few tokens or class concentration drive per-channel $\sigma_{t,c}\!\to\!0$ and the unclamped ratio diverges; the estimator's batch-size invariance thus relies on it.}
\label{tab:clamp}
\footnotesize
\begin{tabular}{c|cc|cc}
\toprule
& \multicolumn{2}{c|}{i.i.d.\ stream} & \multicolumn{2}{c}{class-sorted stream} \\
Batch & $\rho{=}3$ & $\rho{=}\infty$ & $\rho{=}3$ & $\rho{=}\infty$ \\
\midrule
$64$ & $54.37$ & $54.54$ & $54.18$ & $54.37$ \\
$8$  & $54.49$ & $54.51$ & $53.36$ & $53.19$ \\
$4$  & $54.35$ & $54.40$ & $52.70$ & $52.53$ \\
$1$  & $\mathbf{54.31}$ & $53.29$ & $\mathbf{51.86}$ & $50.26$ \\
\bottomrule
\end{tabular}
\end{table}

\paragraph{Estimator choice.} The target moments in Equation~\ref{eq:ema} are a running average of the
per-batch moments: its running mean equals the pooled mean over all tokens, while its running standard deviation averages the per-batch standard deviations. A natural alternative is a streaming \emph{Welford} estimate, which tracks running raw moments and gives the exact pooled variance over all tokens seen (including the between-batch component). Table~\ref{tab:estimator} swaps the deployed estimator for this alternative in the local branch, the global branch, or both, at batch sizes $64$ and $1$. All four variants fall within about $0.15$ points of one another; at batch size $64$ they are statistically tied (spread $0.02$, per-cell std ${\sim}0.02$), and at batch size $1$ the deployed estimator is the best. The estimator is therefore not a sensitive knob: the local branch already pools $B\!\cdot\!N$ tokens, so its per-channel statistics are well estimated even at batch size $1$, and the variance-ratio clamp bounds the only branch (CLS) whose per-batch estimate can degenerate at batch size $1$. Since the exact pooled (Welford) estimate offers no empirical advantage, we adopt the simpler running form; it additionally covers both the cumulative (stationary) and EMA (continual) regimes with a single update, whereas a streaming Welford variance would fall back to the exponentially weighted form under drift, which conflates the moving mean with the spread.

\begin{table}[t]
\centering
\caption{\emph{Target-moment estimator ablation} (W6A6, 15-corruption mean top-1 \%, ImageNet-C 5k
subset, 3-seed mean $\pm$ std over seeds $0/1/2$). Deployed \emph{cumulative} estimator
vs.\ a streaming Welford (exact pooled) estimate, applied to the
local branch, the global (CLS) branch, or both, at batch sizes $64$ and $1$. All within ${\sim}0.15$
points; at batch size $64$ the four variants are statistically tied, and at batch size $1$ the deployed
estimator is best. Deployed configuration \textbf{bold}.}
\label{tab:estimator}
{\small
\begin{tabular}{ll|cc}
\toprule
Local est. & Global est. & bs $64$ & bs $1$ \\
\midrule
\textbf{cumulative} & \textbf{cumulative} & \ms{\mathbf{54.35}}{0.02} & \ms{\mathbf{54.27}}{0.04} \\
Welford & Welford & \ms{54.33}{0.03} & \ms{54.21}{0.02} \\
Welford & cumulative & \ms{54.33}{0.02} & \ms{54.22}{0.01} \\
cumulative & Welford & \ms{54.34}{0.03} & \ms{54.25}{0.04} \\
\bottomrule
\end{tabular}}
\end{table}

\paragraph{How much the target moments need to be pooled.} Equation~\ref{eq:ema} pools the per-channel moments across the batches of a domain. A local interface could instead be corrected with the moments of the batch at hand, at no cost in forward passes, since the hook sees the whole batch before it corrects. This is a different question from when each branch updates (below), where the current batch joins the running average rather than replacing it. Table~\ref{tab:pooling} answers it on both sides, varying the local branch only and leaving the global branch at its deployed setting. Correcting the local interfaces with the current batch's own moments costs $2.0$ points at batch size $64$ and $16.2$ at batch size $1$, and under label shift a class-dominated batch normalizes itself, falling to $42.39$ and $23.89$ at imbalance $10^{3}$ and $5{\times}10^{5}$; re-selecting $\lambda$ for that estimator recovers little ($52.76$ at $\lambda{=}0.25$), and a full match of those moments is worse still ($48.85$ at $\lambda{=}1$). Pooling across batches averages the per-channel noise down and keeps the correction a fixed map that the batch composition cannot move. In the other direction, fixing the moments at the values they converge to over a full extra pass through the same stream stays within $0.3$ points of the deployed estimate at every setting, and is below it under the strongest label shift. Moment quality is therefore not what limits accuracy.

\begin{table}[t]
\centering
\caption{\emph{How much the target moments need to be pooled} (W6A6 ViT-B, 15-corruption mean
top-1 \%). The local branch's target moments are the deployed running estimate, the moments of the
batch at hand, or the values the running estimate converges to over a full extra pass through the
stream; the global branch is the deployed estimate in every row. Correcting with the batch at hand
needs no extra forward, since the hook sees the whole batch before it corrects. Batch size $64$ and $1$: ImageNet-C 5k subset, 3-seed mean $\pm$ std, except the
converged row, which is seed 0. Label shift: full $50$k images per corruption, SAR imbalance ratio
$10^{3}$ and $5{\times}10^{5}$, seed 0. Deployed configuration \textbf{bold}.}
\label{tab:pooling}
{\small
\begin{tabular}{l|cc|cc}
\toprule
& \multicolumn{2}{c|}{i.i.d.\ stream} & \multicolumn{2}{c}{label shift} \\
Local target moments & bs $64$ & bs $1$ & $10^{3}$ & $5{\times}10^{5}$ \\
\midrule
Current batch only & \ms{52.37}{0.01} & \ms{38.11}{0.01} & $42.39$ & $23.89$ \\
\rowcolor{oursgreen}\textbf{Running, cumulative (deployed)} & \ms{\mathbf{54.35}}{0.02} & \ms{\mathbf{54.27}}{0.04} & $\mathbf{54.72}$ & $\mathbf{54.82}$ \\
Converged over an extra pass & $54.43$ & $54.53$ & $54.74$ & $54.66$ \\
\bottomrule
\end{tabular}}
\end{table}

\paragraph{When each branch updates.} The two branches differ in when they update their running moments, for the same positional reason that lets the global branch match fully (Section~\ref{sec:method}): what lies downstream of each correction. The CLS representation is terminal, consumed only by the linear head, so an error in its running moments becomes one per-dimension affine map of the CLS features shared by every sample in the batch, hence a fixed linear distortion of the head plus a class-prior offset, and no later layer compounds it; empirically the global branch is insensitive to it (the global-lagged row of Table~\ref{tab:branchlag} is within $0.04$ of the deployed design on the full stream and at batch size 1, and within $0.16$ under label shift). The global branch therefore updates first and corrects with moments that include the current batch: the correction is informative from the very first batch of a domain. A local interface instead feeds the remaining blocks and quantizers, which carry a statistics error to the logits with the downstream gain $K_{>\ell}$ (Corollary~\ref{prop:depth}) and do so differently for every token, so a class-skewed batch fed into its own moments would distort its own predictions. The local branch therefore corrects with the moments of the preceding batches and updates afterwards. Table~\ref{tab:branchlag} confirms both halves: making the global branch lag costs $2.60$ points on the first batch, making the local branch inclusive costs $0.84$ and $2.85$ points under label shift, and at convergence the three designs agree within $0.04$. The cost is a one-batch warm-up: on the first batch of a domain only the global branch acts, and after five batches accuracy is within about one point of the full-stream value.

\begin{table}[t]
\centering
\caption{\textbf{When each branch updates} (W6A6 ViT-B, ImageNet-C sev.\ 5, 15 corruptions). \emph{Deployed}: the global branch corrects with running moments that include the current batch, the local branch with the moments of the preceding batches only. The two alternatives make the branches symmetric. First batch ($K{=}1$), first five batches ($K{=}5$) and the full stream ($K{=}79$): $5$k subset, batch 64, 3-seed mean. Label shift: full $50$k images per corruption, imbalance ratio $10^{3}$ and $5{\times}10^{5}$, seed 0. Batch size 1: $5$k subset, seed 0. Mean top-1 (\%).}
\label{tab:branchlag}
{\small
\begin{tabular}{l|rrr|rr|r}
\toprule
& \multicolumn{3}{c|}{i.i.d.\ stream} & \multicolumn{2}{c|}{label shift} & \\
Design & $K{=}1$ & $K{=}5$ & full ($K{=}79$) & $10^{3}$ & $5{\times}10^{5}$ & bs 1 \\
\midrule
Deployed (global inclusive, local lagged) & 51.11 & 53.33 & \textbf{54.35} & 54.72 & 54.82 & \textbf{54.31} \\
Global branch also lagged & 48.51 & 52.59 & 54.31 & \textbf{54.75} & \textbf{54.97} & 54.30 \\
Local branch also inclusive & \textbf{53.09} & \textbf{53.47} & 54.33 & 53.88 & 51.96 & 54.24 \\
\bottomrule
\end{tabular}}
\end{table}

\paragraph{Branch and recalibration-mode decomposition.} Table~\ref{tab:factorial} varies each
branch independently over $\{$off, mean-only, mean+var$\}$ (W6A6, 5k subset, 3 seeds). Over the global branch alone ($51.20$), a mean-only local recalibration adds $+0.85$ (to $52.05$) and the local scale match a further $+2.30$ (to $54.35$), so the local branch's gain is primarily the restored quantizer scale rather than the recentering that the global branch (and NEO/FOA) already provides. Full \method{} (both branches at mean+var) is best at both batch sizes; dropping the local variance match (local mean-only) or either branch costs accuracy.

\begin{table}[t]
\centering
\caption{\emph{Branch $\times$ recalibration-mode factorial} (W6A6, 15-corruption mean top-1 \%,
ImageNet-C 5k subset, 3-seed mean$\pm$std over seeds $0/1/2$; cumulative estimator). Each branch is
off, \emph{mean} (recenter only), or \emph{mean+var} (recenter and rescale). \method{} is (local mean+var, global mean+var), \textbf{bold}; (off, off)
is No Adapt.}
\label{tab:factorial}
{\small
\begin{tabular}{ll|cc}
\toprule
Local & Global & bs $64$ & bs $1$ \\
\midrule
off       & off       & 47.68 & 47.75 \\
mean      & off       & \ms{50.55}{0.02} & \ms{50.56}{0.01} \\
mean+var  & off       & \ms{53.42}{0.04} & \ms{53.38}{0.04} \\
off       & mean      & \ms{51.15}{0.02} & \ms{49.61}{0.08} \\
off       & mean+var  & \ms{51.20}{0.02} & \ms{51.25}{0.01} \\
mean      & mean      & \ms{51.78}{0.00} & \ms{50.22}{0.01} \\
mean      & mean+var  & \ms{52.05}{0.02} & \ms{52.01}{0.01} \\
mean+var  & mean      & \ms{54.27}{0.02} & \ms{52.90}{0.05} \\
\rowcolor{oursgreen}\textbf{mean+var} & \textbf{mean+var} & \ms{\mathbf{54.35}}{0.02} & \ms{\mathbf{54.27}}{0.04} \\
\bottomrule
\end{tabular}}
\end{table}

\section{Generalization to other backbones}
\label{app:backbones}

\begin{figure}[t]
\centering
\includegraphics[width=0.45\textwidth]{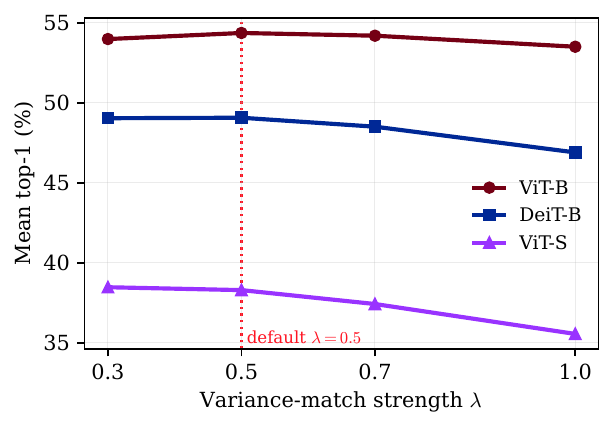}
\vspace{-10pt}
\caption{Effect of the recalibration strength $\lambda$ on W6A6 accuracy across three
backbones (ViT-B, DeiT-B, ViT-S; seed 0); the default $\lambda{=}0.5$ is marked dotted.}
\label{fig:lambda}
\end{figure}

To test generality beyond ViT-B, we apply \method{} to two further backbones at W6A6, DeiT-B and ViT-S, each freshly quantized with PTQ4ViT (their clean W6A6 accuracy is within $3$--$4$ points of full precision, comparable to ViT-B). Table~\ref{tab:backbones} compares against every baseline of Table~\ref{tab:main}. With the same configuration as the main experiments ($\lambda=0.5$), \method{} is the best method on all three backbones, and seed-stable (per-backbone std $\leq 0.06$, against up to $0.07$ for FOA and $0.65$ for FOZO). This is where the partial recalibration matters: the full match ($\lambda=1$) over-corrects the per-channel scale on these backbones, most visibly on the noise corruptions, and the local branch then fails to help; the gentler $\lambda=0.5$ recovers a consistent gain (Figure~\ref{fig:lambda}; e.g.\ ViT-S $35.55 \to 38.29$ at seed $0$, ahead of NEO by $+2.2$ in Table~\ref{tab:backbones}) while being neutral-to-better on ViT-B. The local recalibration thus transfers across backbones once its strength is moderated.

\section{Applying \method{} to a convolutional backbone (CNN)}
\label{app:cnn}

This paper targets quantized ViTs, but \method{} is defined at a frozen activation quantizer rather than at anything specific to a transformer, so it should carry over to a convolutional network. The premise carries over first. A frozen activation quantizer has its range fixed to the source distribution, a test-time shift moves the per-channel activations within that range, and the network's own normalization does not undo it, since GroupNorm acts within each sample across channel groups whereas the drift that mis-ranges the quantizer is per channel across the batch. The two act along different axes, exactly as Section~\ref{sec:method} argues for LayerNorm and tokens. Both branches then carry over by naming the convolutional counterparts of the tensors they precede: the local branch recalibrates the activations entering the quantizers of the last stage (\texttt{layer4}, 8 activation quantizers here), and the global branch recalibrates the pooled feature entering the classifier, the analogue of the CLS representation at the head input. Per-channel statistics are reduced over batch, height and width, the conv counterpart of the reduction over batch and tokens, and we transfer the placement rule of Section~\ref{app:scope}, recalibrate the last stage, rather than re-deriving a scope here.

We use the ResNet-50 variant that prior test-time adaptation work evaluates on \citep{niu2023towards,niu2024test,spa2025} (timm \texttt{resnet50\_gn}, which normalizes with GroupNorm). PTQ4ViT is designed for transformers, so we quantize with MQBench \citep{mqbench2021}, the standard post-training quantization toolkit for convolutional networks, at W6A6 on its Academic backend with 32 calibration batches and its default observers. MQBench inserts each activation quantizer as its own module, so \method{} attaches to precisely the tensor entering a frozen quantizer, with no change to the quantizer itself. The quantized model reaches $79.44$ clean top-1 against $80.25$ in full precision.

On ImageNet-C at severity 5 over all 15 corruptions (5k subset, batch 64, 3 seeds), the deployed configuration, $\lambda{=}0.5$ selected on ViT-B and left unchanged, lifts accuracy from $30.12$ to \ms{31.40}{0.05} and leads both parameter-free baselines, NEO at \ms{30.79}{0.00} and T3A at \ms{28.05}{0.02}. T3A landing below No Adapt is the behaviour Table~\ref{tab:main} already reports on ViT, where it stays within $0.15$ of No Adapt at the three lower bit-widths and falls $0.20$ below it at W8A8.

The recalibration strength transfers in form but not in value. Sweeping it over the same corruptions (seed 0) gives $33.86$, $\mathbf{33.91}$, $32.67$, $31.38$, $28.59$ and $26.54$ at $\lambda = 0.1$, $0.25$, $0.4$, $0.5$, $0.75$ and $1.0$, so the curve is single-peaked with its optimum at $0.25$, where a 3-seed run reaches \ms{33.91}{0.01}, a gain of $3.79$ points. This continues the trend already visible across transformers, where ViT-B peaks at $0.5$ while the smaller ViT-S and distilled DeiT-B peak at $0.3$ (Section~\ref{sec:hpsens}). A full match stays harmful, $\lambda=1$ falling to $26.54$, which is $3.58$ below No Adapt and is the shrinkage effect of Proposition~\ref{prop:shrink} appearing again away from transformers. The component structure reproduces as well: at $\lambda=0.25$ the global branch alone reaches $32.66$, the local branch alone \ms{33.25}{0.06}, and the two together \ms{33.91}{0.01}, so both branches contribute, the local branch is the larger driver, and the combination is best, which is the pattern of the ViT ablation in Section~\ref{sec:ablation}. What carries over is therefore not only the sign of the gain but the shape of the method.

The same configuration applies unchanged to a BatchNorm ResNet-50, where it lifts accuracy from $17.32$ to \ms{24.57}{0.06} on the same 15 corruptions. We note this only to show that nothing in the method depends on the normalization layer. A BatchNorm network additionally stores per-channel running statistics, which opens adaptation routes that a GroupNorm backbone does not have and that are outside the scope of this paper. FORGE \citep{forge2026} takes that route on folded integer-only CNNs.

\section{Recalibration scope}
\label{app:scope}

The scope ablation in the main text (Table~\ref{tab:scope}) is run in the Local-only configuration (no global CLS branch; ImageNet-C severity 5, 15 corruptions, 5k subset, batch size 64, 3-seed (0/1/2), episodic; No Adapt is 47.68 at W6A6 and 47.30 at W4A4); here we give the full reasoning behind it. Four observations drive the scope choice. (i) Recalibrating the full network collapses accuracy (31.93 at W6A6, vs.\ $53.42$ for the late-block default), because the affine corrections compound block by block; the gain comes from recalibrating only the late blocks. (ii) Within the late-block region accuracy rises with the number of trailing blocks up to four and then \emph{plateaus}: blocks 8--11 reach $53.42$, and extending to blocks 7--11, 6--11, or 5--11 does not improve on it (each within $0.15$ points, none exceeding it) before it declines, while each added block recalibrates four more interfaces. We therefore fix the \emph{smallest} scope on this plateau, blocks 8--11 (16 interfaces): the plateau accuracy at the lowest recalibration overhead. (iii) Including the post-GELU \texttt{mlp.fc2} interface helps at every bit-width: at the blocks 8--11 default it adds $+1.56$ (W6A6) and $+1.60$ (W4A4) over the \{qkv, proj, fc1\} set (module axis), $+2.33$ and $+1.10$ at W3A3 and W8A8 (Table~\ref{tab:scope4bw}), and the gain holds at every multi-block depth (Table~\ref{tab:fc2depth}). (iv) Recalibration must target the \emph{late} blocks: applying the same four-block, all-module recalibration to the early (blocks 0--3) or middle (blocks 4--7) blocks collapses accuracy to $32.12$ and $43.29$, far \emph{below} No Adapt ($47.68$), whereas the late blocks (8--11) reach $53.42$. A correction inserted early must propagate through every subsequent quantized layer, which re-quantizes and compounds it (the mechanism behind the all-blocks collapse); a late correction sits just before the head and directly restores the representation the classifier consumes, with few downstream layers to compound through. Corollary~\ref{prop:depth} formalizes this: a correction's residual at interface $\ell$ reaches the logits scaled by the downstream gain $K_{>\ell}$, which we measure to fall ${\sim}4.4\times$ from the early to the late blocks (Figure~\ref{fig:dgain}). We therefore recalibrate the last four blocks at all four interfaces. We note that \texttt{mlp.fc2}, whose pre-quant statistics a calibration-reliability check flags as atypical (large dynamic range, most values near a quantizer boundary), is nonetheless beneficial to recalibrate: excluding it, as an earlier four-corruption pilot had suggested, leaves accuracy on the table.

\begin{table}[t]
\centering
\caption{Full per-depth contrast of recalibrating the post-GELU \texttt{mlp.fc2} interface
(Local-only, W6A6; mean top-1 \%, seed 0). Each row recalibrates
the trailing blocks at \{qkv, proj, fc1\}, without (w/o
\texttt{fc2}) vs.\ with the \texttt{fc2} interface added; $\Delta$ is the \texttt{fc2} gain.
The deployed blocks 8--11 configuration is in \textbf{bold}.}
\label{tab:fc2depth}
{\small
\begin{tabular}{l|rr|r}
\toprule
Trailing blocks & w/o \texttt{fc2} & with \texttt{fc2} & $\Delta$ \\
\midrule
1 (block 11) & 49.08 & 48.55 & $-0.53$ \\
2 (blocks 10--11) & 50.88 & 51.48 & $+0.60$ \\
3 (blocks 9--11) & 51.34 & 52.59 & $+1.25$ \\
4 (blocks 8--11) & 51.85 & \textbf{53.44} & $+1.59$ \\
5 (blocks 7--11) & 52.10 & 53.39 & $+1.29$ \\
6 (blocks 6--11) & 51.79 & 53.31 & $+1.52$ \\
7 (blocks 5--11) & 51.46 & 53.37 & $+1.91$ \\
all 12 blocks & 31.86 & 31.86 & $0.00$ \\
\bottomrule
\end{tabular}}
\end{table}

\paragraph{The scope ablation at all four bit-widths.}
Table~\ref{tab:scope} reports the ablation at W4A4 and W6A6; Table~\ref{tab:scope4bw} repeats every
setting at W3A3 and W8A8 under the same recipe (Local-only, $\lambda{=}0.5$, cumulative, 15
corruptions, 5k subset, batch 64, episodic, seeds 0/1/2). The deployed scope reproduces the
Local-only row of Table~\ref{tab:main} at each bit-width, $35.30$ and $59.11$ at W3A3 and W8A8.
The three conclusions hold throughout. Recalibrating the early or middle blocks falls below No Adapt
at every bit-width, by $7.3$ and $2.1$ points at W3A3 and by $17.2$ and $4.5$ at W8A8; recalibrating
all twelve blocks collapses at every bit-width, to $22.47$ at W3A3 and $32.64$ at W8A8; and the
post-GELU \texttt{mlp.fc2} interface helps at every bit-width, adding $2.33$, $1.60$, $1.56$ and
$1.10$ points from W3A3 to W8A8, its contribution growing as the grid coarsens.
Along the depth axis the four trailing blocks we deploy are the best depth at W4A4, W6A6 and W8A8.
At W3A3 five blocks is $0.39$ points higher ($35.69$ against $35.30$) and the plateau runs from four
to seven blocks ($35.30$, $35.69$, $35.24$, $35.21$); taking the smallest scope on the plateau
therefore costs $0.39$ points there, against four more recalibrated interfaces.

\begin{table}[t]
\centering
\caption{Recalibration-scope ablation at all four bit-widths, Local-only with no global CLS branch,
15-corruption mean top-1 (\%) on the 5k subset, batch 64, episodic, 3-seed mean $\pm$ std. Deployed
default configuration \textbf{bold}; W4A4 and W6A6 reproduce Table~\ref{tab:scope}.}
\label{tab:scope4bw}
{\small
\begin{tabular}{l|cccc}
\toprule
Setting & W3A3 & W4A4 & W6A6 & W8A8 \\
\midrule
No Adapt & \ms{28.35}{0.00} & \ms{47.30}{0.00} & \ms{47.68}{0.00} & \ms{54.36}{0.00} \\
\midrule
\multicolumn{5}{l}{\emph{Modules} (blocks 8--11)} \\
\quad \texttt{attn.qkv}+\texttt{attn.proj} & \ms{31.75}{0.02} & \ms{50.18}{0.05} & \ms{50.72}{0.05} & \ms{57.10}{0.02} \\
\quad\quad{}+\,\texttt{mlp.fc1} & \ms{32.97}{0.04} & \ms{51.27}{0.05} & \ms{51.86}{0.05} & \ms{58.01}{0.01} \\
\rowcolor{oursgreen}\quad\quad{}+\,\texttt{mlp.fc2} & \ms{\mathbf{35.30}}{0.06} & \ms{\mathbf{52.87}}{0.03} & \ms{\mathbf{53.42}}{0.04} & \ms{\mathbf{59.11}}{0.02} \\
\midrule
\multicolumn{5}{l}{\emph{Block position} (4 blocks, all modules)} \\
\quad early (blocks 0--3) & \ms{21.03}{0.04} & \ms{30.93}{0.08} & \ms{32.12}{0.04} & \ms{37.13}{0.02} \\
\quad middle (blocks 4--7) & \ms{26.27}{0.09} & \ms{42.09}{0.01} & \ms{43.29}{0.04} & \ms{49.88}{0.06} \\
\quad late (blocks 8--11) & \ms{\mathbf{35.30}}{0.06} & \ms{\mathbf{52.87}}{0.03} & \ms{\mathbf{53.42}}{0.04} & \ms{\mathbf{59.11}}{0.02} \\
\midrule
\multicolumn{5}{l}{\emph{Trailing-block depth} (all modules)} \\
\quad 1 block (11) & \ms{28.71}{0.04} & \ms{48.31}{0.03} & \ms{48.54}{0.02} & \ms{55.67}{0.02} \\
\quad 2 blocks (10--11) & \ms{32.66}{0.04} & \ms{50.85}{0.03} & \ms{51.53}{0.05} & \ms{57.87}{0.03} \\
\quad 3 blocks (9--11) & \ms{34.16}{0.02} & \ms{52.03}{0.03} & \ms{52.56}{0.05} & \ms{58.48}{0.00} \\
\rowcolor{oursgreen}\quad 4 blocks (8--11) & \ms{\mathbf{35.30}}{0.06} & \ms{\mathbf{52.87}}{0.03} & \ms{\mathbf{53.42}}{0.04} & \ms{\mathbf{59.11}}{0.02} \\
\quad 5 blocks (7--11) & \ms{35.69}{0.07} & \ms{52.75}{0.02} & \ms{53.39}{0.07} & \ms{58.70}{0.04} \\
\quad 6 blocks (6--11) & \ms{35.24}{0.11} & \ms{52.11}{0.06} & \ms{53.28}{0.03} & \ms{58.21}{0.11} \\
\quad 7 blocks (5--11) & \ms{35.21}{0.07} & \ms{51.82}{0.05} & \ms{53.35}{0.02} & \ms{57.91}{0.06} \\
\quad all 12 blocks & \ms{22.47}{0.03} & \ms{28.24}{0.04} & \ms{31.93}{0.07} & \ms{32.64}{0.12} \\
\bottomrule
\end{tabular}}
\end{table}

\section{Generalization beyond ImageNet-C: ImageNet-R, Sketch, \texorpdfstring{\cbar}{C-bar}, and 3DCC}
\label{app:generalization}

We evaluate two non-corruption shifts: ImageNet-R \citep{hendrycks2021many} (renditions) and ImageNet-Sketch \citep{wang2019learning}. Table~\ref{tab:t4app} reports top-1 accuracy across all bit-widths. \method{} is the best method on ImageNet-Sketch at \emph{every} bit-width, and it also improves over the no-adaptation baseline on ImageNet-R at every bit-width. On ImageNet-R the recentering baselines edge ahead by a small margin: \ours{} trails the best method by $0.11$ at W3A3, $0.37$ at W4A4, $0.65$ at W6A6 and $0.40$ at W8A8, while remaining the strongest method overall. The partial recalibration ($\lambda=0.5$) is what enables this: the full match ($\lambda=1$) over-corrects the rendition-style per-channel statistics and pushes ImageNet-R below No Adapt, whereas the gentler $\lambda=0.5$ keeps ImageNet-R above No Adapt at every bit-width while preserving the corruption and sketch gains. \ours{} therefore no longer trades ImageNet-R for its other gains.

\paragraph{Held-out corruption suites: ImageNet-\texorpdfstring{\cbar}{C-bar} and ImageNet-3DCC.}
We further test two corruption benchmarks disjoint from the ImageNet-C corruptions used to choose \method{}'s configuration: ImageNet-\cbar{} \citep{mintun2021interaction} (10 algorithmic corruptions designed to be perceptually dissimilar from ImageNet-C) and ImageNet-3DCC \citep{kar2022threedcc} (12 corruptions synthesized using 3D scene geometry). Both use severity 5 on the 1000-class RobustBench \citep{croce2021robustbench} 5k subset, batch size 64, episodic, with the same frozen $\lambda{=}0.5$ configuration; the combined Table~\ref{tab:gen_full} reports 3-seed mean$\pm$std (seeds 0/1/2) for all four benchmarks. \method{} is the best method at every bit-width on both suites: over No Adapt it adds $+5.5$ to $+7.8$ on \cbar{} and $+2.8$ to $+3.4$ on 3DCC, so the local recalibration restores accuracy even on corruption types it was never tuned on.

\begin{table}[t]
\centering
\caption{Full per-method results on the four OOD generalization benchmarks (quantized ViT-B;
3-seed mean$\pm$std over seeds 0/1/2). \cbar{}:
ImageNet-\cbar{} (10 algorithmic corruptions); R: ImageNet-R; Sk.: ImageNet-Sketch; 3DCC:
ImageNet-3DCC (12 3D common corruptions). Best per column \textbf{bold}, second
\underline{underlined}.}
\label{tab:gen_full}
{\small
\begin{tabular}{l|cccc}
\toprule
Method & W3A3 & W4A4 & W6A6 & W8A8 \\
\midrule
\multicolumn{5}{l}{\emph{\cbar} (ImageNet-\cbar{})} \\
\quad No Adapt & \ms{36.63}{0.00} & \ms{55.65}{0.01} & \ms{59.07}{0.00} & \ms{63.94}{0.00} \\
\quad T3A & \ms{36.82}{0.03} & \ms{56.09}{0.03} & \ms{59.83}{0.02} & \ms{64.54}{0.04} \\
\quad FOA & \ms{40.86}{0.07} & \ms{59.92}{0.07} & \ms{63.40}{0.09} & \ms{67.85}{0.11} \\
\quad ZOA & \ms{40.65}{0.17} & \ms{59.97}{0.16} & \ms{\underline{64.02}}{0.08} & \ms{67.84}{0.06} \\
\quad NEO & \ms{41.17}{0.02} & \ms{\underline{60.11}}{0.02} & \ms{63.94}{0.04} & \ms{67.80}{0.01} \\
\quad FOZO & \ms{\underline{41.57}}{0.14} & \ms{59.90}{0.13} & \ms{61.67}{0.20} & \ms{\underline{68.13}}{0.10} \\
\rowcolor{oursgreen}\quad \textbf{\method{}} & \ms{\mathbf{44.92}}{0.10} & \ms{\mathbf{61.91}}{0.13} & \ms{\mathbf{65.45}}{0.02} & \ms{\mathbf{69.32}}{0.06} \\
\midrule
\multicolumn{5}{l}{\emph{R} (ImageNet-R)} \\
\quad No Adapt & \ms{42.48}{0.01} & \ms{53.29}{0.00} & \ms{55.89}{0.00} & \ms{58.19}{0.00} \\
\quad T3A & \ms{38.27}{0.04} & \ms{50.00}{0.15} & \ms{52.74}{0.06} & \ms{56.40}{0.05} \\
\quad FOA & \ms{43.27}{0.19} & \ms{53.74}{0.14} & \ms{56.17}{0.10} & \ms{59.01}{0.20} \\
\quad ZOA & \ms{43.18}{0.25} & \ms{\mathbf{54.69}}{0.21} & \ms{\underline{56.92}}{0.06} & \ms{\underline{59.26}}{0.08} \\
\quad NEO & \ms{\mathbf{44.01}}{0.01} & \ms{\underline{54.40}}{0.01} & \ms{\mathbf{57.25}}{0.01} & \ms{59.21}{0.04} \\
\quad FOZO & \ms{43.02}{0.22} & \ms{53.23}{0.17} & \ms{54.72}{0.24} & \ms{\mathbf{59.37}}{0.11} \\
\rowcolor{oursgreen}\quad \textbf{\method{}} & \ms{\underline{43.90}}{0.04} & \ms{54.32}{0.06} & \ms{56.60}{0.12} & \ms{58.97}{0.03} \\
\midrule
\multicolumn{5}{l}{\emph{Sk.} (ImageNet-Sketch)} \\
\quad No Adapt & \ms{25.01}{0.00} & \ms{37.29}{0.00} & \ms{40.13}{0.00} & \ms{43.71}{0.00} \\
\quad T3A & \ms{27.63}{0.13} & \ms{\underline{40.67}}{0.11} & \ms{\underline{44.08}}{0.08} & \ms{\underline{47.29}}{0.02} \\
\quad FOA & \ms{27.83}{0.08} & \ms{39.44}{0.16} & \ms{42.27}{0.14} & \ms{45.76}{0.01} \\
\quad ZOA & \ms{27.99}{0.11} & \ms{40.56}{0.20} & \ms{42.81}{0.18} & \ms{45.92}{0.05} \\
\quad NEO & \ms{28.22}{0.01} & \ms{39.76}{0.01} & \ms{42.78}{0.01} & \ms{45.63}{0.01} \\
\quad FOZO & \ms{\underline{29.47}}{0.07} & \ms{40.45}{0.07} & \ms{42.51}{0.34} & \ms{46.67}{0.12} \\
\rowcolor{oursgreen}\quad \textbf{\method{}} & \ms{\mathbf{32.10}}{0.05} & \ms{\mathbf{42.99}}{0.03} & \ms{\mathbf{45.65}}{0.04} & \ms{\mathbf{48.88}}{0.03} \\
\midrule
\multicolumn{5}{l}{\emph{3DCC} (ImageNet-3DCC)} \\
\quad No Adapt & \ms{34.21}{0.00} & \ms{47.75}{0.00} & \ms{47.68}{0.00} & \ms{52.49}{0.00} \\
\quad T3A & \ms{34.18}{0.01} & \ms{47.69}{0.01} & \ms{47.63}{0.01} & \ms{52.48}{0.01} \\
\quad FOA & \ms{35.55}{0.16} & \ms{48.75}{0.05} & \ms{48.69}{0.03} & \ms{53.45}{0.04} \\
\quad ZOA & \ms{35.78}{0.03} & \ms{\underline{49.09}}{0.09} & \ms{49.31}{0.03} & \ms{53.66}{0.08} \\
\quad NEO & \ms{35.87}{0.02} & \ms{48.87}{0.03} & \ms{\underline{49.40}}{0.01} & \ms{53.65}{0.02} \\
\quad FOZO & \ms{\underline{35.88}}{0.06} & \ms{48.93}{0.10} & \ms{48.50}{0.08} & \ms{\underline{53.95}}{0.06} \\
\rowcolor{oursgreen}\quad \textbf{\method{}} & \ms{\mathbf{37.50}}{0.05} & \ms{\mathbf{50.61}}{0.04} & \ms{\mathbf{51.20}}{0.05} & \ms{\mathbf{55.45}}{0.05} \\
\bottomrule
\end{tabular}}
\end{table}

\paragraph{Held-out ImageNet-C corruptions.}
As a final within-family generalization test, we evaluate the four \emph{extra} ImageNet-C corruptions absent from the 15 standard corruptions used throughout the paper, hence unseen when \method{}'s scope and $\lambda$ were chosen: Gaussian blur, saturate, spatter, and speckle noise (severity 5; same 5k subset, batch size 64, episodic, $\lambda{=}0.5$). Table~\ref{tab:heldc_full} reports 3-seed mean$\pm$std. \method{} is again the best method at every bit-width, improving over No Adapt by $+3.3$ to $+6.7$, so the recalibration generalizes to corruption types it never saw during method design.

\begin{table}[t]
\centering
\caption{Held-out ImageNet-C: the four \emph{extra} ImageNet-C corruptions (Gaussian blur, saturate,
spatter, speckle noise) not in the standard 15, sev.\ 5; mean over the 4 corruptions. Quantized
ViT-B; 3-seed mean$\pm$std over seeds 0/1/2. Best
per column \textbf{bold}, second \underline{underlined}.}
\label{tab:heldc_full}
{\small
\begin{tabular}{l|cccc}
\toprule
Method & W3A3 & W4A4 & W6A6 & W8A8 \\
\midrule
No Adapt & \ms{35.76}{0.00} & \ms{56.53}{0.00} & \ms{56.73}{0.00} & \ms{62.74}{0.00} \\
T3A & \ms{35.84}{0.02} & \ms{56.53}{0.01} & \ms{56.74}{0.01} & \ms{62.79}{0.04} \\
FOA & \ms{38.28}{0.08} & \ms{57.95}{0.23} & \ms{57.88}{0.11} & \ms{63.85}{0.06} \\
ZOA & \ms{38.45}{0.19} & \ms{58.33}{0.11} & \ms{58.83}{0.07} & \ms{64.25}{0.12} \\
NEO & \ms{38.90}{0.04} & \ms{\underline{58.26}}{0.07} & \ms{\underline{58.93}}{0.01} & \ms{64.11}{0.03} \\
FOZO & \ms{\underline{39.31}}{0.12} & \ms{58.10}{0.24} & \ms{57.40}{0.19} & \ms{\underline{64.39}}{0.14} \\
\rowcolor{oursgreen}\textbf{\method{}} & \ms{\mathbf{43.08}}{0.22} & \ms{\mathbf{60.40}}{0.12} & \ms{\mathbf{60.83}}{0.07} & \ms{\mathbf{65.91}}{0.05} \\
\bottomrule
\end{tabular}}
\end{table}

\section{Generalization beyond ImageNet: CIFAR-10-C, CIFAR-100-C and DomainNet-126}
\label{app:beyond}

The benchmarks above all derive from ImageNet, so we also evaluate two datasets outside that family: CIFAR-10-C and CIFAR-100-C, which repeat the ImageNet-C corruptions on a different dataset, and DomainNet-126, where the shift is a change of domain rather than a corruption.

\paragraph{CIFAR-10-C and CIFAR-100-C.}
We repeat the ImageNet-C protocol on CIFAR-10-C and CIFAR-100-C \citep{hendrycks2019benchmarking}, with ViT-B/16 backbones fine-tuned on CIFAR-10 and CIFAR-100 (the public Hugging Face checkpoints \texttt{nateraw/vit-base-patch16-224-cifar10}, which NEO also evaluates on, and \texttt{pkr7098/cifar100-vit-base-patch16-224-in21k}) and images upscaled to $224\times224$; the full-precision models reach $98.52$ and $92.61$ clean top-1. Quantization follows the main protocol: PTQ4ViT at W6A6/W8A8 with Hessian-guided calibration on $32$ training images, AdaLog at W4A4/W3A3 with calibration and block reconstruction on training images, and source moments accumulated over a seeded $12{,}800$-image training subset ($200$ batches of $64$, as for ImageNet); the ZOA source statistics are likewise computed on the training split. Evaluation uses the $15$ standard corruptions at severity $5$, the full $10{,}000$ images per corruption, batch $64$, episodic reset, and three sample-order seeds.

\paragraph{DomainNet-126.}
We follow the 126-class protocol of AdaContrast \citep{chen2022contrastive} on DomainNet \citep{peng2019moment}: the source domain is \emph{real} and the three target domains are \emph{clipart}, \emph{painting} and \emph{sketch}, each evaluated in full ($18{,}523$, $30{,}042$ and $24{,}147$ images). The protocol needs a model trained on the source domain, so we fine-tune one on \emph{real} alone with a stock recipe: a \texttt{vit\_base\_patch16\_224.augreg\_in21k} backbone fine-tuned for five epochs on a seeded $90\%$ split (AdamW, learning rate $5{\times}10^{-5}$, weight decay $0.05$, $10\%$ warm-up then cosine decay, batch $64$, label smoothing $0.1$, mixed precision), reaching $94.28$ top-1 on the held-out $10\%$; that held-out split also supplies the PTQ calibration data and the source moments. Quantization follows the main protocol, PTQ4ViT at W6A6/W8A8 and AdaLog at W4A4/W3A3, and leaves the clean held-out accuracy at $93.58$, $93.65$ and $91.41$ at W6A6, W4A4 and W3A3. Adaptation is episodic per target domain, at batch $64$ and three sample-order seeds. On both benchmarks no test image enters calibration or statistics, \method{} keeps the fixed configuration of the main text, blocks 8--11 at all four interfaces, $\lambda{=}0.5$, cumulative moments, and we run the full baseline set of Table~\ref{tab:main}, each in its published configuration. Every method adapts the same frozen quantized checkpoint, so nothing in how that checkpoint was obtained can favour one of them.

\paragraph{Results.}
Table~\ref{tab:beyond} reports both benchmarks. \method{} is the most accurate method in all twelve dataset$\times$bit-width cells, ahead of the runner-up by $0.6$ to $3.7$ points and with a seed std of at most $0.04$. The runner-up itself changes with the dataset: FOZO on CIFAR-10 at every bit-width and on CIFAR-100 at W4A4/W3A3, ZOA on CIFAR-100 at W6A6/W8A8, and T3A on DomainNet-126 at every bit-width. On CIFAR the gain over No Adapt grows monotonically as precision falls, from $+2.9$ to $+5.7$ on CIFAR-10 and from $+3.7$ to $+6.5$ on CIFAR-100, the pattern of ImageNet-C, and on DomainNet-126 it grows the same way within each quantizer, from $+2.2$ to $+4.4$ under PTQ4ViT and from $+3.6$ to $+5.3$ under AdaLog. On DomainNet-126 the gain is concentrated where the shift is largest: on \emph{sketch} \method{} adds $5.5$ to $9.4$ points over No Adapt, whereas on \emph{painting} at W8A8, the mildest of the three, it matches No Adapt, the pattern we also report on ImageNet-R (Appendix~\ref{app:generalization}). Per corruption on CIFAR (seed 0) the largest gains are on pixelate, glass blur, impulse noise and Gaussian noise; on CIFAR-10 at W3A3 pixelate rises from $44.7$ to $69.3$ and impulse noise from $53.5$ to $70.9$.

\paragraph{The baselines reorder.}
T3A, which barely moves on ImageNet-C ($+0.1$ at W6A6 in Table~\ref{tab:main}), adds $0.2$ to $1.8$ points on CIFAR and is the runner-up on DomainNet-126 at every bit-width, so the ImageNet-C ordering does not carry over. FOA and FOZO instead fall below No Adapt at some DomainNet-126 bit-widths ($61.86$ and $56.95$ at W6A6, $52.25$ for FOZO at W3A3), while FOZO is the strongest baseline on CIFAR. Neither method's paper evaluates CIFAR or DomainNet and we ran both unchanged: FOA at the two-forward budget our protocol imposes on every optimization-based baseline, which costs it accuracy here as it does on ImageNet-C (Appendix~\ref{app:foa28} measures the recommended population), and FOZO with the three prompt tokens its authors fix, having argued that tuning the prompt count on unseen data is not practical at test time. Inserting three randomly initialized prompt tokens is what costs it accuracy on the DomainNet-126 backbone; we report that setting rather than re-tune per dataset.

\begin{table}[t]
\centering
\caption{Generalization beyond ImageNet: CIFAR-10-C and CIFAR-100-C (severity 5, 15 corruptions, full
$10$k images per corruption) and DomainNet-126 (mean over the three target domains, \emph{real} as
source), each on a ViT-B/16 fine-tuned on its own source data and quantized as in the main protocol;
batch 64, episodic, 3-seed mean $\pm$ std over sample order. Best \textbf{bold}, second
\underline{underlined}.}
\label{tab:beyond}
{\small
\begin{tabular}{l|cccc}
\toprule
Method & W3A3 & W4A4 & W6A6 & W8A8 \\
\midrule
\multicolumn{5}{l}{\emph{CIFAR-10-C}} \\
\quad No Adapt & \ms{76.34}{0.00} & \ms{81.05}{0.00} & \ms{82.86}{0.00} & \ms{84.39}{0.00} \\
\quad T3A & \ms{76.67}{0.01} & \ms{81.29}{0.01} & \ms{83.15}{0.01} & \ms{84.54}{0.01} \\
\quad FOA & \ms{78.04}{0.05} & \ms{82.22}{0.02} & \ms{83.84}{0.05} & \ms{85.19}{0.03} \\
\quad ZOA & \ms{78.35}{0.03} & \ms{82.57}{0.07} & \ms{84.54}{0.11} & \ms{85.82}{0.05} \\
\quad NEO & \ms{78.42}{0.01} & \ms{82.42}{0.00} & \ms{84.05}{0.01} & \ms{85.26}{0.01} \\
\quad FOZO & \ms{\underline{79.93}}{0.03} & \ms{\underline{83.69}}{0.10} & \ms{\underline{85.03}}{0.08} & \ms{\underline{86.09}}{0.02} \\
\rowcolor{oursgreen}\quad \textbf{\method{}} & \ms{\mathbf{82.05}}{0.03} & \ms{\mathbf{85.34}}{0.00} & \ms{\mathbf{85.84}}{0.02} & \ms{\mathbf{87.30}}{0.01} \\
\midrule
\multicolumn{5}{l}{\emph{CIFAR-100-C}} \\
\quad No Adapt & \ms{54.15}{0.00} & \ms{60.85}{0.00} & \ms{63.38}{0.00} & \ms{65.34}{0.00} \\
\quad T3A & \ms{55.98}{0.02} & \ms{62.00}{0.02} & \ms{64.56}{0.02} & \ms{66.26}{0.03} \\
\quad FOA & \ms{55.61}{0.05} & \ms{61.61}{0.06} & \ms{63.84}{0.03} & \ms{65.79}{0.02} \\
\quad ZOA & \ms{54.87}{0.08} & \ms{61.44}{0.11} & \ms{\underline{65.09}}{0.33} & \ms{\underline{66.67}}{0.19} \\
\quad NEO & \ms{56.05}{0.00} & \ms{61.95}{0.01} & \ms{64.30}{0.01} & \ms{66.05}{0.01} \\
\quad FOZO & \ms{\underline{56.93}}{0.17} & \ms{\underline{62.78}}{0.02} & \ms{64.84}{0.02} & \ms{66.44}{0.01} \\
\rowcolor{oursgreen}\quad \textbf{\method{}} & \ms{\mathbf{60.64}}{0.03} & \ms{\mathbf{66.15}}{0.02} & \ms{\mathbf{67.33}}{0.01} & \ms{\mathbf{69.07}}{0.01} \\
\midrule
\multicolumn{5}{l}{\emph{DomainNet-126} (mean over clipart, painting, sketch)} \\
\quad No Adapt & \ms{54.40}{0.00} & \ms{63.27}{0.00} & \ms{62.45}{0.00} & \ms{67.65}{0.00} \\
\quad T3A & \ms{\underline{57.29}}{0.04} & \ms{\underline{65.43}}{0.03} & \ms{\underline{65.00}}{0.04} & \ms{\underline{69.28}}{0.01} \\
\quad FOA & \ms{55.40}{0.11} & \ms{64.18}{0.13} & \ms{61.86}{0.16} & \ms{68.05}{0.11} \\
\quad ZOA & \ms{56.19}{0.13} & \ms{64.71}{0.03} & \ms{64.58}{0.04} & \ms{68.74}{0.03} \\
\quad NEO & \ms{56.48}{0.01} & \ms{64.71}{0.01} & \ms{64.56}{0.01} & \ms{68.75}{0.01} \\
\quad FOZO & \ms{52.25}{0.95} & \ms{63.73}{0.36} & \ms{56.95}{2.22} & \ms{69.00}{0.22} \\
\rowcolor{oursgreen}\quad \textbf{\method{}} & \ms{\mathbf{59.72}}{0.04} & \ms{\mathbf{66.84}}{0.03} & \ms{\mathbf{66.85}}{0.02} & \ms{\mathbf{69.89}}{0.04} \\
\bottomrule
\end{tabular}}
\end{table}

\section{FOA at its recommended configuration ($K{=}28$)}
\label{app:foa28}
The main comparison (Table~\ref{tab:main}) runs FOA at its minimum-forward CMA-ES population ($K{=}2$, two forwards per batch), matching ZOA's minimum-forward protocol \citep{deng2025test}. A natural fairness question is whether FOA at its recommended population ($K{=}28$, the \citet{niu2024test} default) would overturn the ranking. Table~\ref{tab:foa28} re-runs FOA at $K{=}28$ (${\sim}29$ forwards per batch, ${\sim}14\times$ its minimum-forward budget) across ImageNet-C and the four generalization suites. The larger population lifts FOA by $0.2$ to $2.3$ points over its minimum-forward setting, but it remains below \method{} (one forward) at every bit-width on ImageNet-C, \cbar{}, Sketch, and 3DCC. The single exception is ImageNet-R, the domain-shift suite where \method{}'s partial recalibration is already weakest and where NEO and ZOA also lead \method{} (Table~\ref{tab:gen_full}); there $K{=}28$ edges \method{} by at most $1.2$ points. Measured under the same per-image latency protocol, $K{=}28$ runs at ${\sim}162$ ms per image, roughly $25\times$ \method{}'s $6.6$ ms (median over five runs; a same-session $K{=}2$ re-measurement reproduces FOA's $13.0$ ms minimum-forward latency to within ${\sim}13\%$, and we normalize $K{=}28$ to that frame). The ImageNet-C ranking in Table~\ref{tab:main} is therefore not an artifact of FOA's minimum-forward budget.

\begin{table}[t]
\centering
\caption{FOA at its minimum-forward setting ($K{=}2$, two forwards/batch) vs.\ its recommended CMA-ES population ($K{=}28$, ${\sim}29$ forwards; \citet{niu2024test} default), and \method{} (one forward), across ImageNet-C and the four generalization suites (quantized ViT-B; W6A6/W8A8 PTQ4ViT, W4A4/W3A3 AdaLog; 3-seed mean$\pm$std). At ${\sim}14\times$ the forwards, $K{=}28$ stays below \method{} at every bit-width except ImageNet-R. Best of the three \textbf{bold}.}
\label{tab:foa28}
{\small
\begin{tabular}{l|cccc}
\toprule
Method & W3A3 & W4A4 & W6A6 & W8A8 \\
\midrule
\multicolumn{5}{l}{\emph{ImageNet-C}} \\
\quad FOA ($K{=}2$) & \ms{31.52}{0.09} & \ms{50.25}{0.07} & \ms{50.39}{0.05} & \ms{57.85}{0.16} \\
\quad FOA ($K{=}28$) & \ms{31.88}{0.13} & \ms{50.47}{0.07} & \ms{52.30}{0.24} & \ms{59.00}{0.11} \\
\rowcolor{oursgreen}\quad \textbf{\method{}} & \ms{\mathbf{35.98}}{0.00} & \ms{\mathbf{53.58}}{0.04} & \ms{\mathbf{54.35}}{0.02} & \ms{\mathbf{60.44}}{0.01} \\
\midrule
\multicolumn{5}{l}{\emph{\cbar} (ImageNet-\cbar{})} \\
\quad FOA ($K{=}2$) & \ms{40.86}{0.07} & \ms{59.92}{0.07} & \ms{63.40}{0.09} & \ms{67.85}{0.11} \\
\quad FOA ($K{=}28$) & \ms{41.10}{0.02} & \ms{60.40}{0.10} & \ms{64.22}{0.10} & \ms{68.52}{0.02} \\
\rowcolor{oursgreen}\quad \textbf{\method{}} & \ms{\mathbf{44.92}}{0.10} & \ms{\mathbf{61.91}}{0.13} & \ms{\mathbf{65.45}}{0.02} & \ms{\mathbf{69.32}}{0.06} \\
\midrule
\multicolumn{5}{l}{\emph{R} (ImageNet-R)} \\
\quad FOA ($K{=}2$) & \ms{43.27}{0.19} & \ms{53.74}{0.14} & \ms{56.17}{0.10} & \ms{59.01}{0.20} \\
\quad FOA ($K{=}28$) & \ms{\mathbf{44.33}}{0.06} & \ms{\mathbf{54.67}}{0.17} & \ms{\mathbf{56.83}}{0.26} & \ms{\mathbf{60.18}}{0.15} \\
\rowcolor{oursgreen}\quad \textbf{\method{}} & \ms{43.90}{0.04} & \ms{54.32}{0.06} & \ms{56.60}{0.12} & \ms{58.97}{0.03} \\
\midrule
\multicolumn{5}{l}{\emph{Sk.} (ImageNet-Sketch)} \\
\quad FOA ($K{=}2$) & \ms{27.83}{0.08} & \ms{39.44}{0.16} & \ms{42.27}{0.14} & \ms{45.76}{0.01} \\
\quad FOA ($K{=}28$) & \ms{28.49}{0.10} & \ms{40.17}{0.07} & \ms{43.63}{0.40} & \ms{48.10}{0.53} \\
\rowcolor{oursgreen}\quad \textbf{\method{}} & \ms{\mathbf{32.10}}{0.05} & \ms{\mathbf{42.99}}{0.03} & \ms{\mathbf{45.65}}{0.04} & \ms{\mathbf{48.88}}{0.03} \\
\midrule
\multicolumn{5}{l}{\emph{3DCC} (ImageNet-3DCC)} \\
\quad FOA ($K{=}2$) & \ms{35.55}{0.16} & \ms{48.75}{0.05} & \ms{48.69}{0.03} & \ms{53.45}{0.04} \\
\quad FOA ($K{=}28$) & \ms{35.89}{0.04} & \ms{49.30}{0.04} & \ms{50.35}{0.08} & \ms{54.49}{0.02} \\
\rowcolor{oursgreen}\quad \textbf{\method{}} & \ms{\mathbf{37.50}}{0.05} & \ms{\mathbf{50.61}}{0.04} & \ms{\mathbf{51.20}}{0.05} & \ms{\mathbf{55.45}}{0.05} \\
\bottomrule
\end{tabular}}
\end{table}

\section{Comparison with Adaptive Quantile Recalibration (AQR)}
\label{app:aqr}
AQR \citep{aqr2026} recalibrates pre-activations in full precision by matching the full channel-wise activation distribution with per-channel quantile transforms at normalization-layer outputs. As no public implementation exists, we implement it in the configuration AQR uses for ViT: after-LayerNorm placement in the top half of the blocks (6--11), $101$ percentiles, source percentiles estimated from $10{,}000$ training images, and AQR's paper-default \emph{standard} tail handling (raw min/max). We report this default configuration rather than a variant tuned on our setting. We first validate the implementation on the \emph{full-precision} ViT-B, where it improves over No~Adapt by $+2.8$ to $+3.3$ points at batch sizes $\ge 64$ (e.g.\ $55.58\to58.42$ at batch $64$), matching the direction reported by \citet{aqr2026}; AQR's behavior on the quantized model is therefore a genuine finding, not an implementation artifact.

Table~\ref{tab:aqr} compares AQR and \method{} on the quantized W6A6 ViT-B across batch sizes. Even at the large batches AQR is built for ($64$--$512$), where its per-batch quantile estimation is most reliable, \method{} is $+2.0$ to $+3.0$ points higher; and, crucial for streaming deployment, AQR's per-batch quantile estimation is unstable at small batches and collapses at batch size~$1$ ($18.73$, below No~Adapt, both here and on the full-precision model), whereas \method{}'s source-initialized cumulative statistics keep it robust down to a single image ($54.27$).

Two design differences explain \method{}'s higher accuracy even in AQR's favorable large-batch regime. First, \method{} acts \emph{at the frozen quantizer} and matches the \emph{source calibration} moments that define its operating point, directly restoring the quantization-specific failure it targets (the distorted code distribution); AQR matches the full activation distribution only at normalization-layer outputs, a quantizer-agnostic correction that is not aligned to the frozen quantizer's calibrated range.
Second, in this quantized setting, the more effective intervention is to restore the per-channel operating point, mean and scale, at the frozen quantizer input. QuAR does this with stable moment estimates, whereas AQR's nonlinear per-channel quantile remap is quantizer-agnostic and depends on per-batch percentile estimates, which become unreliable as the batch shrinks.

The conclusion is insensitive to AQR's remaining design choices: varying its placement and scope never lifts it above \method{} (its best configuration, a post-hoc block-$8$--$11$ scope tuned to our setting, still reaches only $52.41$ at batch~$64$). Absolute full-precision numbers are not directly comparable to \citet{aqr2026} (whose ViT is trained differently), so we use the full-precision result only as a correctness check.

\begin{table}[t]
\centering
\caption{\emph{Comparison with AQR} on quantized W6A6 ViT-B (ImageNet-C, $15$-corruption mean top-1
\%, $5$k subset, sev.\ 5; 3-seed mean$\pm$std over seeds $0/1/2$). AQR is our implementation of the
paper-default ViT configuration (after-LayerNorm, top half blocks 6--11, $101$ percentiles, standard
raw-tail handling); No~Adapt is measured at each batch size at seed $0$, varying by $0.07$
over this range. Best per column \textbf{bold}.}
\label{tab:aqr}
{\small
\begin{tabular}{l|cccc}
\toprule
Batch size & $1$ & $64$ & $128$ & $512$ \\
\midrule
No Adapt & 47.75 & 47.69 & 47.69 & 47.68 \\
AQR \citep{aqr2026} (our impl.) & \ms{18.73}{0.00} & \ms{51.37}{0.04} & \ms{51.63}{0.06} & \ms{51.75}{0.02} \\
\rowcolor{oursgreen}\textbf{\method{}} & \ms{\mathbf{54.27}}{0.05} & \ms{\mathbf{54.35}}{0.02} & \ms{\mathbf{54.29}}{0.08} & \ms{\mathbf{53.77}}{0.02} \\
\bottomrule
\end{tabular}}
\end{table}

\section{Online label shift}
\label{app:noniid}
Table~\ref{tab:noniid} gives the per-method values and Figure~\ref{fig:noniid} plots them. The sampler and the imbalance ratios are specified in Appendix~\ref{app:repro}.

\begin{table}[ht]
\centering
\caption{Non-i.i.d.\ online label shift by imbalance ratio, $1$ being i.i.d.\ and
$5{\times}10^{5}$ near class-sorted. W6A6, \emph{Full} ImageNet-C ($50$k/corruption),
15-corruption mean (\%), 3 seeds; best \textbf{bold}. These are the values behind
Figure~\ref{fig:noniid}.}
\label{tab:noniid}
{\small
\begin{tabular}{l|ccccc}
\toprule
Method & $1$ & $333$ & $10^{3}$ & $3{\times}10^{3}$ & $5{\times}10^{5}$ \\
\midrule
No Adapt & \multicolumn{5}{c}{$47.55$} \\
T3A                                    & \ms{48.67}{0.01} & \ms{48.61}{0.03} & \ms{48.19}{0.02} & \ms{47.60}{0.09} & \ms{47.24}{0.11} \\
FOA                                    & \ms{50.19}{0.11} & \ms{50.32}{0.24} & \ms{50.33}{0.14} & \ms{50.10}{0.35} & \ms{49.86}{0.33} \\
ZOA                                    & \ms{52.81}{0.06} & \ms{52.86}{0.05} & \ms{52.68}{0.11} & \ms{52.56}{0.12} & \ms{52.05}{0.04} \\
NEO                                    & \ms{51.39}{0.00} & \ms{51.45}{0.08} & \ms{51.39}{0.04} & \ms{51.33}{0.11} & \ms{51.24}{0.04} \\
FOZO                                   & \ms{49.49}{0.16} & \ms{49.77}{0.26} & \ms{49.66}{0.21} & \ms{49.45}{0.26} & \ms{48.92}{0.21} \\
\rowcolor{oursgreen}\textbf{\method{}} & \ms{\mathbf{54.55}}{0.03} & \ms{\mathbf{54.66}}{0.12} & \ms{\mathbf{54.64}}{0.06} & \ms{\mathbf{54.65}}{0.13} & \ms{\mathbf{54.73}}{0.07} \\
\bottomrule
\end{tabular}}
\end{table}

\begin{figure}[ht]
\centering
\includegraphics[width=0.5\textwidth]{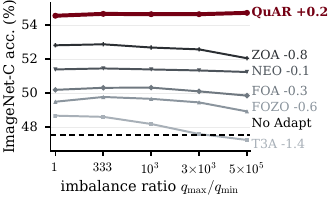}
\caption{Online label shift (W6A6, full ImageNet-C, 3-seed mean) by imbalance ratio; labels give each method's accuracy change from i.i.d.\ to $5{\times}10^{5}$. \method{} is the most accurate at every ratio and the only method whose accuracy does not fall.}
\label{fig:noniid}
\end{figure}

\section{SPA under online label shift}
\label{app:spanoniid}
We place SPA on the protocol of Table~\ref{tab:noniid}, which differs from Table~\ref{tab:backprop_acc} in using the full ImageNet-C test set ($50$k images per corruption) rather than the $5$k subset: W6A6 PTQ4ViT ViT-B, all $15$ corruptions at severity~$5$, batch size $64$, episodic reset per corruption, the SAR sampler with $T{=}50{,}000$, and seeds $0/1/2$ with sampler seed $2023{+}s$. At the highest imbalance, $\mathrm{ir}{=}5{\times}10^{5}$, SPA scores $53.85$, $47.92$ and $56.13$ over the three seeds, giving $52.63{\pm}4.24$ against \method{}'s $54.73{\pm}0.07$.

The spread is the more informative half. At this ratio every method in Table~\ref{tab:noniid} stays within $0.35$ of its own mean and \method{} within $0.07$, whereas SPA varies by more than eight points across seeds on the same stream. The loss concentrates in the blur corruptions, where individual cells fall below $2$ top-1. Bare SPA optimizes a self-consistency loss between a clean view and two perturbed ones, with no entropy threshold and no redundancy filter, so nothing in it rejects the near-single-class batches this stream produces. Adaptation is reset per corruption, so these collapses occur inside individual corruptions rather than accumulating along the stream.

\section{Clean source: behavior under no shift}
\label{app:noshift}
A test-time method that tracks stream statistics could in principle harm accuracy when there is \emph{no} shift to correct. We test this directly: we evaluate \method{} in its deployed configuration (local recalibration at blocks 8--11 with $\lambda{=}0.5$, cumulative running target moments, plus the global CLS branch) on \emph{clean} ImageNet validation (no corruption; the same RobustBench 5k subset, ViT-B, batch size 64, seed 0), the in-distribution match to the clean data on which the source PTQ-calibration moments were collected. Under no shift the running target moments coincide with the frozen source moments, so the recalibration of Equation~\ref{eq:local} reduces to the identity up to the finite-sample estimation noise in the running moments (the $\eta$ of Proposition~\ref{prop:shrink}). Table~\ref{tab:noshift} confirms this: across all four bit-widths \method{} changes clean top-1 by at most $0.42$ points (W8A8 $-0.42$, W6A6 $-0.08$, W4A4 $+0.08$, W3A3 $-0.14$). Enabling the recalibration is therefore safe even when adaptation is unnecessary: it neither requires a shift detector nor degrades the in-distribution operating point.

\begin{table}[t]
\centering
\caption{Clean-source (no-shift) sanity: top-1 accuracy (\%) on \emph{clean} ImageNet validation
(no corruption; RobustBench 5k subset, ViT-B, batch size 64, seed 0; W6A6/W8A8 PTQ4ViT, W4A4/W3A3
AdaLog). \method{} uses its deployed configuration. $\Delta$ is \method{} minus No Adapt; a value
near zero means the recalibration does not harm the in-distribution operating point.}
\label{tab:noshift}
{\small
\begin{tabular}{l|cccc}
\toprule
 & W3A3 & W4A4 & W6A6 & W8A8 \\
\midrule
No Adapt & $76.76$ & $82.88$ & $82.74$ & $84.12$ \\
\rowcolor{oursgreen}\textbf{\method{}} & $76.62$ & $82.96$ & $82.66$ & $83.70$ \\
\midrule
$\Delta$ (\method{} $-$ No Adapt) & $-0.14$ & $+0.08$ & $-0.08$ & $-0.42$ \\
\bottomrule
\end{tabular}}
\end{table}

\section{Robustness to the quantizer choice}
\label{app:crossquant}
The main results follow ZOA's protocol \citep{deng2025test}, using AdaLog at W3A3/W4A4 and PTQ4ViT at W6A6/W8A8. To check that \method{}'s gains are not tied to a particular PTQ method, we additionally run \emph{both} quantizers at \emph{every} bit-width (Table~\ref{tab:crossquant}; same 15-corruption ImageNet-C protocol, seed 0, episodic; \method{} uses per-quantizer source statistics). Excluding the degenerate PTQ4ViT-at-W3A3 column, where PTQ4ViT collapses to near-chance for \emph{all} methods (it is not intended for 3-bit, which is exactly why the main protocol uses AdaLog there), \method{} is the best method in all seven remaining quantizer$\times$bit-width settings, so the recalibration is quantizer-agnostic. PTQ4ViT at W6A6/W8A8 reproduces the main tables (\method{} $54.35$ at W6A6, $60.44$ at W8A8), confirming the pipeline. Within either quantizer the gain over No Adapt grows monotonically as precision falls, from $5.08$ at W8A8 to $7.63$ at W3A3 under AdaLog and from $6.08$ at W8A8 to $7.47$ at W4A4 under PTQ4ViT; the dip between W6A6 and W4A4 in Figure~\ref{fig:teaser}a is the change of quantizer, not of the mechanism.

\begin{table*}[t]
\centering
\caption{Cross-quantizer comparison: mean top-1 (\%) on ImageNet-C with \emph{both} PTQ methods at every bit-width. \method{}: per-quantizer source stats ($\lambda{=}0.5$, cumulative moments, blocks 8--11 $+$ CLS). PTQ4ViT degrades sharply at low bit-widths (near-chance at W3A3), hence AdaLog in the main protocol; ZOA's source stats are collected on the deployed quantizer, so it is slightly mismatched here. Best per column \textbf{bold}, second \underline{underlined}.}
\label{tab:crossquant}
\resizebox{\textwidth}{!}{%
\begin{tabular}{l|cc|cc|cc|cc}
\toprule
Method & \multicolumn{2}{c|}{W3A3} & \multicolumn{2}{c|}{W4A4} & \multicolumn{2}{c|}{W6A6} & \multicolumn{2}{c}{W8A8} \\
 & AdaLog & PTQ4ViT & AdaLog & PTQ4ViT & AdaLog & PTQ4ViT & AdaLog & PTQ4ViT \\
\midrule
No Adapt & 28.35 & 0.35 & 47.30 & 18.93 & 54.79 & 47.69 & 53.91 & 54.36 \\
T3A & 28.50 & 0.42 & 47.42 & 18.95 & 55.05 & 47.80 & 53.54 & 54.17 \\
FOA & 31.49 & 0.71 & 50.33 & 22.75 & 57.83 & 50.40 & 56.96 & 57.88 \\
ZOA & 31.40 & 0.79 & \underline{50.78} & 23.44 & 57.93 & 51.35 & 57.30 & \underline{58.21} \\
NEO & 31.90 & \textbf{0.84} & 50.29 & \underline{23.59} & \underline{58.24} & \underline{51.36} & \underline{57.40} & 58.10 \\
FOZO & \underline{32.15} & 0.37 & 50.10 & 19.25 & 58.07 & 50.06 & 57.37 & 57.90 \\
\rowcolor{oursgreen}\textbf{\method{}} & \textbf{35.98} & \underline{0.80} & \textbf{53.58} & \textbf{26.40} & \textbf{60.55} & \textbf{54.35} & \textbf{58.99} & \textbf{60.44} \\
\bottomrule
\end{tabular}%
}
\end{table*}

\section{Asymmetric deployment bit-widths}
\label{app:w4a8}

The bit-widths in the main text are symmetric, but integer inference deploys asymmetric ones. Weights are static and can be quantized per channel, whereas activations are dynamic and the integer kernels of mainstream accelerators are built around 8-bit multiply-accumulate, so the practical combination is aggressive weights with conservative activations. Two formats dominate, W8A8, which the main table already covers, and W4A8, which it does not. We therefore add W4A8 on the same protocol as Table~\ref{tab:main} (Table~\ref{tab:w4a8}).

We quantize it with AdaLog, following the paper's existing practice of using AdaLog wherever the weights are low-bit. That choice is forced rather than convenient. Measured under one protocol on $10{,}240$ held-out clean images, full precision reaches $85.16$, PTQ4ViT gives $85.14$ at W8A8 and $83.06$ at W6A6, and AdaLog gives $83.18$ at W4A4 and $85.22$ at W4A8, while PTQ4ViT at W4A8 collapses to $76.62$. The asymmetric format costs essentially nothing against full precision, which is the reason it is used.

\begin{table}[ht]
\centering
\caption{Asymmetric bit-widths. Mean top-1 (\%) on ImageNet-C sev.\ 5, 15 corruptions, 5k
subset, batch 64, episodic, 3-seed mean $\pm$ std; No Adapt within $0.01$ across seeds. W4A4 and W8A8 are
reproduced from Table~\ref{tab:main}.}
\label{tab:w4a8}
{\small
\begin{tabular}{l|c|c|c}
\toprule
Method & W4A4 & W4A8 & W8A8 \\
\midrule
No Adapt & 47.30 & 54.12 & 54.36 \\
T3A & \ms{47.42}{0.04} & \ms{53.74}{0.01} & \ms{54.16}{0.02} \\
FOA & \ms{50.25}{0.07} & \ms{57.55}{0.14} & \ms{57.85}{0.16} \\
ZOA & \ms{50.87}{0.09} & \ms{58.10}{0.01} & \ms{58.16}{0.05} \\
FOZO & \ms{50.09}{0.02} & \ms{57.88}{0.14} & \ms{57.74}{0.20} \\
NEO & \ms{50.29}{0.02} & \ms{58.01}{0.02} & \ms{58.09}{0.03} \\
\midrule
\rowcolor{oursgreen}\textbf{\method{}} & \ms{\mathbf{53.58}}{0.04} & \ms{\mathbf{60.11}}{0.07} & \ms{\mathbf{60.44}}{0.01} \\
\bottomrule
\end{tabular}}
\end{table}

\method{} leads at W4A8 as it does at every symmetric setting, reaching \ms{60.11}{0.07} against $58.01$ for the strongest baseline. The more informative reading is across the columns. Halving the weight bit-width from 8 to 4 moves every one of the seven rows by at most $0.42$, while dropping the activations from 8 to 4 moves the same rows by six to eight points. \method{}'s margin over the strongest baseline sorts the same way, $+2.01$ at W4A8 and $+2.28$ at W8A8 against $+2.71$ at W4A4. The quantity that governs both the difficulty of the setting and the size of the gain is the activation bit-width, which is the quantizer \method{} is defined on, and this follows from deployment-realistic settings rather than from a diagnostic one.

\section{Mechanism analysis}
\label{app:analysis}

\paragraph{Quantization amplifies OOD degradation.}
The clean-to-corruption accuracy gap of the \emph{unadapted} model widens monotonically as the model is quantized more aggressively (Table~\ref{tab:oodgap}): full precision loses $28.7$ points from clean ImageNet to ImageNet-C (severity 5), but W3A3 loses $48.4$. Quantized models are thus markedly more fragile under distribution shift than full-precision ones, consistent with prior reports \citep{deng2025test}; this is the regime our method targets. Accuracies are on the same $5000$-image ImageNet / ImageNet-C subset (No Adapt, seed $0$).

\begin{table}[ht]
\centering
\caption{Unadapted clean vs.\ ImageNet-C (sev.\ 5) accuracy by precision (5k subset, No Adapt,
seed 0). The clean$\to$corrupt gap widens as precision drops.}
\label{tab:oodgap}
{\small
\begin{tabular}{lccccc}
\toprule
 & FP (W32) & W8A8 & W6A6 & W4A4 & W3A3 \\
\midrule
Clean & $84.3$ & $84.1$ & $82.7$ & $82.9$ & $76.8$ \\
ImageNet-C & $55.6$ & $54.4$ & $47.7$ & $47.3$ & $28.4$ \\
Gap (clean $-$ corrupt) & $28.7$ & $29.8$ & $35.1$ & $35.6$ & $48.4$ \\
\bottomrule
\end{tabular}}
\end{table}

\paragraph{Intermediate activation mismatch.}
Table~\ref{tab:mismatch} measures the per-channel source--target mismatch at the late-block input-side recalibration interfaces (blocks 8--11 $\times$ \{qkv, proj, fc1, fc2\}, W6A6). The mismatch is the per-channel mean absolute deviation of the target moments from the source calibration moments, averaged over channels; the measured tensor is the activation entering the quantizer (post-recalibration for Local-only/\ours{}). Local-only and \ours{} reduce the combined mismatch by $56.0\%$, while Global-only and NEO do not reduce it at all, because they act on the final CLS representation downstream of these interfaces. The reduction is partial (not full) because $\lambda=0.5$ deliberately applies only a partial recalibration, which we find more robust than the full match (Section~\ref{sec:method}). The \emph{code-distribution distance} reported in Figure~\ref{fig:teaser}b (the $-52\%$ restoration) compares the test-time and source quantizer code histograms at all $17$ recalibration interfaces, averaged over the 15 corruptions and three seeds. The codes are compared as the \emph{values they stand for}, not as bin indices: each histogram bin is placed at the activation value its code de-quantizes to, and the distance between two histograms is the Wasserstein-1 distance along that value axis, that is, the mass-weighted distance the values move. Writing $F^{t}_{\ell}$ and $F^{s}_{\ell}$ for the cumulative test-time and source code histograms at interface $\ell$ and $v_{\ell}(c)$ for the value code $c$ de-quantizes to,
\begin{equation}
D_{\ell} \;=\; \sum_{c}\bigl|F^{t}_{\ell}(c)-F^{s}_{\ell}(c)\bigr|\,
\bigl(v_{\ell}(c{+}1)-v_{\ell}(c)\bigr),
\qquad
v_{\ell}(c)=\begin{cases}c\,\Delta^{+}_{\ell}, & c\ge 0,\\[2pt] c\,\Delta^{-}_{\ell}, & c<0,\end{cases}
\label{eq:codedist}
\end{equation}
where $\Delta^{+}_{\ell}=\Delta^{-}_{\ell}$ on the single-scale interfaces. The histogram at an interface pools its channels and tokens, which is the distribution the \emph{per-tensor} quantizer actually has: one step serves every channel, so the code distribution is a property of the tensor and not of a channel. That also makes this distance smaller than the per-channel displacement reported in Table~\ref{tab:mismatch}, since pooling lets channels that move in opposite directions cancel and $D_\ell$ is bounded by the channel-average of per-channel distances; the per-channel view is what Table~\ref{tab:mismatch} and Table~\ref{tab:levels} give. The quantity plotted is the fraction of the No-Adapt distance that a method removes,
\begin{equation}
\text{restored} \;=\; 100\,\frac{\bar{D}^{\,\textup{No Adapt}}-\bar{D}^{\,\textup{method}}}
{\bar{D}^{\,\textup{No Adapt}}},
\qquad
\bar{D}=\frac{1}{|\mathcal{L}|\,|\mathcal{K}|}\sum_{\ell\in\mathcal{L}}\sum_{k\in\mathcal{K}}D_{\ell,k},
\label{eq:coderestored}
\end{equation}
with $\mathcal{L}$ the 17 interfaces and $\mathcal{K}$ the 15 corruptions. The source histograms $F^{s}_{\ell}$ are collected in one pass over the same $12{,}800$ clean ImageNet validation images as the source moments (Appendix~\ref{app:repro}), and the test histograms over the full $5$k stream of each corruption. The diagnostic is run at W6A6, where every interface is uniformly or twin-uniformly quantized; at the AdaLog bit-widths \texttt{mlp.fc2} is log-quantized and is covered analytically instead (Proposition~\ref{prop:logq}). Comparing values rather than indices is what lets a single metric span the interfaces the method recalibrates, because their code grids are not commensurable. The intermediate interfaces use a symmetric single-scale uniform grid, on which every bin is equally wide, whereas the post-GELU \texttt{mlp.fc2} interfaces are quantized by a twin uniform quantizer, with a negative step of $0.0053$ against a positive step of up to $0.9$. A distance defined on bin indices rather than values would score a change $169\times$ smaller identically and would be dominated by the finer sub-grid; on the de-quantized axis every interface is measured in activation units, so none has to be treated separately.

\begin{table}[t]
\centering
\caption{Intermediate activation mismatch at the $16$ pre-quant / input-side
interfaces the local branch acts on (blocks 8--11 $\times$ \{attn.qkv,
attn.proj, mlp.fc1, mlp.fc2\}; $\lambda{=}0.5$) on W6A6 ViT-B.
Mismatch is the per-channel mean absolute deviation of
the target moments from the source calibration moments, averaged over channels.
Accuracy is mean top-1 over 15 corruptions (seed 0).}
\label{tab:mismatch}
{\small
\begin{tabular}{l|r|rrr|r}
\toprule
Method & Acc. & Mean Mis.\,$\downarrow$ & Std Mis.\,$\downarrow$ & Comb.\,$\downarrow$ & Reduction (\%)\,$\uparrow$ \\
\midrule
No Adapt & 47.69 & 0.235 & 0.112 & 0.347 & 0.0 \\
Local-only & 53.35 & 0.095 & 0.057 & 0.153 & 56.0 \\
Global-only & 51.16 & 0.235 & 0.112 & 0.347 & 0.0 \\
\rowcolor{oursgreen}\textbf{\method{}} & 54.36 & 0.095 & 0.057 & 0.153 & 56.0 \\
NEO & 51.36 & 0.235 & 0.112 & 0.347 & 0.0 \\
\bottomrule
\end{tabular}}
\end{table}

\paragraph{Margin and quantization error under shift.}
Proposition~\ref{prop:acc} routes the amplification through the decision: a prediction survives quantization while the propagated perturbation stays inside the unquantized logit margin, the shift shrinks that margin while the activation quantization error does not, and the error itself is set by the step. We measure those quantities directly. The full-precision and the quantized network are run over the same inputs on the $5000$-image subset, clean and all $15$ corruptions at severity~5, with $\|e_\ell\|$ and the clipped fraction read at the sixteen recalibration interfaces using the deployed quantizer's own step and clip range. Nothing adapts in this measurement, so every quantity is a per-image or per-element aggregate and does not depend on the stream order. The margin collapses under shift, from $3.78$ to $2.05$, a drop of $46\%$; it is a property of the unquantized network, so it is one number for all four precisions. The activation quantization error does not follow it: $\|e_\ell\|$ moves by between $-3.2\%$ and $+2.2\%$, and the clipped fraction falls at every bit-width, so the invariance is the granular term dominating, which is how Proposition~\ref{prop:acc} splits $\mathbb{E}\,e_{\ell,c}^2$ and is consistent with the under-filling reported above. Across precisions $\|e_\ell\|$ instead tracks the step, rising monotonically from $2404$ at W8A8 to $17436$ at W3A3. The consequence is the predicted one: the fraction of predictions that are correct in full precision and wrong once quantized grows under shift at every precision, and it grows more the coarser the step, by $2.7$ points at W8A8 against $19.5$ at W3A3 (Table~\ref{tab:marginflip}).

\begin{table}[ht]
\centering
\caption{\emph{The shift collapses the margin while leaving the quantization error where it was}
(ViT-B, $5000$-image subset, clean and all 15 ImageNet-C corruptions at severity~5, No Adapt).
$\|e_\ell\|$ and the clipped fraction are measured at the sixteen recalibration interfaces;
harmful flips are predictions correct in full precision and wrong once quantized. The
full-precision margin is $3.78$ clean against $2.05$ corrupted, independent of the bit-width.}
\label{tab:marginflip}
{\small
\begin{tabular}{l|cc|cc|cc}
\toprule
& \multicolumn{2}{c|}{$\|e_\ell\|$} & \multicolumn{2}{c|}{clipped (\%)} & \multicolumn{2}{c}{harmful flips (\%)} \\
& clean & corrupt & clean & corrupt & clean & corrupt \\
\midrule
W8A8 & $2404$   & $2457$   & $0.011$ & $0.011$ & $1.08$  & $3.82$ \\
W6A6 & $6773$   & $6600$   & $0.045$ & $0.035$ & $3.50$  & $11.79$ \\
W4A4 & $12384$  & $11992$  & $0.900$ & $0.817$ & $3.18$  & $11.86$ \\
W3A3 & $17436$  & $17233$  & $2.717$ & $2.369$ & $10.18$ & $29.70$ \\
\bottomrule
\end{tabular}}
\end{table}

\paragraph{Precision sweep.}
Table~\ref{tab:precision} sweeps the local-branch gain across the tested quantization precisions (W3A3--W8A8). The Local Gain (Local-only $-$ No Adapt) is substantial at every precision ($+4.5$ to $+5.9$), confirming that the local branch is a general activation-recalibration principle across the quantization range; its \emph{marginal} contribution on top of the global branch grows as precision drops, peaking at the most aggressive W3A3 setting ($+3.88$).

\begin{table}[t]
\centering
\caption{Precision sweep of the local recalibration branch on ImageNet-C (3-seed mean
top-1 \% over seeds 0/1/2; accuracy columns from Table~\ref{tab:main}). All variants use source-initialized cumulative running target moments,
lagged order, $\lambda{=}0.5$. Local Gain $=$ Local-only $-$ No Adapt;
Marg.\ Local Gain $=$ \method{} $-$ Global-only.}
\label{tab:precision}
{\small
\begin{tabular}{l|rrrr|rr}
\toprule
Precision & No Adapt & Global-only & Local-only & \method{} & Local Gain & Marg.\ Local Gain \\
\midrule
W8A8 & 54.36 & 58.10 & 59.11 & 60.44 & +4.75 & +2.34 \\
W6A6 & 47.68 & 51.20 & 53.42 & 54.35 & +5.74 & +3.15 \\
W4A4 & 47.30 & 50.21 & 52.87 & 53.58 & +5.57 & +3.37 \\
W3A3 & 28.35 & 31.43 & 35.30 & 35.98 & +6.95 & +4.55 \\
\bottomrule
\end{tabular}}
\end{table}

\paragraph{Placement ablation.}
Table~\ref{tab:placement} isolates \emph{where} the same correction is applied, with every variant at the deployed local configuration so that only the site differs. Before the quantizer, on the tensor it is about to round, is best at every bit-width. After the quantizer, on the dequantized tensor of the same interface, the codes the quantizer selects are those of the uncorrected activation: accuracy falls below No Adapt at W3A3 and W4A4 ($27.98$ and $47.01$ against $28.35$ and $47.30$), because the drifted activation is clipped and coarsely rounded before any correction and the affine then rescales an already truncated signal, and it recovers $4.2$ and $3.6$ points at W6A6 and W8A8 ($51.90$ and $57.96$), close to the Global-only branch. What the correction must change is which codes are used, not the values downstream of them. At the block output, the residual stream that no quantizer reads and the site at which normalization-statistics adaptation \citep{schneider2020improving,mirza2023actmad} acts, the same affine reaches $32.18$, $50.32$, $51.14$ and $57.29$: about half the gain of the quantizer-input placement ($2.9$ to $3.8$ points over No Adapt against $4.8$ to $7.0$), within a point of our Global-only branch, which recalibrates only the head input, and $1.8$ to $3.1$ points below the quantizer inputs, a gap that widens as precision falls. Both placements away from the quantizer input therefore stop at the gain of aligning feature statistics, and the further $1.8$ to $3.1$ points appear only when the correction sits at the quantizer input. That is the site the code-distribution diagnostic (Figure~\ref{fig:teaser}b) and Proposition~\ref{prop:levels} single out, and it is what separates \method{} from normalization-statistics adaptation applied to a quantized network. The block-output source moments are collected on the same calibration batches as the quantizer-input moments.

\begin{table}[t]
\centering
\caption{Placement ablation (mean top-1 \%, 3-seed mean $\pm$ std). Every variant uses the deployed
local configuration (blocks 8--11, $\lambda{=}0.5$, cumulative estimator, no global branch); only where
the same per-channel affine is applied differs: before the activation quantizer, on the tensor it is
about to round (our default, the Local-only row of Table~\ref{tab:main}); after it, on the dequantized
tensor of the same interface; or at the block output, the residual stream that no quantizer reads and
the site at which normalization-statistics adaptation acts. The Global-only branch of Table~\ref{tab:main},
which recalibrates only the head input, is repeated for reference. Best in \textbf{bold}.}
\label{tab:placement}
{\small
\begin{tabular}{l|cccc}
\toprule
Placement & W3A3 & W4A4 & W6A6 & W8A8 \\
\midrule
No Adapt & \ms{28.35}{0.00} & \ms{47.30}{0.01} & \ms{47.68}{0.00} & \ms{54.36}{0.00} \\
\textbf{Before the quantizer, input side (\method{})} & \ms{\mathbf{35.30}}{0.07} & \ms{\mathbf{52.87}}{0.03} & \ms{\mathbf{53.42}}{0.05} & \ms{\mathbf{59.11}}{0.02} \\
After the quantizer, output side & \ms{27.98}{0.03} & \ms{47.01}{0.07} & \ms{51.90}{0.01} & \ms{57.96}{0.04} \\
Block output, residual stream & \ms{32.18}{0.05} & \ms{50.32}{0.12} & \ms{51.14}{0.05} & \ms{57.29}{0.02} \\
\midrule
Global-only (Table~\ref{tab:main}) & \ms{31.43}{0.02} & \ms{50.21}{0.02} & \ms{51.20}{0.03} & \ms{58.10}{0.01} \\
\bottomrule
\end{tabular}}
\end{table}

\section{Measured test-time cost}
\label{app:cost}

All quantized models in this paper run in the PTQ4ViT and AdaLog PyTorch harness, which quantizes and dequantizes weights and activations in floating point, and every cost below is measured in that harness. What transfers to a deployment is therefore the ratio between methods rather than the absolute time, which Appendix~\ref{app:ondevice} reproduces on a Jetson Orin Nano; the scope of an integer-only deployment is set out in Appendix~\ref{app:deploy}. Table~\ref{tab:cost} reports the measured per-method test-time cost at W6A6: forward-pass compute (\texttt{torch.profiler} GFLOPs over the entire per-batch adaptation), peak GPU memory (\texttt{max\_memory\_allocated}, per-process and hence independent of other GPU tenants), and wall-clock latency (per-image forward, timed on an idle NVIDIA TITAN RTX with one measurement at a time, as the median over 5 repetitions of the per-batch forward median across ${\sim}50$ batches). \method{} stays in the single-forward tier on all three axes: it runs at $1.10\times$ the No Adapt latency, versus ${\sim}2\times$ for the two-forward FOA/FOZO/ZOA, while peak memory is dominated by the frozen quantized backbone and batch activations, which \method{} does not add to. Applying the recalibration as a fresh tensor would: the blend needs the input and the corrected copy live at once, and at the \texttt{mlp.fc2} interface that temporary is the $3072$-wide MLP hidden state, four times the width of the other fifteen, which alone costs $148$\,MB at batch $64$. The deployed configuration instead applies the affine in place over batch slices (\texttt{-{}-quar\_local\_chunk 2}), which leaves $1601.08$\,MB against No Adapt's $1600.91$\,MB, a $0.17$\,MB residual on top of the ${\sim}170$\,KB of stored statistics. This is bit-identical rather than an approximation, so no accuracy number depends on it: every operation is elementwise or a per-channel broadcast and the two reductions that span the tensor are taken before any slice is touched, and we confirmed it end to end, with top-1 and top-5 equal to the buffered path at batch $64$ and batch $16$ ($46.80/69.50$ and $46.54/69.30$ on \texttt{gaussian\_noise}). T3A's per-sample nearest-neighbour search adds latency not visible in matmul GFLOPs and grows with its support set, so its reported value is early-stream.

\begin{table}[t]
\centering
\caption{Measured test-time cost on W6A6 ViT-B (batch size 64, ImageNet-C). \#FP: forward
passes per batch. GFLOPs/image profiled over the full per-batch adaptation. Peak memory:
\texttt{max\_memory\_allocated} (per-process), median over 5 reps. Latency: per-image forward
(protocol above). $^{*}$T3A latency is early-stream (grows with its support set).}
\label{tab:cost}
{\small
\begin{tabular}{l|cccc}
\toprule
Method & \#FP & GFLOPs/img & Peak mem (MB) & Latency (ms/img) \\
\midrule
No Adapt & 1 & 35.2 & 1601 & 6.0 \\
T3A & 1 & 35.2 & 1627 & 8.2$^{*}$ \\
FOA & 2 & 71.4 & 1619 & 13.0 \\
ZOA & 2 & 70.3 & 1606 & 12.3 \\
NEO & 1 & 35.2 & 1602 & 6.0 \\
FOZO & 2 & 71.4 & 1619 & 12.3 \\
\rowcolor{oursgreen}\textbf{\method{}} & 1 & 35.2 & 1601 & 6.6 \\
\bottomrule
\end{tabular}}
\end{table} Both the memory and latency figures are reported as ratios between quantities that grow together
with the batch, so they characterise the methods rather than the batch size they were measured at.

\section{On-device measurement on a Jetson Orin Nano}
\label{app:ondevice}

The cost argument is about edge hardware, so we also ran it there: a Jetson Orin Nano Super devkit, Orin \texttt{sm\_87}, with the GPU sharing $7.31$\,GiB with the CPU. What transfers across a change of device is the ratio between methods rather than the absolute time, so Table~\ref{tab:ondevice} reports every side by ratio. The ordering is reproduced and the two-forward group again lands at ${\approx}2\times$. \method{} is \emph{cheaper} in relative terms on the device than on the server, $1.071\times$ against $1.10\times$. T3A moves the other way, $1.144\times$ against $1.36\times$, because its per-sample nearest-neighbour search is bound by dispatch rather than arithmetic and the forward it is measured against is $8.6\times$ slower here.

Batch $16$ is the informative column: it is the largest batch at which the backprop-based methods fit this board, so it is the only one where all four families sit on a single axis, and per-image latency there is within $3.9\%$ of batch $64$, so it is a cheap batch to deploy at. Memory is what separates the families. At the protocol batch every forward-only method fits, peaking between $1603$ and $1751$\,MB, and none of the four backprop-based methods does: TENT exhausts memory and a batch-$32$ retry hard-reset the board, because on Tegra the GPU allocates from system RAM and exhausting it starves the kernel rather than raising a clean CUDA OOM. At batch $16$ TENT, EATA and SAR do fit, at $3300$--$3306$\,MB against \method{}'s $910$\,MB, a $3.6\times$ gap, while SPA still does not: the kernel's OOM killer terminated it at $4256$\,MB resident, with the triggering request arriving through Tegra's allocator, so of the eleven methods SPA is the only one that runs at batch $1$ alone. Batches $2$, $4$ and $8$ were not measured. Probing this safely needs one precaution, since a bare attempt can take the board down: capping the caching allocator first, here to $4200$ of $7486$\,MB, leaves the kernel enough to survive and lets the OOM killer terminate only the process, which also leaves a record where a reboot would leave none.

The in-place correction that Appendix~\ref{app:cost} describes behaves the same way here (Table~\ref{tab:chunkmem}). Buffering the corrected tensor instead costs $148$\,MB at batch $64$ and $37$\,MB at batch $16$, whereas applied in place the residual over an unadapted forward is $0.17$\,MB at both batches and on both devices, and the latency change is $-0.25\%$ and $+0.35\%$, inside run-to-run noise.

\begin{table}[t]
\centering
\caption{Per-image latency as a ratio to No Adapt on the same device and batch (W6A6 ViT-B).
Device values are medians over $15$ batches at batch $64$ and $25$ at batch $16$, over $3$
repetitions. Batch $16$ is the largest batch at which the backprop-based methods fit this board.}
\label{tab:ondevice}
{\small
\begin{tabular}{l|c|cc}
\toprule
 & TITAN RTX & \multicolumn{2}{c}{Orin Nano} \\
\cmidrule(lr){2-2}\cmidrule(lr){3-4}
Method & batch $64$ & batch $64$ & batch $16$ \\
\midrule
No Adapt & $1.00\times$ & $1.000\times$ & $1.000\times$ \\
NEO      & $1.00\times$ & $1.002\times$ & $1.000\times$ \\
\rowcolor{oursgreen}\textbf{\method{}} & $\mathbf{1.10\times}$ & $\mathbf{1.071\times}$ & $\mathbf{1.084\times}$ \\
T3A      & $1.36\times$ & $1.144\times$ & $1.465\times$ \\
ZOA      & $2.04\times$ & $2.001\times$ & $1.998\times$ \\
FOZO     & $2.04\times$ & $2.022\times$ & $2.013\times$ \\
FOA      & $2.16\times$ & $2.034\times$ & $2.061\times$ \\
\midrule
TENT     & $1.87\times$ & \emph{no fit} & $1.890\times$ \\
EATA     & $1.90\times$ & \emph{no fit} & $1.933\times$ \\
SAR      & $3.71\times$ & \emph{no fit} & $3.777\times$ \\
SPA      & $5.19\times$ & \emph{no fit} & \emph{no fit} \\
\bottomrule
\end{tabular}}
\end{table}

\begin{table}[t]
\centering
\caption{Peak memory of the local correction, buffered against applied in place
(\texttt{-{}-quar\_local\_chunk 2}), at W6A6. Residual is over No Adapt on the same device and
batch. The two forms are bit-identical, so no accuracy figure depends on the choice.}
\label{tab:chunkmem}
{\small
\begin{tabular}{ll|rrrr}
\toprule
Batch & Device & No Adapt & Buffered & In place & Residual \\
\midrule
$64$ & TITAN RTX  & $1600.91$ & $1748.83$ & $\mathbf{1601.08}$ & $+0.17$ \\
$64$ & Orin Nano  & $1603.16$ & $1751.33$ & $\mathbf{1603.33}$ & $+0.17$ \\
$16$ & Orin Nano  & $\phantom{0}909.47$ & $\phantom{0}946.58$ & $\phantom{0}\mathbf{909.64}$ & $+0.17$ \\
\bottomrule
\end{tabular}}
\end{table}

\section{Relation to backprop-based test-time adaptation}
\label{app:backprop}

Our comparison set (Table~\ref{tab:main}) targets the TTA methods that run on the same \emph{frozen, quantized} graph as inference, which are backprop-free and use at most a few forward passes, making them realistic candidates for quantized edge deployment. A second, larger family is \emph{backprop-based} -- it backpropagates an unsupervised loss and updates model parameters: entropy-minimization adaptation (TENT \citep{wang2021tent}, EATA \citep{niu2022efficient}, SAR \citep{niu2023towards}), multi-layer activation-moment matching (ActMAD \citep{mirza2023actmad}), and auxiliary self-supervised feature alignment (TTT++ \citep{liu2021ttt}). Table~\ref{tab:regime} contrasts the two families. Backprop-based methods are structurally ill-suited to quantized inference: (i) they require backpropagation through the (fake-)quantized graph, which the integer inference kernels of quantized accelerators do not implement; (ii) they carry an optimizer plus gradient and activation buffers; and (iii) they overwrite model parameters, which on a quantized model means re-quantization or full-precision shadow weights and risks drift over a long stream. We therefore do not use these as primary baselines; we nonetheless run four of them (TENT, EATA, SAR, and SPA \citep{spa2025}) on the same W6A6 model with their official code and hyperparameters (Table~\ref{tab:backprop_acc}). Each requires an added straight-through estimator to make the frozen quantized graph differentiable. Even granted that, the three lightweight ones do not reach \ours{}'s backprop-free, single-forward accuracy. SPA does exceed it, at the cost of three forward passes and a backward pass per batch and a peak memory footprint that no longer fits the edge devices that quantization targets. We do not tabulate ActMAD. It is full-model test-time training that updates \emph{all} parameters by gradient (and likewise needs the added STE), a far heavier regime than the lightweight TTA compared here.

\method{} shares with ActMAD the premise that per-channel activation moments are the quantity to align at test time, and departs from it in what is aligned, where, and how much. ActMAD aligns by descending a matching loss, so its correction is realized in the model parameters and the layers it measures are a free design choice. The local branch (\emph{Local-only} in Table~\ref{tab:main}) instead applies the correction as an operator on the tensor each frozen quantizer reads, which is what fixes where it has to act. Its target is therefore the quantizer's source-calibrated operating point rather than the feature distribution itself, it matches in closed form against the same statistics that fixed those quantizers, and it matches only \emph{partially} ($\lambda{<}1$), where its accuracy optimum lies (Proposition~\ref{prop:shrink}). The placement and the partial strength are both load-bearing. The same matching applied away from the late pre-quant interfaces falls below No Adapt (Table~\ref{tab:scope}), and a full match ($\lambda{=}1$) is worse than a partial one (Table~\ref{tab:sens}).

\begin{table}[t]
\centering
\caption{Two families of test-time adaptation. The \emph{backprop-based} family updates
parameters by gradient and cannot run on a frozen quantized inference graph; our comparison set
(and \method{}) is \emph{backprop-free}. F: forward, B: backward pass per batch.}
\label{tab:regime}
\resizebox{\textwidth}{!}{%
\begin{tabular}{l|cccc}
\toprule
Method & Backprop & Updates params & Passes/batch & Frozen-quant deployable \\
\midrule
TENT / EATA / SAR & yes & LN affine & 1F\,+\,1B & no \\
ActMAD & yes & all params & 1F\,+\,1B & no \\
TTT++ & yes & feature extractor & 1F\,+\,1B & no \\
\midrule
No Adapt & no & none & 1F & yes \\
T3A / NEO & no & none & 1F & yes \\
FOA / FOZO & no & prompt (grad-free) & 2F & yes \\
ZOA & no & LN affine (grad-free) & 2F & yes \\
\rowcolor{oursgreen}\textbf{\method{}} & no & none & 1F & yes \\
\bottomrule
\end{tabular}%
}
\end{table}

\begin{table}[t]
\centering
\caption{Seed spread behind Table~\ref{tab:backprop_acc}, which reports the means alone
(W6A6, ImageNet-C 15-corruption mean, 3 seeds, $\pm$ sample std). No Adapt performs no
adaptation and its spread across seeds is below $0.01$.}
\label{tab:backprop_std}
{\small
\begin{tabular}{l|cccccc}
\toprule
 & No Adapt & SAR & TENT & EATA & SPA & \method{} \\
\midrule
Mean top-1 & $47.68$ & \ms{49.77}{0.11} & \ms{50.05}{0.13} & \ms{52.79}{0.53} & \ms{59.14}{0.15} & \ms{\mathbf{54.35}}{0.02} \\
\bottomrule
\end{tabular}}
\end{table}

\section{ONNX representability check}
\label{app:deploy}

We verify that \ours{} is representable in a standard quantized inference graph rather than only in our PyTorch harness. Exporting the W8A8 model to ONNX (opset~17) succeeds, and each local recalibration site is exactly representable as a standard per-channel \texttt{Mul}/\texttt{Add} pair: ONNX Runtime reproduces the PyTorch logits, and the recalibration op itself matches bit-for-bit at the site level. The scale factor folds exactly into the activation quantizer, whereas the continuous per-channel shift over a per-tensor quantizer does not, so the explicit \texttt{Mul}/\texttt{Add} wrapper is the intended deployable representation; it adds one multiply and one add per recalibrated activation and updates no weights or quantizer parameters. This is a representability/executability check, not a latency, memory, or hardware claim. A full integer-only deployment with end-to-end accuracy parity, including the per-channel quantization that would let the shift fold and the aggressive W3A3/W4A4 bit-widths, is left to future work.

\section{Proofs and assumptions}
\label{app:theory}

This appendix states the full assumptions, gives the formal statements of Lemma~\ref{prop:qerr} and Proposition~\ref{prop:acc} that Section~\ref{sec:why} summarizes in prose, and proves them together with Proposition~\ref{prop:levels} and the recalibration-strength optimum Proposition~\ref{prop:shrink}, and adds a drift-free continual-adaptation result (Proposition~\ref{prop:drift}). The arguments are at the level of per-channel first and second moments and rest on the assumptions below; they are explanatory, not worst-case guarantees.

\paragraph{Setup and assumptions.}
We analyze a \emph{simplified uniform-quantizer model}. Consider one recalibrated interface and one channel $c$. The frozen activation quantizer is a $b$-bit uniform affine quantizer with step $\Delta_c$ and clipping range $[\mu_{s,c}-\alpha\sigma_{s,c},\ \mu_{s,c}+\alpha\sigma_{s,c}]$ \emph{calibrated on source data}, where $\alpha$ is the calibration clip factor and $(\mu_{s,c},\sigma_{s,c})$ are the source per-channel moments. This describes the activation quantizers of PTQ4ViT \citep{yuan2022ptq4vit} (twin uniform quantizers; our W6A6/W8A8 setting) at \emph{every} interface, and, crucially, also the $\{$qkv, proj, fc1$\}$ interfaces under AdaLog \citep{wu2024adalog} (W4A4/W3A3): AdaLog applies its logarithmic quantizer only to the post-Softmax and post-GELU activations, so among our recalibration interfaces only the post-GELU input to \texttt{mlp.fc2} is log-quantized, and only at W4A4/W3A3. The uniform model below therefore applies \emph{directly} to three of the four module types at all bit-widths; the single log-quantized interface is treated separately in paragraph~(iv) (Proposition~\ref{prop:logq}). We assume: \textbf{(A1)} activation shift is captured by the per-channel mean and variance, i.e.\ the standardized activation $(x_c-\mu_{\cdot,c})/\sigma_{\cdot,c}$ has an approximately shift-invariant, sub-Gaussian shape; \textbf{(A2)} the quantizer is calibrated to the source moments as above; \textbf{(A3)} the running target moments $(\mu_{t,c},\sigma_{t,c})$ track the test stream (slowly varying); \textbf{(A4)} the source moments are representative of the calibration distribution; and, used only for the optimality statement in paragraph~(ii), \textbf{(A5)} the frozen path from the interface to the logits is calibrated to be accurate at the source operating point, with expected task loss increasing locally quadratically in the per-channel operating-point deviation. (A1) is the standard covariate-shift premise behind test-time batch-norm statistic adaptation \citep{schneider2020improving}.

\paragraph{(i) Recalibration restores the source-calibrated quantization error.}
\begin{lemma}[Restored quantizer operating point]\label{prop:qerr}
Under \textup{(A1)--(A2)}, an activation whose per-channel moments have been matched to their source
values $\mu_{s,c}$ and $\sigma_{s,c}$ reoccupies the range the quantizer was calibrated on, and its
code occupancy matches that of the source-calibrated quantizer. A moment mismatch instead over- or
under-fills that range, displacing the code occupancy.
\end{lemma}
\vspace{-4pt}
\begin{proof}[Proof of Lemma~\ref{prop:qerr}]
Matching the per-channel moments fixes the affine map from the standardized activation onto the grid,
and by (A1) the standardized shape is preserved, so the induced distribution over code levels is the
one the quantizer was calibrated on; that is the occupancy claim. Its two familiar consequences
follow. The quantization error splits into a granular term and a clipping term
\citep{banner2019post}.
The granular term is $\Delta_c^2/12$, independent of the input, so relative to the signal power
the $\mathrm{SQNR}\approx 12\,\mathrm{Var}(\hat{x}_c)/\Delta_c^2$. After full recalibration
$\mathrm{Var}(\hat{x}_c)=\sigma_{s,c}^2$, exactly the variance $\Delta_c$ was calibrated for, so
the SQNR returns to its source value. The clipping term is the tail energy beyond
$\pm\alpha\sigma_{s,c}$ around $\mu_{s,c}$; by (A1) the standardized shape is preserved, so after
recalibration the standardized clip threshold is again $\pm\alpha$ and the clipping probability
$P(|Z|>\alpha)$ returns to its source value. Without recalibration a mean offset pushes mass
past one clip edge, an inflated $\sigma_{t,c}>\sigma_{s,c}$ moves mass into both tails (both raise
the clipping energy), and an under-scaled $\sigma_{t,c}<\sigma_{s,c}$ lowers the SQNR.
\end{proof}
This matches the diagnostic: the local branch reduces the per-channel source--target moment mismatch (Table~\ref{tab:mismatch}), whereas a final-layer recalibration (Global-only, NEO) acts downstream of these interfaces and leaves it untouched. The mismatch is a property of the shifted activation rather than something quantization newly induces: the post-dequantization tensor is nearly identical to the pre-quant one at these interfaces (Table~\ref{tab:mismatch}), so what the quantizer sees is displaced for the same reason the activation is.

\paragraph{(i-b) What a displaced operating point costs, before asking how much.}
Lemma~\ref{prop:qerr} says what restoration returns. The next statement isolates what a \emph{displaced} operating point costs, and it is deliberately free of any claim about magnitude.

\begin{proof}[Proof of Proposition~\ref{prop:levels}]
(i) The calibrated range is $\pm\rho\sigma_{s,c}$ about $\mu_{s,c}$ in steps of $\Delta$, so the
occupied level count is proportional to the scale and $\log_2(L_{s,c}/L_{t,c})=b_c$. (ii) A random
variable supported on at most $L$ values has entropy at most $\log_2 L$, and
$I(Q(x_c);y)\le H(Q(x_c))$. (iii) A floating-point representation allocates a fixed mantissa to
every magnitude, so scaling the activation leaves the number of representable values unchanged;
a quantizer calibrated on the test stream would set $\Delta\propto\sigma_{t,c}$, leaving $L$ fixed.
(iv) Immediate from (i), and $b_c$ is channel-dependent, so one scalar gain cannot zero every
$b_c$. Only the frozen step $\Delta$ is common to the channels; the occupied fraction of it is not.
\end{proof}
Two things this does not say. It bounds a channel's information budget, not the accuracy that budget buys, so translating $b_c$ into points would need a model of how the decision weights channels, which we do not assume. And it is silent on channels that expand, for which $b_c<0$ and the binding constraint is clipping instead (Lemma~\ref{prop:qerr}).

\emph{Which term operates in practice.} Empirically the clipping term is negligible at the tested bit-widths (activation saturation is $\leq 0.1\%$ at W6A6 and $\approx 3\%$ at W4A4, and recalibration does not reduce it), so the operative restored quantity is the code distribution rather than a lowered clipping rate. Neither figure is surprising. At W6A6 the calibrated clip edges sit between $5.5$ and $18.6$ source standard deviations from the mean at the recalibration interfaces, while the shift the corruptions produce is $0.082\sigma$ in the mean and a contraction of the scale to $0.92$ of its calibrated value, both measured per channel against that channel's own source scale, so the distribution never approaches the edges and the tail mass beyond them is unchanged either way. At W4A4 the edges are much closer, sixteen levels forcing calibration to trade clipping for a finer step, but the $3\%$ is then the rate the quantizer was \emph{calibrated} to have: recalibration returns the activation to the source operating point, and that point clips at $3\%$. Restoring the clipping probability and lowering it are different things, and Lemma~\ref{prop:qerr} claims only the first. Two facts fix which part of the mismatch is doing the work. The granular error power $\Delta_c^2/12$ is set by the frozen step and no test-time correction can change it, so what recalibration can move is where the activation sits on the code grid and how many of its levels each channel occupies, not how finely that grid is spaced. And the mismatch is dominated by position: at the recalibration interfaces the per-channel mean deviation is $0.235$ against $0.112$ for the scale (Table~\ref{tab:mismatch}), a factor of $2.1$. These are activation units rather than the per-channel $\sigma$ used just above, which is the comparison that matters for the code distribution: one step $\Delta$ serves every channel of an interface, so the two components are already on a common scale and dividing each by its own $\sigma$ would compare quantities the quantizer never sees separately. A mean offset leaves the granular error power untouched while relocating the activation on the frozen grid, which is precisely what the code distribution registers and what the frozen downstream weights see.

Aggregate SQNR is the wrong statistic for this, and measurably so. It is a ratio of sums, $\sum_c\sigma_c^2$ against $d\,\Delta^2/12$, so it is invariant to how a fixed total activation energy is allocated across channels, whereas the resolution each channel actually receives is a per-channel quantity (Proposition~\ref{prop:levels}). The two disagree on this stream. Corruption costs $0.53$\,dB of aggregate SQNR at these interfaces and \method{} leaves it essentially where it finds it ($-0.06$\,dB), yet per channel $20.6\%$ lose at least half a bit of code resolution and $4.8\%$ lose a full bit, which \method{} cuts to $6.9\%$ and $0.3\%$ (Table~\ref{tab:levels}). The aggregate figure understates the per-channel degradation by $4.5\times$, and the understatement is not a small-denominator artifact: restricting to channels above the median source scale raises it rather than removing it. What the corruption does to these interfaces is under-fill them, the summed per-channel second moment falling to $0.92$ of its calibration value at $15$ of the $16$ interfaces, which is a displacement of the operating point and not a coarsening of the grid. The failure is positional, and so is the correction.

\begin{table}[ht]
\centering
\caption{\emph{Per-channel code resolution lost under shift} (W6A6 ViT-B, ImageNet-C sev.\ 5,
all 15 corruptions, 5k subset; all $17$ recalibration interfaces). For
each channel we report $b_c=\log_2(\sigma_{s,c}/\sigma_{t,c})$, the code levels it forfeits against
its own calibration (Proposition~\ref{prop:levels}); the mean counts a channel that expands as zero,
since the proposition is silent on those. T3A, NEO and Global-only act at or after the classifier,
downstream of these interfaces, and are identical to No Adapt to every decimal. FOA and FOZO
perturb input-side prompts and make the starvation worse. ZOA writes directly upstream, adapting
the LayerNorm affine parameters of blocks 1--8, and still removes $1\%$. Local-only coincides with
\method{} because both apply the local branch here. The head column is included for scope rather than
as a second restoration result: the head input is post-LayerNorm and starves under no method, with
no configuration putting a single one of its $768$ channels past half a bit, so differences within
that column are between negligible quantities and are not restoration.}
\label{tab:levels}
{\small
\begin{tabular}{l|ccccc|c}
\toprule
 & \multicolumn{5}{c|}{$16$ local interfaces ($21{,}504$ channels)} & head ($768$) \\
 & mean $b_c$ & $b_c\!\geq\!0.5$ & $b_c\!\geq\!1$ & 99th pct. & worst & mean $b_c$ \\
\midrule
No Adapt   & $0.277$ & $20.6\%$ & $4.8\%$ & $1.43$ & $2.51$ & $0.003$ \\
T3A        & $0.277$ & $20.6\%$ & $4.8\%$ & $1.43$ & $2.51$ & $0.003$ \\
FOA        & $0.319$ & $24.3\%$ & $6.3\%$ & $1.54$ & $3.27$ & $0.115$ \\
FOZO       & $0.386$ & $28.7\%$ & $10.6\%$ & $1.83$ & $3.71$ & $0.221$ \\
ZOA        & $0.274$ & $20.4\%$ & $4.6\%$ & $1.43$ & $2.49$ & $0.007$ \\
NEO        & $0.277$ & $20.6\%$ & $4.8\%$ & $1.43$ & $2.51$ & $0.030$ \\
Global-only & $0.277$ & $20.6\%$ & $4.8\%$ & $1.43$ & $2.51$ & $0.001$ \\
Local-only & $\underline{0.178}$ & $\underline{6.9\%}$ & $\underline{0.3\%}$ & $\underline{0.78}$ & $\underline{2.05}$ & $0.004$ \\
\rowcolor{oursgreen}\textbf{\method{}} & $\mathbf{0.178}$ & $\mathbf{6.9\%}$ & $\mathbf{0.3\%}$ & $\mathbf{0.78}$ & $\mathbf{2.05}$ & $\mathbf{0.000}$ \\
\bottomrule
\end{tabular}}
\end{table}

On the $4{,}432$ channels the shift actually starves (those losing at least half a bit) \method{} removes $48\%$ of the loss, and the fraction of channels running below five effective activation bits at W6A6 falls from $4.8\%$ to $0.3\%$. No compared method removes more than $1\%$, and two of them add to it. The split is not a matter of degree but of where a method writes: the three that act at or after the classifier cannot reach these interfaces at all, and of the three that do write upstream, the two prompt-based ones move the activations further off the grid they were calibrated on while \method{} moves them back.

\emph{Where the starvation lives.} It is not spread over the interfaces. Splitting them by whether a LayerNorm sits immediately before the quantizer separates them completely (Table~\ref{tab:levelsgroup}): at the interfaces that follow one, \texttt{attn.qkv} after \texttt{norm1} and \texttt{mlp.fc1} after \texttt{norm2}, not a single channel loses half a bit, and the worst loses $0.31$. The starvation is confined to the two interfaces with no normalization in front of them, the attention output and the post-GELU input, where $28.9\%$ of channels pass that threshold and the worst loses $2.51$ bits. This is consistent with the normalization returning the activation to a scale close to the one the quantizer was calibrated for before it is quantized, leaving the frozen grid its calibrated fill; we report it as an observation about where the mechanism operates rather than as a further claim. The head input is post-LayerNorm and behaves accordingly, losing $0.003$ bits on average under No Adapt with no channel past $0.07$. No compared method takes any head channel past half a bit either, the largest excursion any method produces at a post-LayerNorm interface being $0.46$ bits, against $3.71$ at an unnormalized one. The two post-LayerNorm groups were checked for all nine configurations, so this is measured rather than inferred from where each method writes; T3A, which touches only the classifier, reproduces the No-Adapt head statistics exactly.

\emph{Is restoring the resolution worth anything?} Proposition~\ref{prop:levels} predicts that matching the \emph{scale}, as opposed to only recentering, should be worth more where resolution has actually been forfeited. Because the two interface groups are disjoint, this can be tested within each group, which holds module type, width and depth fixed and toggles only whether the scale is matched. At W6A6 with the local branch alone, against No Adapt's $47.68$: recentering only (\emph{mean}) gives $50.21\pm0.07$ at the eight post-LayerNorm interfaces and $49.77\pm0.04$ at the eight unnormalized ones, and adding the rescale (\emph{mean+var}) gives $51.25\pm0.05$ and $52.49\pm0.08$, all 3-seed. The rescale is therefore worth $+1.04$ where no resolution was lost and $+2.72$ where $28.9\%$ of channels forfeited half a bit or more, a factor of $2.6$, and the unnormalized group alone accounts for $97\%$ of the rescale's effect over all sixteen ($+2.79$). Recentering runs the other way, worth more at the post-LayerNorm interfaces ($+2.52$ against $+2.08$), where the scale is already close and the displacement is nearly a pure mean offset. The control reproduces Table~\ref{tab:factorial}: recentering only over all sixteen gives $50.54\pm0.01$ against that table's $50.55$. Two limits on the reading. The $+1.04$ at the interfaces with nothing to restore shows the rescale also does something other than restoring resolution, so the $+2.72$ cannot be assigned to resolution alone; the defensible statement is the ratio, not a split of the points. And this compares a fitted quantity across groups that differ in more than resolution, which is why the comparison is made within each group rather than between them. The global branch therefore does not act by restoring code resolution, and the head's code-distribution restoration is a displacement of the mean rather than of the scale. Proposition~\ref{prop:levels} covers the local branch, and the value-space distance of Equation~\ref{eq:codedist}, which moves with both, is what covers both.

\begin{table}[ht]
\centering
\caption{\emph{Code-level starvation is confined to the quantizers with no LayerNorm in front of
them} (W6A6, ImageNet-C sev.\ 5, all 15 corruptions). Groups partition the $16$ local interfaces
by whether a LayerNorm immediately precedes the quantizer; the head input is listed separately and
is also post-LayerNorm. $b_c\!\geq\!0.5$ is the fraction of channels forfeiting at least half a bit
of code resolution.}
\label{tab:levelsgroup}
{\small
\begin{tabular}{l|c|cc|cc}
\toprule
& & \multicolumn{2}{c|}{No Adapt} & \multicolumn{2}{c}{\method{}} \\
Interface group & channels & mean $b_c$ & $b_c\!\geq\!0.5$ & mean $b_c$ & $b_c\!\geq\!0.5$ \\
\midrule
\texttt{attn.proj}, \texttt{mlp.fc2} (no preceding norm) & $15{,}360$ & $0.365$ & $28.9\%$ & $0.238$ & $9.7\%$ \\
\texttt{attn.qkv}, \texttt{mlp.fc1} (post-LayerNorm) & $6{,}144$ & $0.056$ & $0.0\%$ & $0.030$ & $0.0\%$ \\
head input (post-LayerNorm) & $768$ & $0.003$ & $0.0\%$ & $0.000$ & $0.0\%$ \\
\bottomrule
\end{tabular}}
\end{table}
We deliberately stop there and do not convert bits into accuracy points: the bit-width exchange rate prices a change of precision across the whole network, including the weights and every interface we leave alone, so it cannot price a local statistic at sixteen activation interfaces, and accuracy is not linear in log-power either. The claim Proposition~\ref{prop:levels} supports is directional, that a frozen grid can only lose resolution when the activation contracts and that the lost resolution caps what those channels can carry, and that is the claim the table tests.

Measuring the quantizer code-distribution distance to source on the deployed adapter over the full ImageNet-C stream (5000 images per corruption), at the $17$ recalibration interfaces (blocks 8--11 $\times$ \{\texttt{attn.qkv}, \texttt{attn.proj}, \texttt{mlp.fc1}, \texttt{mlp.fc2}\} and the head input) recalibration cuts this distance by $52\%$ (W6A6, all 15 corruptions, 3 seeds), i.e.\ it restores the quantizer's code usage toward its source calibration. The two interface groups behave differently and both are covered: over the 16 intermediate interfaces the local branch removes $47.2\%$ of the drift where no baseline exceeds $1.5\%$, and at the head input, the one interface NEO and FOA also act at, the global branch removes $90.0\%$ against their $20.6\%$ and $48.0\%$. Only \method{} is high in both. That head ordering is the scale correction made visible: NEO subtracts a running test mean and never references the source, FOA shifts by the source mean but not the scale, and \method{} matches both. The code-distribution restoration \emph{is} the quantization-specific manifestation of the moment alignment, and Proposition~\ref{prop:acc} below makes its link to accuracy precise, bounding the quantization error's contribution to the logits. This restoration is specific to matching the source statistics. ZOA, the one baseline designed for the quantized setting, adapts LayerNorm affine parameters against a feature-discrepancy objective rather than acting at the quantizer, and barely moves the code distribution (episodic, same interfaces, $-4.0\%$ versus \ours{}'s $-51.7\%$); T3A leaves it exactly unchanged, FOA and NEO move it only through the head interface ($-5.6\%$ and $-2.2\%$ overall), and FOZO increases it ($+11.8\%$) (Figure~\ref{fig:teaser}b).

\begin{table}[t]
\centering
\caption{\emph{Per-corruption quantizer code-distribution restored (\%), W6A6.} For each method,
the percentage of the corruption-induced code-distribution drift removed at the late
pre-quant interfaces (blocks 8--11 $\times$ \{qkv, proj, fc1\}; higher is better, $0=$ No Adapt).
Restored \% $=(d_{\text{NoAdapt}}-d_{\text{method}})/d_{\text{NoAdapt}}\times100$ on the $L_1$
code-histogram distance to the source calibration. 3-seed mean${}\pm{}$std. \method{} uses the
cumulative estimator. \emph{Reference:} No Adapt $=0$ by definition; T3A and NEO coincide with No
Adapt at these interfaces (they modify only the classifier prototypes / head input, never the
blocks 8--11 pre-quant activations), hence $0$. The \emph{Mean} row is the unweighted mean over the
15 corruptions; Figure~\ref{fig:teaser}b reports instead the drift-weighted summary (ratio of the
15-corruption mean distances; $+41.3\pm0.1\%$ for \method{} across the three seeds), which
up-weights high-drift corruptions such as contrast.}
\label{tab:codedist_full}
\footnotesize
\begin{tabularx}{\textwidth}{X|ccccc|c}
\toprule
Corruption & T3A & FOA & FOZO & ZOA & NEO & \method{} (cum.) \\
\midrule
Gaussian noise    & $0.0$ & $-2.5{\pm}1.9$ & $-14.4{\pm}1.2$ & $-0.4{\pm}1.6$ & $-4.3$ & $\mathbf{+58.9{\pm}0.2}$ \\
Shot noise        & $0.0$ & $-1.0{\pm}0.1$ & $-12.7{\pm}1.8$ & $+0.9{\pm}0.2$ & $-3.6$ & $\mathbf{+57.2{\pm}0.6}$ \\
Impulse noise     & $0.0$ & $+1.2{\pm}2.4$ & $-13.6{\pm}1.9$ & $-2.0{\pm}1.2$ & $-4.8$ & $\mathbf{+58.8{\pm}0.4}$ \\
Defocus blur      & $0.0$ & $+9.4{\pm}2.4$ & $-19.3{\pm}6.4$ & $+6.6{\pm}2.9$ & $+9.0$ & $\mathbf{+40.5{\pm}0.2}$ \\
Glass blur        & $0.0$ & $+10.2{\pm}2.2$ & $-9.3{\pm}5.4$ & $+4.4{\pm}3.4$ & $+10.0$ & $\mathbf{+46.4}$ \\
Motion blur       & $0.0$ & $+15.6{\pm}3.3$ & $-10.0{\pm}4.5$ & $+7.0{\pm}1.5$ & $+15.9$ & $\mathbf{+42.2{\pm}0.7}$ \\
Zoom blur         & $0.0$ & $+20.3{\pm}2.4$ & $-3.2{\pm}3.5$ & $+8.5{\pm}2.2$ & $+16.9$ & $\mathbf{+45.1{\pm}0.4}$ \\
Snow              & $0.0$ & $+13.7{\pm}2.9$ & $-3.8{\pm}4.0$ & $+5.5{\pm}3.8$ & $+13.3$ & $\mathbf{+41.6{\pm}0.4}$ \\
Frost             & $0.0$ & $+15.7{\pm}2.7$ & $-1.1{\pm}5.7$ & $+6.0{\pm}1.8$ & $+11.5$ & $\mathbf{+39.7{\pm}0.1}$ \\
Fog               & $0.0$ & $+14.9{\pm}6.4$ & $-11.3{\pm}6.1$ & $+15.6{\pm}0.8$ & $-3.7$ & $\mathbf{+43.9{\pm}0.3}$ \\
Brightness        & $0.0$ & $-15.1{\pm}10.1$ & $-52.6{\pm}22.6$ & $+6.7{\pm}2.2$ & $-6.1$ & $\mathbf{+37.9{\pm}1.2}$ \\
Contrast          & $0.0$ & $-0.8{\pm}2.5$ & $-7.5{\pm}2.1$ & $+2.3{\pm}0.3$ & $-5.7$ & $\mathbf{+61.6{\pm}0.1}$ \\
Elastic transform & $0.0$ & $+26.7{\pm}1.2$ & $+2.3{\pm}4.4$ & $+14.3{\pm}0.6$ & $+25.2$ & $\mathbf{+32.1{\pm}0.2}$ \\
Pixelate          & $0.0$ & $+8.4{\pm}6.0$ & $-31.1{\pm}17.3$ & $+5.7{\pm}1.1$ & $+14.3$ & $\mathbf{+48.3{\pm}1.2}$ \\
JPEG compression  & $0.0$ & $+4.3{\pm}2.5$ & $-34.4{\pm}7.8$ & $+4.1{\pm}0.8$ & $+5.4$ & $\mathbf{+50.5{\pm}0.9}$ \\
\midrule
\textbf{Mean}     & $0.0$ & $+8.1$ & $-14.8$ & $+5.7$ & $+6.2$ & $\mathbf{+47.0}$ \\
\bottomrule
\end{tabularx}
\end{table}

\paragraph{Per-corruption code-distribution restoration.}
Table~\ref{tab:codedist_full} breaks the Figure~\ref{fig:teaser}b summary down by corruption. \method{} restores the quantizer code distribution toward its source calibration on \emph{every} one of the 15 corruptions (from $+32.1\%$ on elastic transform to $+61.6\%$ on contrast), and the restoration is highly stable across seeds (std $\le 1.2\%$ per cell). The baselines do not. T3A is exactly $0\%$ on every corruption, as it rewrites the classifier's decision rule and leaves every activation in the network untouched. NEO, FOA and ZOA move the distribution only slightly in the mean ($+6.2\%$, $+8.1\%$, $+5.7\%$), and almost all of that comes from the head interface where they act; over the 16 intermediate interfaces they sit at $1.5\%$ or below. FOZO makes it \emph{worse} ($-14.8\%$ mean), as re-optimizing the model or prompts is not aligned with the frozen quantizers' calibration. This is the direct, per-corruption evidence that the quantizer-level mismatch is a real and otherwise-unaddressed failure mode, and that \method{}'s pre-quant recalibration is what targets it.

\paragraph{From the restored code distribution to a bounded accuracy loss.}
\begin{proposition}[Bounded accuracy loss]\label{prop:acc}
Let the sub-network from a recalibrated interface to the logits be $K$-Lipschitz, with $K_{>\ell}$
the product of the per-block operator norms after interface $\ell$, and
$e_{\ell,c}=Q(h_{\ell,c})-h_{\ell,c}$ the activation quantization error. The logit perturbation
relative to the unquantized network obeys $\|\hat z-z\|\le\sum_\ell K_{>\ell}\,\|e_\ell\|$. Each
$\mathbb{E}\,e_{\ell,c}^2$ splits into a granular term $\Delta_{\ell,c}^2/12$, fixed by the frozen
step and hence \emph{invariant} to recalibration, and a clipping term that grows once the activation
leaves the calibrated range.
\end{proposition}
\vspace{-4pt}
\begin{proof}[Proof of Proposition~\ref{prop:acc}]
The quantized forward pass is the FP composition with per-interface additive perturbations
$e_\ell$; propagating each through the remaining $K_{>\ell}$-Lipschitz maps and applying the
triangle inequality gives the sum bound. The granular error power $\Delta_c^2/12$ is fixed by the
frozen step, so recalibration leaves the absolute $\|e_\ell\|$ unchanged in the granular regime and
reduces it only through the clipping term (empirically small at the tested bit-widths); what it
returns to the source value is the operating point itself, the activation reoccupying its calibrated
range and source code levels (Lemma~\ref{prop:qerr}); the relative fidelity
$\mathrm{SQNR}_{\ell,c}=\mathrm{Var}(\hat h_{\ell,c})/\mathbb{E}\,e_{\ell,c}^2$ follows from that.
A prediction is preserved while the propagated perturbation stays below the unquantized logit
margin. The shift shrinks that margin while $\|e_\ell\|$ is unchanged (both measured in
Appendix~\ref{app:analysis}), which is why quantization
costs more under shift and costs more the coarser the step, $\|e_\ell\|$ being set by it; returning
the activation to the calibrated operating point restores the margin toward its source value and so
caps the quantization-induced label flips.
\end{proof}
This isolates why the gain concentrates on quantized models (Figure~\ref{fig:teaser}a): the term it acts on, $\{e_\ell\}$, exists \emph{only} under quantization, and on a \emph{frozen} quantizer its power cannot be reduced at all, only returned to the source operating point by moving the activation back into the calibrated range, which is precisely what the code-distribution restoration of Figure~\ref{fig:teaser}b measures. Two caveats keep the claim honest. First, the bound is on the quantization error alone; the recalibration's own estimation error (the $\eta$ of Proposition~\ref{prop:shrink}) adds a competing term, so it is minimized at the shrinkage optimum $\lambda^\star$, not at maximal restoration. This is measurable: the full local match ($\lambda{=}1$) restores $84.2\%$ of the code-distribution drift against \method{}'s $51.7\%$, and scores $53.52$ against $54.35$ (Table~\ref{tab:sens}). The metric is therefore not a proxy for accuracy, and the configuration that maximises it is not the one to deploy. Second, restoring the code distribution is necessary to remove $\{e_\ell\}$ but not by itself sufficient for accuracy, which also requires the corrected activation to remain a faithful estimate of the test signal; the two hold together only for a stable, partial correction (made precise in Corollary~\ref{cor:partial}).

\paragraph{Depth-asymmetry: why the late blocks.} The same bound also speaks to \emph{where} to
recalibrate. A recalibration at interface $\ell$ is imperfect: the scale-estimation noise of Proposition~\ref{prop:shrink} (a finite-sample $\widehat{\sigma}_{t,c}$) leaves an activation-space residual of norm $r_\ell$ on top of the corrected activation.
\begin{corollary}[Depth-asymmetry of recalibration]\label{prop:depth}
Adopt the hypotheses of Proposition~\ref{prop:acc} and model the path from interface $\ell$ to the
logits as a fixed $K_{>\ell}$-Lipschitz map (i.e.\ hold the downstream quantization errors fixed, or
absorb their first-order change into $K_{>\ell}$). Then an additional activation-space residual of
norm $r_\ell$ at interface $\ell$ enlarges the logit-perturbation bound
$\|\hat z-z\|\le\sum_\ell K_{>\ell}\|e_\ell\|$ by at most one further term, $K_{>\ell}\,r_\ell$.
Hence, for a fixed per-interface residual, its bounded contribution to the logit perturbation scales
with the downstream gain $K_{>\ell}$.
\end{corollary}
\begin{proof}[Proof sketch]
Under the fixed-downstream-map model, the residual $r_\ell$ propagates to the logits through the
same $K_{>\ell}$-Lipschitz composition as the quantization error $e_\ell$ in
Proposition~\ref{prop:acc}; by the triangle inequality its contribution to $\|\hat z-z\|$ is at most
$K_{>\ell}\,r_\ell$, an additional term of the same form.
\end{proof}
So the downstream gain $K_{>\ell}$ controls how much an imperfect correction at $\ell$ can perturb the logits: it bounds the \emph{risk} of intervening at $\ell$. We probe $K_{>\ell}$ with an empirical local proxy (Figure~\ref{fig:dgain}): injecting a controlled perturbation at each block input and reading the induced logit change. This proxy generally decreases with depth, from $\approx\!0.26$ averaged over the early blocks (0--3) to $\approx\!0.059$ over the late blocks (8--11), a $\sim\!4.4\times$ attenuation (near-monotone, with one small reversal); propagated through the \emph{quantized} downstream the same proxy falls from $>\!1$ at early blocks, i.e.\ genuine amplification, to $<\!1$ at late blocks. A correction at a late interface therefore has its residual attenuated most before reaching the head, while the same correction placed early is weighted by a much larger $K_{>\ell}$, so its residual, and any error it injects into the activation that the downstream quantizers then re-quantize, is amplified toward the logits. This is the risk-side counterpart of the empirical collapse of early- and middle-block recalibration below No Adapt (Table~\ref{tab:scope}). The bound identifies the late blocks as the \emph{lower-risk} site for an imperfect correction; the accuracy benefit and the number of blocks are determined empirically (the four-block plateau; Appendix~\ref{app:scope}).

\begin{figure}[t]
\centering
\includegraphics[width=0.55\textwidth]{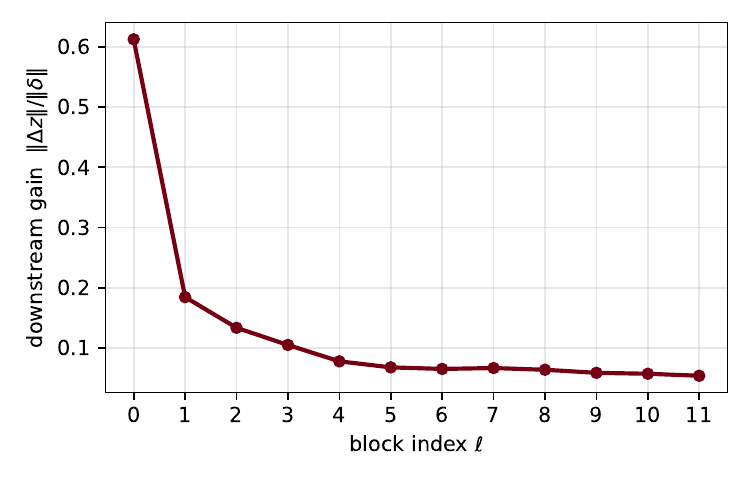}
\caption{\emph{Downstream perturbation gain by depth} (W6A6 ViT-B, clean ImageNet; empirical local proxy for $K_{>\ell}$, full-precision composition). A perturbation injected at each block input induces a logit change $\|\Delta z\|/\|\delta\|$ that generally decreases with depth (early blocks 0--3 mean $0.26$ vs.\ late 8--11 mean $0.059$, ${\sim}4.4\times$; near-monotone). By Corollary~\ref{prop:depth}, late interfaces have the smallest gain, so an imperfect correction there perturbs the logits least.}
\label{fig:dgain}
\end{figure}

\paragraph{(ii) A partial recalibration is the bias--variance optimum.}
In practice $\sigma_{t,c}$ is unknown: it is estimated from the $B\cdot N$ tokens of the current stream and an EMA. Write $\hat{\sigma}_{t,c}=\sigma_{t,c}(1+\eta)$ with $\mathbb{E}[\eta]=0$ and $\mathrm{Var}(\eta)=\delta^2$, the relative error of the estimate against the scale the batch at hand actually has. It grows as the number of tokens or the channel dimension shrinks, and it does not vanish with more data, since the true scale itself varies from batch to batch within a domain. The effective per-channel gain applied to the activation is $a_\lambda=(1-\lambda)+\lambda\,\sigma_{s,c}/\hat{\sigma}_{t,c}$, and the oracle gain that makes the output standard deviation exactly $\sigma_{s,c}$ is $a^\star=\sigma_{s,c}/\sigma_{t,c}=:r$.
\begin{proposition}[Optimal recalibration strength]\label{prop:shrink}
To first order in $\delta$, the recalibration strength minimizing the expected squared
per-channel gain error is
\[
\lambda^\star=\frac{(1-r)^2}{(1-r)^2+r^2\delta^2}\ \in\ (0,1)\quad\text{whenever }\delta^2>0,
\]
strictly decreasing in the estimation noise $\delta^2$ and equal to $1$ only in the noiseless
limit $\delta^2{=}0$.
\end{proposition}
This is an optimal-shrinkage intensity (structure over structure-plus-noise, as in covariance shrinkage \citep{ledoit2004well}), which is why wider-channel ViT-B tolerates a larger $\lambda$ than the smaller ViT-S or the distilled DeiT-B (Appendix~\ref{app:backbones}).
\begin{proof}[Proof of Proposition~\ref{prop:shrink}]
Expand $\sigma_{s,c}/\hat{\sigma}_{t,c}=r(1-\eta+\eta^2-\dots)$. Then
$\mathbb{E}[a_\lambda-a^\star]=(1-\lambda)(1-r)+\lambda r\delta^2$ (using $\mathbb{E}[1/(1+\eta)]\approx 1+\delta^2$)
and $\mathrm{Var}(a_\lambda)=\lambda^2r^2\,\mathrm{Var}(1/(1+\eta))\approx\lambda^2r^2\delta^2$, so the
expected squared gain error is
\[
\mathbb{E}\big[(a_\lambda-a^\star)^2\big]=\big[(1-\lambda)(1-r)+\lambda r\delta^2\big]^2+\lambda^2 r^2\delta^2.
\]
Setting its derivative in $\lambda$ to zero yields
$\lambda^\star=a(a-b)/[(a-b)^2+r^2\delta^2]$ with $a=1-r,\ b=r\delta^2$, which reduces to the stated
leading-order form $\lambda^\star=(1-r)^2/[(1-r)^2+r^2\delta^2]+O(\delta^2)\in(0,1)$, decreasing
monotonically in $\delta^2$ and equal to $1$ only when $\delta^2=0$.
\end{proof}
This has exactly the form of an optimal shrinkage intensity, $\text{structure}/(\text{structure}+\text{noise})$, as in covariance shrinkage \citep{ledoit2004well}: full standardization ($\lambda{=}1$) is optimal only with perfect moment estimates, and the noisier the per-channel statistics the more one should shrink the variance match toward the identity. It predicts the ordering in Figure~\ref{fig:lambda}: ViT-B, whose wider channels give lower-variance estimates, peaks at a larger $\lambda$ ($0.5$--$0.7$), whereas the smaller ViT-S and the distilled DeiT-B prefer $\lambda\!\le\!0.5$, and $\lambda{=}0.5$ is a single robust default. The shrinkage is governed by the \emph{scale} estimate, which carries the finite-sample noise $\delta^2$; the per-channel mean is estimated far more reliably, so it is the noisy scale that pulls $\lambda^\star$ below one. In the deployed convex-combination form the same $\lambda$ also shrinks the mean offset toward the identity (Equation~\ref{eq:local}); this retains a fraction of the test activation and does not enter the scale-driven gain error above.

\medskip\noindent\emph{From the gain error to accuracy.} The same bias--variance expression makes the two caveats of paragraph~(i) precise. The recalibrated activation has standard deviation $\sigma'=a_\lambda\sigma_{t,c}$, so its deviation from the source operating point is $\sigma'-\sigma_{s,c}=\sigma_{t,c}(a_\lambda-r)$. Under (A5), and with the granular quantization floor $\lambda$-independent and clipping negligible at the tested bit-widths (paragraph~(i)), the expected excess task loss contributed by this channel is, to leading order, $J(\lambda)=\kappa\,\sigma_{t,c}^2\,\mathbb{E}[(a_\lambda-r)^2]+\mathrm{const}$ with $\kappa>0$, an increasing affine function of the gain error $\mathbb{E}[(a_\lambda-r)^2]=[(1-\lambda)(1-r)+\lambda r\delta^2]^2+\lambda^2r^2\delta^2$.
\begin{corollary}[Stable partial restoration is optimal]\label{cor:partial}
For a channel under a genuine scale shift ($r\neq1$) with a noisy scale estimate ($\delta^2>0$),
$J$ shares the minimizer of $\mathbb{E}[(a_\lambda-r)^2]$, and therefore: \emph{(i)} restoration is
\emph{necessary}, $J(0)>\min_\lambda J(\lambda)$, since the bias equals $(1-r)^2>0$ at $\lambda{=}0$;
\emph{(ii)} maximal restoration is \emph{insufficient}, the per-batch full match $\lambda{=}1$ does
not minimize $J$; and \emph{(iii)} the optimum is \emph{interior and partial},
$\lambda^\star\in(0,1)$ as in Proposition~\ref{prop:shrink}, shrinking toward the identity as the
estimation noise $\delta^2$ grows.
\end{corollary}
\begin{proof}
$J$ is increasing affine in $\mathbb{E}[(a_\lambda-r)^2]$ and hence shares its minimizer. At
$\lambda{=}0$ the error is $(1-r)^2$, strictly above its value at $\lambda^\star>0$ (which exists
because $r\neq1$), giving (i). The error is the bias--variance sum from the proof of
Proposition~\ref{prop:shrink}, whose unique stationary point is
$\lambda^\star=(1-r)^2/[(1-r)^2+r^2\delta^2]+O(\delta^2)$; for $\delta^2>0$ this is strictly below
$1$, so $\lambda{=}1$ is not the minimizer (ii) and the optimum is interior (iii).
\end{proof}
This formalizes the second caveat. The per-batch full match restores the code distribution \emph{most} in expectation ($\mathbb{E}[a_1]\approx r$, so the aggregate code histogram sits closest to source), yet it injects the variance $\lambda^2r^2\delta^2$ that an aggregate code-distance does not register, so accuracy is governed by the stable, partial optimum rather than by maximal restoration. The optimum $\lambda^\star$ depends on the unknown $r$ and $\delta^2$ and is therefore data- and backbone-dependent; we do not claim $\lambda{=}0.5$ equals it, only that it is a robust value in the empirically flat region (Section~\ref{sec:hpsens}). The statements above describe the first-order regime $\lambda^\star\in(0,1)$ relevant at our batch sizes; pushed to the limit, the same shrinkage logic sends $\lambda^\star\to0$ when the scale estimate is so noisy that its bias reduction no longer outweighs the injected variance, i.e.\ the channel is then best left uncorrected.

\paragraph{(iii) Drift-free continual adaptation.}
\begin{proposition}[Drift-free continual adaptation]\label{prop:drift}
Under (A3), for every batch $k$ of an arbitrarily long stream the recalibrated activation has
first moment $\mathbb{E}[\hat{x}^{(k)}_c]=(1-\lambda)\mu^{(k)}_{t,c}+\lambda\mu_{s,c}$ and a gain
bounded by the clip ratio $\rho$,
$a_\lambda\in[(1-\lambda)+\lambda/\rho,\ (1-\lambda)+\lambda\rho]$. Because the running mean
$\mu_{t,c}$ is a source-initialized average of bounded activations and the gain is clamped, both the
first moment and the gain stay bounded for every batch, so the running state remains bounded over
arbitrarily long streams rather than drifting.
\end{proposition}
\begin{proof}[Proof sketch]
The correction is a convex combination of the activation and a re-anchoring to the \emph{fixed} source moments,
$\hat{x}^{(k)}_c=(1-\lambda)x^{(k)}_c+\lambda\big[(x^{(k)}_c-\mu^{(k)}_{t,c})\,r+\mu_{s,c}\big]$ with
$r=\sigma_{s,c}/(\sigma_{t,c}+\epsilon)$. By (A3)
$\mathbb{E}[x^{(k)}_c-\mu^{(k)}_{t,c}]\approx 0$, so
$\mathbb{E}[\hat{x}^{(k)}_c]=(1-\lambda)\mu^{(k)}_{t,c}+\lambda\mu_{s,c}$; the running mean is a
source-initialized average of bounded activations and hence bounded, and the clamped ratio
$r\in[1/\rho,\rho]$ (we use $\rho=3$) bounds $a_\lambda$, so both the first moment and the
output variance are bounded for every $k$.
\end{proof}
The fixed source anchor is the key difference from optimizer-based test-time adaptation. Methods that update parameters online from an unsupervised objective (entropy minimization \citep{wang2021tent} or the zeroth-order LayerNorm updates of ZOA \citep{deng2025test}) have no such anchor: their iterate $\theta^{(k)}=\theta^{(k-1)}-\gamma\,\widehat{\nabla}\mathcal{L}^{(k)}$ accumulates the error of the unsupervised gradient and can drift away from $\theta^{(0)}$, which is why long-stream stability needs explicit safeguards (Fisher anti-forgetting in EATA \citep{niu2022efficient}, sharpness/entropy gating in SAR \citep{niu2023towards}, a domain memory in ZOA \citep{deng2025test}). \method{} stays source-anchored with bounded running state, matching the flat per-round trajectories in Table~\ref{tab:t3}.

\paragraph{(iv) The single log-quantized (post-GELU) interface.}
Among the recalibrated interfaces only the post-GELU input to \texttt{mlp.fc2} is logarithmically quantized, and only under AdaLog (W4A4/W3A3); $\{$qkv, proj, fc1$\}$ use uniform activation quantizers at every bit-width, so Lemma~\ref{prop:qerr} and Proposition~\ref{prop:acc} apply to them directly. For the post-GELU interface we give the analogue.
\begin{proposition}[Restoration under a log quantizer]\label{prop:logq}
An AdaLog quantizer is a uniform quantizer with step $\delta$ in the variable
$u=-\log_2\!\big((x+\beta)/s\big)$ restricted to $u\in[0,U_{\max}]$, where $s$, $\beta$,
$\delta=q/r$, and $U_{\max}$ are fixed at source calibration. It has \emph{constant relative
resolution}: the granular relative-error power is ${\approx}(\delta\ln 2)^2/12$, independent of the
activation scale. Under \textup{(A1)--(A2)}, restoring the source per-channel scale returns the
log-magnitude distribution $u$ to its source position within $[0,U_{\max}]$, restoring the
out-of-range (clip/drop) probability, the scale-dependent error term of a log quantizer, to its
source-calibrated value.
\end{proposition}
\begin{proof}[Proof sketch]
$x$ quantized by AdaLog is equivalent to $u=-\log_2((x+\beta)/s)$ quantized on a uniform
$\delta$-grid over $[0,U_{\max}]$ (mass with $u<0$ saturates at the top level; $u>U_{\max}$ is
dropped). A multiplicative scale correction $x\mapsto\gamma x$ (the leading effect of the
$\sigma$-rescale on the post-GELU activation, which is bounded below) maps to the \emph{additive}
shift $u\mapsto u-\log_2\gamma$. Setting $\gamma=\sigma_{s,c}/\sigma_{t,c}$ aligns the shifted
$u$-distribution with the source one to first order, so the mass with $u\notin[0,U_{\max}]$ (the
clip/drop term) returns to its source value, while the scale-invariant granular term is unchanged.
\end{proof}
Two caveats keep this honest. First, the recentering ($\mu$-shift) and AdaLog's learned $\beta$ make the linear-to-log map only approximately additive, so we treat the residual mean correction as a first-order perturbation and read the restoration as driven by the \emph{scale} match (consistent with $\lambda$ acting on the variance). Second, this argument is what covers \texttt{mlp.fc2} at the low bit-widths. The code-distribution diagnostic of Appendix~\ref{app:analysis} is run at W6A6, where that interface is twin-uniform quantized and is measured alongside the others; under AdaLog it is log-quantized instead, and there the restoration is argued here rather than measured, through the same scale/log-range mechanism that also improves accuracy.

\paragraph{Relation to the sensitivity analysis of ZOA.}
ZOA \citep{deng2025test} analyzes a linear model $y=Wx$ under MSE with weight-only uniform quantization, an i.i.d.\ quantization error $\Delta W$ and a zero-mean perturbation $\delta$ independent of the Gaussian input, and show that the excess loss quantization adds under shift is $\Delta\mathcal L=\mathrm{Tr}\big(\mathbb E[\Delta W^{\top}\Delta W](\Sigma_x+\Sigma_\delta)\big) \propto(2^{n}-1)^{-2}$: quantization always adds loss under shift, and more so at low precision. Our analysis is complementary in what it quantizes. That model has no activation quantizer, so the shift enters only through the input second moment, which a zero-mean $\delta$ can only inflate; the excess it describes is the weight-step counterpart of the granular floor $\Delta^{2}/12$ in Proposition~\ref{prop:acc}, which recalibration leaves unchanged. Appendix~\ref{app:w4a8} puts a number on that split: halving the weight bit-width from 8 to 4 moves every method by at most $0.42$ points, while halving the activation bit-width moves the same methods by six to eight, so the term their analysis scales with is not where the accuracy under shift is lost. The quantities of Lemma~\ref{prop:qerr} and Proposition~\ref{prop:levels}, a per-channel mean drift and a scale contraction at a frozen activation grid, have no counterpart there ($\mathbb E[\delta]=0$ and $\Sigma_x+\Sigma_\delta\succeq\Sigma_x$), yet they are what Section~\ref{sec:analysis} measures and what a test-time correction can move without touching $W$. The two analyses agree on the direction of the precision dependence and differ on what is recoverable: theirs motivates adapting a quantized model at all, ours says which displacement to undo and where.

\paragraph{Scope.}
Quantizer coverage is as follows. The uniform model (Lemma~\ref{prop:qerr} and Proposition~\ref{prop:acc}) applies \emph{exactly} to PTQ4ViT (W6A6/W8A8, all interfaces) and \emph{directly} to the $\{$qkv, proj, fc1$\}$ interfaces under AdaLog (W4A4/W3A3), which stay uniformly quantized; the only log-quantized recalibration interface, post-GELU \texttt{mlp.fc2}, is covered by Proposition~\ref{prop:logq}. Consistently, the symmetric-code restoration we report is measured on all $17$ recalibration interfaces at W6A6 ($52\%$), and \method{} is the best adapter at every bit-width where the quantizer yields a usable model (Appendix~\ref{app:crossquant}). The remaining limitation is that all results are first-order, moment-level statements under (A1)--(A4): they explain the observed mismatch reduction, the $\lambda<1$ optimum, and the absence of drift, but do not bound the residual error when the standardized \emph{shape} itself changes (shifts that alter higher moments or inter-channel correlations), which we leave to future work.

\end{document}